\documentclass[10pt,journal,compsoc]{IEEEtran}
\usepackage{amsfonts}
\usepackage{amsthm}
\usepackage{epsfig}
\usepackage{graphicx}
\usepackage{bm}
\usepackage{dsfont}
\usepackage{url} 
\usepackage{amsmath,amssymb,verbatim}
\usepackage{float}
\usepackage{threeparttable}
\usepackage{color} 
\usepackage{dblfloatfix}
\usepackage{changepage}
\usepackage{comment}
\usepackage{multirow}
\usepackage{colortbl}
\usepackage{enumerate}  
\usepackage{lipsum}
\usepackage{bbm}
\usepackage{subfigure}
\usepackage{tabularx}
\usepackage{pict2e}
\usepackage{svg}
\usepackage[normalem]{ulem}
\usepackage[ruled,linesnumbered]{algorithm2e}

\newcommand{\revise}[1]{{\color{black}#1}}
\newcommand{\reviseSecond}[1]{{\color{black}#1}}

\RestyleAlgo{ruled}
\SetKwComment{Comments}{/* }{ */}
\SetKwComment{Comment}{// }{ }
\SetKwInOut{Input}{Input}
\SetKwInOut{Output}{Output}
\SetKwInput{Init}{Init.}
\SetKw{KwAll}{all}

\newtheorem{theorem}{\it Theorem}

\newtheorem{lemma}{\it Lemma}

\newtheorem{remark}{\it Remark}

\newtheorem{observation}{\it Observation}

\ifCLASSOPTIONcompsoc
  \usepackage[nocompress]{cite}
\else
  \usepackage{cite}
\fi
\ifCLASSINFOpdf
\else
\fi
\begin{document}
%
\title{Learning-augmented Online Minimization of Age of Information and Transmission Costs}
%
%
%
%

\author{Zhongdong Liu, Keyuan Zhang, Bin Li, Yin Sun, Y.~Thomas Hou, and Bo Ji 
\IEEEcompsocitemizethanks{\IEEEcompsocthanksitem This research was supported in part by 
ONR MURI grant N00014-19-1-2621, NSF grants CNS-2106427 and CNS-2239677, Army Research Office grants W911NF-21-1-0244 and W911NF-24-1-0103, Virginia Commonwealth Cyber Initiative (CCI), Virginia Tech Institute for Critical Technology and Applied Science (ICTAS), and Nokia Corporation. A preliminary version of this work is to be presented at IEEE INFOCOM 2024 Age and Semantics of Information Workshop \cite{Liu2405:Learning}.
\IEEEcompsocthanksitem Zhongdong Liu (zhongdong@vt.edu), Keyuan Zhang (keyuanz@vt.edu), and Bo Ji (boji@vt.edu) are with the Department of Computer Science, Virginia Tech, Blacksburg, VA. 
Bin Li (binli@psu.edu) is with the Department of Electrical Engineering,  Pennsylvania State University, University Park, PA. 
Yin Sun (yzs0078@auburn.edu) is with the Department of Electrical and Computer
Engineering, Auburn University, Auburn, AL.
Y.~Thomas Hou (thou@vt.edu) is with the Bradley Department of Electrical and Computer Engineering, Virginia Tech, Blacksburg,~VA.}
\thanks{Manuscript received April 19, 2005; revised August 26, 2015.}}

%
%

\markboth{Journal of \LaTeX\ Class Files,~Vol.~14, No.~8, August~2015}%
{Shell \MakeLowercase{\textit{et al.}}: Bare Demo of IEEEtran.cls for Computer Society Journals}
%



\IEEEtitleabstractindextext{%
\begin{abstract}
We consider a discrete-time system where a resource-constrained source (e.g., a small sensor) transmits its time-sensitive data to a destination over a time-varying wireless channel. Each transmission incurs a fixed transmission cost (e.g., energy cost), and no transmission results in a staleness cost represented by the \textit{Age-of-Information}. The source must balance the tradeoff between transmission and staleness costs. To address this challenge, we develop a robust online algorithm to minimize the sum of transmission and staleness costs, ensuring a worst-case performance guarantee. While online algorithms are robust, they are usually overly conservative and may have a poor average performance in typical scenarios. In contrast, by leveraging historical data and prediction models, machine learning (ML) algorithms perform well in average cases. However, they typically lack worst-case performance guarantees. To achieve the best of both worlds, we design a learning-augmented online algorithm that exhibits two desired properties: (i) \textit{consistency}: closely approximating the optimal offline algorithm when the ML prediction is accurate and trusted; (ii) \textit{robustness}: ensuring worst-case performance guarantee even ML predictions are inaccurate. Finally, we perform extensive simulations to show that our online algorithm performs well empirically and that our learning-augmented algorithm achieves both consistency and robustness. 
\end{abstract}

}

\maketitle

\IEEEdisplaynontitleabstractindextext

%
\IEEEpeerreviewmaketitle

\IEEEraisesectionheading{\section{Introduction}}
\IEEEPARstart{I}{n} recent years, we have witnessed the swift and remarkable development of the Internet of Things (IoT), which connects billions of entities through wireless networks \cite{li2015internet}. 
These entities range from small, resource-constrained sensors (e.g., temperature sensors and smart cameras) to powerful smartphones.
Among various IoT applications, one of the most important categories is real-time IoT applications, which requires timely information updates from the IoT sensors. 
For example, in industrial automation systems \cite{wu2017real,10008202}, battery-powered IoT sensors are deployed to provide data for monitoring equipment health and product quality.
On the one hand, IoT sensors are usually small and have limited battery capacity, and thus frequent transmissions drain the battery quickly; on the other hand, occasional transmissions render the information at the controller outdated, potentially leading to detrimental decisions.
In addition, wireless channels can be unreliable due to potential channel fading, interference, and the saturation of wireless networks if the traffic load generated by numerous sensors is high \cite{10049773}.
Clearly, under unreliable wireless networks, IoT sensors must transmit strategically to balance the tradeoff between transmission cost (e.g., energy cost) and data freshness. 
Other applications include wildfire real-time monitoring systems, unmanned aerial vehicle systems, and so on (see more discussions in Section~\ref{sec:case_study}).

To this end, in the first part of this work, we study the tradeoff between transmission cost and data freshness under a time-varying wireless channel. 
Specifically, we consider a discrete-time system where a device transmits data to an access point (AP) over an ON/OFF wireless channel under an online setting. In this setting, the device is aware of the current and previous channel states but does not have knowledge of future states. Consequently, transmissions occur only when the channel is in the ON state.
While each transmission incurs a fixed \textit{transmission cost}, no transmission renders the information at the AP outdated.  
To measure the freshness of information at the AP, we use a popular timeliness metric called the  \textit{Age-of-Information (AoI)}~\cite{6195689}, which is defined as the time elapsed since the generation time of the freshest delivered packet (see formal definition in Section~\ref{sec:system_model}).\footnote{We use AoI instead of delay because AoI measures data freshness at the destination by tracking the time since the latest received update, whereas delay measures transmission latency, which is inherently captured in AoI by definition.}
The evolution of AoI depends on the occurrence of transmissions: it increases linearly with time when no transmission occurs and resets to a smaller value when a new update is successfully transmitted. To account for the penalty associated with outdated information at the AP, we introduce the \textit{staleness cost}, quantified by the AoI.
To minimize the sum of transmission costs and staleness costs, we develop a robust online algorithm that achieves a competitive ratio (CR) of 
$3$. That is, different from typical studies with stationary network assumptions, the cost of our online algorithm is at most three times larger than that of the optimal offline algorithm under the worst channel state (see the definition of CR in Section~\ref{sec:system_model}).

While online algorithms exhibit robustness against worst-case situations, they often lean towards excessive caution and may have a subpar average performance in real-world scenarios.
On the other hand, by exploiting historical data to build prediction models, machine learning (ML) algorithms can excel in average cases. Nonetheless, 
ML algorithms could be sensitive to disparity in training and testing data due to distribution shifts or adversarial examples, resulting in poor performance and lacking worst-case performance guarantees.

To that end, we design a novel learning-augmented online algorithm that takes advantage of both ML and online algorithms. 
Specifically, our learning-augmented online algorithm integrates ML prediction (a series of times indicating when to transmit) into our online algorithm, achieving two desired properties: (i) \textit{consistency}: when the ML prediction is accurate and trusted, our learning-augmented online algorithm performs closely to the optimal offline algorithm,  and (ii) \textit{robustness}: even when the ML prediction is inaccurate, our learning-augmented online algorithm still offers a worst-case guarantee. 

Our main contributions are as follows. 

\textit{First}, we study the tradeoff between transmission cost and data freshness in a dynamic wireless channel by formulating an optimization problem to minimize the sum of transmission and staleness costs under an ON/OFF channel. 

\textit{Second}, following a similar line of analysis as in \cite{8849808}, we reformulate our (non-linear) optimization problem into a linear Transmission Control Protocol (TCP) acknowledgment problem \cite{10.1145/380752.380845} and propose a primal-dual-based online algorithm that achieves a CR of $3$. 
While a similar primal-dual-based online algorithm has been claimed to asymptotically achieve a CR of $e/(e-1)$~\cite{8849808}, their analysis of CR relies on an (unrealistic) asymptotic setting
(see Remark~\ref{remark:wrong_cr}). 

\textit{Third}, by incorporating ML predictions into our online algorithm, we design a novel learning-augmented online optimization algorithm that achieves both consistency and robustness. 
\textit{To the best of our knowledge, this is the first study on AoI that incorporates ML predictions into online optimization to achieve consistency and robustness.}

\textit{Finally}, we perform extensive simulations using synthetic and real trace datasets. Our online algorithm outperforms the theoretical analysis, and our learning-augmented algorithm can achieve consistency and robustness.

The remainder of this paper is organized as follows.  Section~\ref{sec:relatedwork} reviews related work. 
The system model and problem formulation are introduced in Section~\ref{sec:system_model}.
Sections~\ref{sec:online_algorithm}  and \ref{sec:ml_augmented_algorithm} present our robust online algorithm and our learning-augmented online algorithm, respectively. 
Finally, we show the numerical results in Section~\ref{sec:simulation}, discuss the limitations of this work in Section~\ref{sec:limitations}, and conclude our paper in Section~\ref{sec:conclusion}.

\section{Related Work}
\label{sec:relatedwork}
This paper connects and contributes to three expanding areas of research: (i) stationary AoI, which makes assumptions under stationary regimes (e.g., data generation or channel availability following specific distributions), (ii) non-stationary AoI, and (iii) learning-augmented online algorithms. We discuss each of these in the following.

\textbf{\textit{Stationary AoI.}}
The first category includes studies that consider the joint minimization of AoI and certain costs under the stationary settings \cite{9163054,10008099,9488746,9358178, 10349861}.
The work of \cite{9163054} studies the problem of minimizing the average cost of sampling and transmission over an unreliable wireless channel subject to average AoI constraints.
Similarly, in \cite{10349861}, a source monitors a stochastic process and, upon sampling, either transmits raw data over an ON/OFF channel at a transmission cost or processes it locally at a processing cost. A stationary randomized policy is proposed to minimize the distortion of the received information while maintaining the
AoI at the destination below a threshold and satisfying a cost constraint at the source.  Among those AoI works that consider the stationary setting, the most related work to ours is \cite{9488746}. 
Although \cite{9488746} examines a similar problem, there are key differences compared to our problem.
First, regarding the system model, \cite{9488746} considers a continuous system where updates are generated with an inter-generation time following a known continuous distribution. Moreover, the updates in \cite{9488746} can be transmitted at any time (i.e., the channel is always ON). In contrast, we consider a discrete-time model where the device can generate updates at any discrete time slot, but the transmission channel could be OFF. 
Second, in terms of algorithm design and analysis, the proposed algorithm in \cite{9488746} assumes that the inter-generation time distribution (e.g., mean and variance) is known. However, in our algorithm, we do not rely on these assumptions. 
In summary, although the assumptions in these studies lead to tractable performance analysis, such assumptions may not hold in practical scenarios.

\textbf{\textit{Non-stationary AoI.}}
The second category contains studies that focus on non-stationary settings  \cite{9795317,9798166,8849808,10007803, 10226091}. 
For example, in \cite{9795317}, the authors proposed online algorithms to minimize the AoI of users in a cellular network under adversarial wireless channels. 
A similar adversarial framework is explored in \cite{10226091}, which considers a time-slotted communication network involving a base station, an adversary, multiple users, and multiple communication channels. Both the scheduler and the adversary operate under average power constraints, and the probability of successful transmission is contingent upon their respective power levels. In \cite{10226091}, the authors provide a universal lower bound for the average AoI and analyze the existence of Nash equilibria within this scenario.
In these AoI works that consider non-stationary settings, the most relevant work to ours is \cite{8849808}, where the authors study the minimization of the sum of download costs and AoI costs under a non-stationary wireless channel. A primal-dual-based randomized online algorithm is shown to have an asymptotic CR of $e/(e-1)$. However, this CR is attained under an (unrealistic) asymptotic setting (see Remark~\ref{remark:wrong_cr}).
In this work, we propose an online primal-dual-based algorithm that achieves a CR of $3$ in the non-asymptotic regime.

\textbf{\textit{Learning-augmented Online Algorithms.}} 
In recent years, advances in ML models have inspired researchers to revisit the design of classical algorithms. Learning-augmented algorithms, also referred to as algorithms with predictions, have been studied to incorporate insights from ML predictions to improve decision-making in a variety of problems, including online optimization~\cite{10.1145/3447579}, online learning~\cite{pmlr-v178-golowich22a}, and offline combinatorial optimization~\cite{bampis2025polynomial}. 
For a comprehensive collection of relevant papers, we refer interested readers to the corresponding website \cite{Learning_augmented_algorithms_collection}. 
Below, we will focus on the online optimization problems and picture this area through two major lenses: (\romannumeral 1) applications and problem domains, and (\romannumeral 2) the design of algorithms and frameworks.

(\romannumeral 1) Applications and problem domains. In the context of online optimization problems, the seminal work~\cite{10.1145/3447579} was the first to incorporate ML predictions into the online Marker algorithm for the caching problem, and it introduced robustness and consistency as two desired properties for the design of learning-augmented online algorithms. Since then, many classical online optimization problems and their variations have been explored, such as ski-rental, online bipartite matching, and online resource allocation. These problems have broad practical applications. For example, \cite{Zhang_Liu_Choi_Ji_2024} studied the two-level ski rental problem, which is well-suited for modeling cloud service subscription scenarios.
Among these problems, the most related problem to ours is the TCP acknowledgement problem~\cite{NEURIPS2020_e834cb11}. In~\cite{NEURIPS2020_e834cb11}, the uncertainty comes from the packet arrival times, and the controller can make decisions at any time. In our work, however, the uncertainty comes from the channel states, and no data can be transmitted when the channel is OFF. Lacking the freedom to transmit data at any time distinguishes our problem from the existing literature and renders their algorithm inapplicable. To the best of our knowledge, we are the first to explore learning-augmented online algorithm design in AoI scheduling applications.

(\romannumeral 2) The design of algorithms and frameworks. Since the seminal work, many learning-augmented algorithms have been proposed for various problems. Although these algorithms achieve promising results by effectively leveraging problem-specific structures, their approaches are often tailored to particular settings and do not easily generalize to other problem domains. Later, \cite{NEURIPS2020_e834cb11} introduced a more general framework for the algorithm design and theoretical performance analysis by leveraging the structure of primal-dual programming. In~\cite{NEURIPS2020_e834cb11}, this framework has been applied to address a diverse range of problems, including the weighted set cover, ski-rental, and TCP acknowledgment.  In this work, we extend this framework to our AoI scheduling problem. This extension is non-trivial, due to the unique challenges arising from differing sources of uncertainty and their distinct impacts. 
Additionally, we assume that the algorithm does not know the quality of the predictions; in other words, the predictions are generated by a black-box ML model. 
This contrasts with other studies that assume access to additional information about the predictions, such as uncertainty quantification~\cite{shen2025algorithms}. Given these differing assumptions, such works are considered complementary and parallel to ours.
Another line of research focuses on how to generate predictions that best support the downstream learning-augmented online algorithm. For example, \cite{3530894} proposed a neural network training algorithm that explicitly incorporates the output of the downstream online algorithm in the context of online convex optimization with switching costs. In contrast, our work concentrates exclusively on how to effectively utilize predictions through the careful design of learning-augmented algorithms. The task of generating predictions is considered orthogonal to our contributions.

\section{System Model and Problem Formulation}
\label{sec:system_model}
\textbf{\textit{System Model.}}
Consider a status-updating system where a resource-limited device sends time-sensitive data to an access point (AP) through an unreliable wireless channel (see Fig.~\ref{fig:system_model}). 
The system operates in discrete time slots, denoted by $t=1,2,\ldots, T$,  where $T$ is finite and can be arbitrarily large.
We use $s(t) \in \{0, 1\}$ to denote the channel state at time $t$, where $s(t) = 1$ means the channel is ON, allowing the device to access the AP; while $s(t) = 0$ means the channel is OFF, preventing access to the AP. 
The sequence of channel states over the time horizon is represented by $\mathbf{s} = \left\{ {s(1), \dots,s(T)} \right\}$.

\begin{figure}
    \centering
    \includegraphics[width=0.99\linewidth]{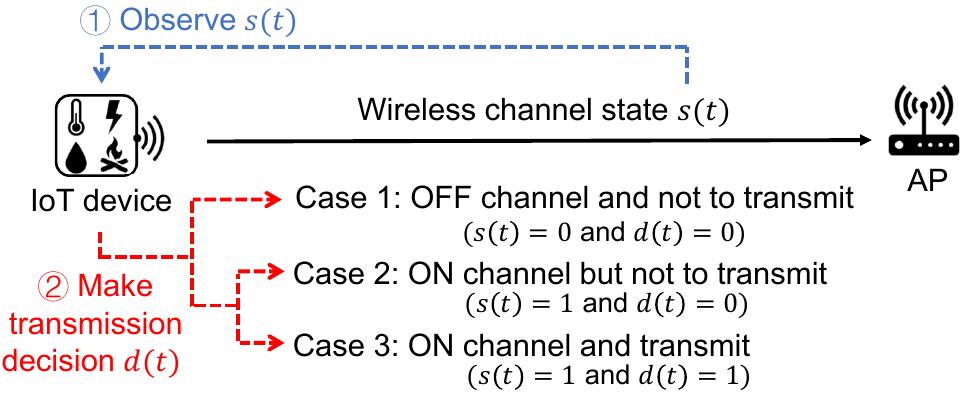}
    \caption{\reviseSecond{An illustration of our system model. The device sends data to AP through an unreliable wireless channel. Assuming the current time slot is $t$. At first, the device probes to know the current channel state $s(t)$. Then, the device decides whether to transmit or not. There are three cases: (i) the channel is OFF, and the device does not transmit; (ii) the channel is ON, but the device chooses not to transmit; and (iii) the channel is ON, and the device transmits.}}
    \label{fig:system_model}
\end{figure}

At the beginning of each slot, the device probes to know the current channel state (see path~$1$ in Fig.~\ref{fig:system_model}) and then decides whether to transmit its freshest data to the AP (see path~$2$ in Fig.~\ref{fig:system_model}). There are three possible cases: (i) the channel is OFF, so the device does not transmit; (ii) the channel is ON, but the device chooses not to transmit; and (iii) the channel is ON, and the device transmits.
The transmission decision at slot $t$ is denoted by $d(t) \in \{0,1\}$, where $d(t) = 1$ if the device decides to transmit (i.e., generates a new update and transmits it to the AP), and $d(t) = 0$ if not. 
A scheduling algorithm $\pi$ is denoted by  $\pi  = \left\{ {{d^\pi }(t)} \right\}_{t = 1}^T$, where ${d^\pi }(t)$ is the transmission decisions made by algorithm $\pi$ at time $t$. For simplicity, we use $\pi  = \left\{ {{d}(t)} \right\}_{t = 1}^T$ throughout the paper. 
When the device decides to transmit, it incurs a fixed \textit{transmission cost} (e.g., energy cost) of $c > 1$, and the data on the AP will be successfully updated at the end of slot $t$ if the channel is ON at slot $t$; otherwise, the data on the AP gets staler.\footnote{If $c \le 1$, which is less than the staleness cost (at least $1$), then the optimal policy is to transmit at every ON slot.} 
To quantify the freshness of data at the AP side, we utilize a metric called \textit{Age-of-Information (AoI)}~\cite{6195689}, which measures the time elapsed since the freshest received update was generated.
We use $a(t)$ to denote the AoI at time $t$, which evolves as
\begin{equation}
{a}(t)=\left\{\begin{array}{ll}
0, & \text {if} \ {s}(t) \cdot {d}(t)=1; \\
{a}(t-1)+1, & \text {otherwise},
\end{array}\right.
\label{eq:AoI_evo}
\end{equation}
where the AoI drops to $0$ if the device transmits at ON slots; otherwise, it increases by $1$.\footnote{Some studies let the AoI drop to $1$, wherein our analysis still holds. We let the AoI drop to $0$ to make the discussion concise.}
We assume that $a(0) = 0$. 
To reflect the penalty that the AP does not get the update at time $t$, we introduce the \textit{staleness cost}, which is represented by the AoI at time $t$.

\textbf{\textit{Problem Formulation.}}
The total cost of an algorithm $\pi$ is 
\begin{equation}
    C(\mathbf{s}, \pi) \triangleq \sum\limits_{t = 1}^T {(c \cdot d(t) + a(t))}, 
    \label{eq:total_cost}
\end{equation}
where the first item $c \cdot d(t)$ is the transmission cost at time $t$, and the second item $a(t)$ is the staleness cost at time $t$.
In this paper, we focus on the class of \emph{online} scheduling algorithms, under which the information available at time $t$ for making decisions includes the transmission cost $c$, the transmission history $\{d(\tau)\}_{\tau=1}^t$, and the channel state pattern $\{s(\tau)\}_{\tau=1}^t$, while the time horizon $T$ and the future channel state $\{s(\tau)\}_{\tau={t+1}}^T$ is unknown.
Conversely, an \emph{offline} scheduling algorithm has the information about the connectivity pattern $\mathbf{s}$ (and the time horizon $T$) beforehand.

Our goal is to develop an online algorithm $\pi$  that minimizes the total cost given a channel state pattern $\mathbf{s}$:
\begin{subequations}
    \begin{align}
     & \underset{d(t)}{\min}  & & \sum\limits_{t = 1}^T {(c \cdot d(t) + a(t))} \\
    & \text{s.t.}  & & d(t) \in \{0,1\} \ \text{for} \ t = 1,2, \dots,T; \label{eq:problem_cons1}\\ 
    & & & a(t) \ \text{evolves as Eq.~\eqref{eq:AoI_evo}}  \ \text{for} \ t = 1,2, \dots,T.
    \label{eq:problem_cons2}
    \end{align}
    \label{problem}
\end{subequations} 
\!\!In Problem~\eqref{problem}, the only decision variables are the transmission decisions $\{ d(t)\} _{t = 1}^T$, and the objective function is a non-linear function of $\{ d(t)\} _{t = 1}^T$ due to the dependence of $a(t)$ on $d(t)$. 
Specifically, based on Eq.~\eqref{eq:AoI_evo}, we can rewrite the AoI at time $t$ as $a(t) = (1 - s(t) \cdot d(t)) \cdot (a(t - 1) + 1)$. Upon rephrasing $a(t)$ with the transmission decisions $\{ d(\tau)\} _{\tau = 1}^t$ (i.e., rewriting $a(t-1)$ with $d(t-1)$ and $a(t-2)$, rewriting $a(t-2)$ with $d(t-2)$ and $a(t-3)$, and so forth), we can observe that $a(t)$ involves the products of the current transmission decision $d(t)$ and the previous transmission decisions  $d(\tau )$ for $\tau \in [1, t-1]$, which indicates that $a(t)$ is not linear with respect to $\{ d(\tau)\} _{\tau = 1}^t$.
This non-linearity poses a challenge to its efficient solutions.
In Section~\ref{subsec:problem_reformulation}, following a similar line of analysis as in \cite{8849808}, we reformulate Problem~\eqref{problem} to an equivalent TCP acknowledgment (ACK) problem, which is linear and can be solved efficiently (e.g., via the primal-dual approach \cite{buchbinder2009design}).
Furthermore, to generalize the problem, constraint~\eqref{eq:problem_cons1} allows the transmission decision to be made on OFF channels, although one desired algorithm transmits only on ON channels.

To measure the performance of an online algorithm, we use the metric
\textit{competitive ratio} (CR) \cite{buchbinder2009design}, which is defined as the worst-case ratio of the cost under the online algorithm to the cost of the optimal offline algorithm.
Formally, we say that an online algorithm $\pi$
is \textit{$\beta$-competitive} if there exists a constant $\beta\geq1$ such that for any channel state pattern $\mathbf{s}$, 
\begin{equation}
    C({\bf{s}},\pi ) \le \beta  \cdot OPT({\bf{s}}),
\end{equation}
where $OPT({\bf{s}})$ is the cost of the optimal offline algorithm for the given channel state $\mathbf{s}$. 
We desire to develop an online algorithm with a CR close to $1$, which implies that our online algorithm performs closely to the optimal offline algorithm.
\section{Robust Online Algorithm}
\label{sec:online_algorithm}
In this section, we first reformulate our AoI Problem~\eqref{problem} to an equivalent linear TCP ACK problem. Then, this TCP ACK problem is further relaxed to a linear primal-dual-based program. 
Finally,  a $3$-competitive online algorithm is developed to solve the linear primal-dual-based program.

\subsection{Problem Reformulation}
\label{subsec:problem_reformulation}
In \cite{8849808}, the authors study the same non-linear Problem~\eqref{problem} and reformulate it to an equivalent linear problem. 
Following a similar line of analysis as in \cite{8849808}, we reformulate the non-linear Problem~\eqref{problem} to an equivalent linear TCP ACK Problem~\eqref{prob:integerl-program} as follows.
Consider a TCP ACK problem,
where the source reliably generates and delivers one packet to the destination in each slot $t=1,2,\ldots, T$. 
Those delivered packets need to be acknowledged (for simplicity, we use ``acked" instead of ``acknowledged" throughout the paper) that they are received by the destination, which requires the destination to send ACK packets (for brevity, we call it ACK) back to the source. 
We use $d(t)\in\{0,1\}$ to denote the ACK decision made by the destination at slot $t$.
Let $z_i(t)\in\{0,1\}$ represent whether packet $i$ (i.e., the packet sent at slot $i$) has been acked by  slot $t$ ($i\leq t$), where $z_i(t) =1$ if packet $i$ is not acked by slot $t$ and $z_i(t) =0$ otherwise. Once packet $i$ is acked at slot $t$, then it is acked forever after slot $t$, i.e.,  $z_i(\tau) = 0$ for all $i \le t$ and all $\tau \ge t$.
The feedback channel is unreliable and its channel state in slot $t$ is modeled by an ON/OFF binary variable $s(t)\in\{0,1\}$. 
We use $\mathbf{s} = \left\{ {s(1), \dots ,s(T)} \right\}$ to denote the entire feedback channel states.
The destination can access the feedback channel state $s(t)$ at the start of each slot $t$. 
When the feedback channel is ON and the destination decides to send an ACK, all previous packets are acked, i.e., the number of unacked packets becomes $0$; otherwise, the number of unacked packets increases by $1$. We can see that the dynamic of the number of unacked packets is the same as the AoI dynamic.

We assume a holding cost at each slot, which is the number of unacked packets in that slot.
In addition, we also assume that each ACK has an ACK cost of $c$. 
The goal of the TCP ACK problem is to develop an online scheduling algorithm $\pi  = \left\{ {{d}(t)} \right\}_{t = 1}^T$ that minimizes the total cost given a feedback channel state pattern $\mathbf{s}$:
\begin{subequations}
    \begin{align}
     & \hspace{-0.3cm} \underset{d(t), z_i(t)}{\min}  \hspace{0.1cm} \sum_{t=1}^{T}\left(c \cdot d(t)+\sum_{i=1}^{t} z_{i}(t)\right) 
    \label{eq:primal-obj} \\
    & \text{s.t. }  \hspace{0.25cm}  \nonumber z_{i}(t)+\sum\nolimits_{\tau=i}^{t}s(\tau) d(\tau) \ge 1 \\
    &  \hspace{2cm} \text{for} \ i \leq t \ \text{and}  \ t=1,2,\ldots,T;  \label{eq:primal-con1} \\
    & \hspace{0.8cm}  d(t), z_i(t) \in \{0,1\}  \ \text{for} \  i \le t \  \text{and} \ t=1,2,\ldots,T,
    \label{eq:primal-con2}
    \end{align}
    \label{prob:integerl-program}
\end{subequations}
where the first item $c \cdot d(t)$ in Eq.~\eqref{eq:primal-obj} is the ACK cost at slot $t$, the second item $\sum\nolimits_{i = 1}^t {{z_i}(t)} $ in Eq.~\eqref{eq:primal-obj} is the holding cost at slot $t$. Constraint~\eqref{eq:primal-con1} states that for packet $i$ at slot $t$, either this packet is not acked (i.e., $z_{i}(t) = 1$) or an ACK was made since its arrival (i.e., $s(\tau )d(\tau ) = 1$ for some $i \le \tau  \le t$). 
While Problem~\eqref{prob:integerl-program} is an integer linear problem,
we demonstrate its equivalence to Problem~\eqref{problem} in the following.

\begin{lemma}
    Problem~\eqref{prob:integerl-program} is equivalent to Problem~\eqref{problem}.
\label{lemma:problem_equivalence}
\end{lemma}
We provide detailed proof in Appendix~\ref{appendix:proof_of_equivalence} and give a proof sketch as follows. 
We can show that: (i) any feasible solution to Problem~\eqref{problem} can be converted to a feasible solution to Problem~\eqref{prob:integerl-program}, and the total costs of these two solutions are the same; (ii) any feasible solution to Problem~\eqref{prob:integerl-program} can be converted to a feasible solution to Problem~\eqref{problem}, and the total cost of the converted solution to Problem~\eqref{problem} is no greater than the total cost of the solution to Problem~\eqref{prob:integerl-program}. This implies that any optimal solution to Problem~\eqref{problem} is also an optimal solution to Problem~\eqref{prob:integerl-program}, and vice versa. Therefore, these two problems are equivalent \cite[Sec. 4.1.3]{boyd2004convex}.

To obtain a linear program of the integer Problem~\eqref{prob:integerl-program}, we relax the integer requirement to real numbers:
\begin{subequations}
    \begin{align}
    & \hspace{-0.2cm} \underset{d(t), z_{i}(t)}{\min} \hspace{0.1cm} \sum_{t=1}^{T}\left(c \cdot d(t)+\sum_{i=1}^{t} z_{i}(t)\right) 
    \label{eq:primal-obj_relax}\\
    & \text{s.t. } \hspace{0.25cm} \nonumber  z_{i}(t)  +\sum\nolimits_{\tau=i}^{t}s(\tau) d(\tau) \geq 1 \\
    &\hspace{2cm} \text{for} \ i \leq t \  \text{and}  \ t=1,2,\ldots,T; 
    \label{eq:relax_primal-con1}\\
    & \hspace{0.6cm} \ d(t),  z_{i}(t) \ge 0 \ \text{for} \  i \le t \  \text{and} \ t=1,2,\ldots,T,
    \label{eq:relax_primal-con2}
    \end{align}
    \label{prob:integerl-program_relax}
\end{subequations}
\!which is referred to as the primal problem.
The corresponding dual problem of Problem~\eqref{prob:integerl-program_relax} is as follows: 
\begin{subequations}
    \begin{align}
    & \underset{y_{i}(t)}{\max} & & \sum_{t=1}^{T} \sum_{i=1}^{t} y_{i}(t)  \label{eq:relax_dual-obj} \\
    & \text{s.t.} & & s(t) \sum_{i=1}^{t} \sum_{\tau=t}^{T} y_{i}(\tau)  \leq c \text { for }  t=1,2,\ldots,T ; \label{eq:relax_dual-con1} \\
     & & & y_{i}(t)  \in [0,1] \text{ for} \  i \le t \  \text{and} \ t=1,2,\ldots,T, \label{eq:relax_dual-con2}
    \end{align}
    \label{problem:dual-program_relax}
\end{subequations}
\!\!which has a dual variables $y_i(t)$ for packet $i$ and time $t \ge i$.

\begin{remark}
In \cite{8849808},  the authors mentioned that their reformulated problem is equivalent to the original AoI problem without proof. 
For the sake of completeness, we provide rigorous proof regarding the equivalence between Problem~\eqref{prob:integerl-program} and Problem~\eqref{problem} in Lemma~\ref{lemma:problem_equivalence}.
\end{remark}

\subsection{Primal-dual Online Algorithm Design and Analysis}
\label{subsec:PDOA_analysis}
To solve the primal-dual Problems  (\ref{prob:integerl-program_relax}) and (\ref{problem:dual-program_relax}), we develop the Primal-dual-based Online Algorithm
(PDOA) and present it in Algorithm~\ref{alg:primal-dual}.
The input is the channel state pattern $\bf{s}$ (revealed in an online manner), and the outputs are the primal variables $d(t)$ and $z_i(t)$, and the dual variable $y_i(t)$. Two auxiliary variables $L$ and $M$ are also introduced: $L$ denotes the time when the latest ACK was made, and $M$ denotes the ACK marker (PDOA should make an ACK when $M \ge 1$). 

PDOA is a threshold-based algorithm. Assuming that the latest ACK was made at slot $L$, when the accumulated holding costs since slot $L+1$ is no smaller than the ACK cost $c$ (i.e., $M \ge 1$), PDOA will make an ACK at the next ON slot $L'$.  
Here, we call the interval $[L + 1, L']$ an ACK interval.
Note that PDOA updates the primal and dual variables only for packets that are not acked in the current ACK interval $[L + 1, L']$.
Specifically, consider packet $i$ that has not been acked by the current slot $t \in [L + 1, L']$: (i) for the primal variable $z_i(t)$, if the threshold is not achieved ($M < 1$) or the channel is OFF at slot $t$, PDOA will update $z_i(t)$ to be $1$ (in Line~\ref{line:online_z_update1} or Line~\ref{line:online_z_update2}, respectively) since packet $i$ is not acked by slot $t$; (ii) for the dual variable $y_i(t)$, 
if packet $i$ is not the last packet in the current ACK interval, PDOA will update $y_i(t)$ to be $1$ to maximize the dual objective function; otherwise, PDOA will update $y_i(t)$ to $c - c \cdot M$ to ensure that when the threshold is achieved ($M \ge 1$),  the sum of all the dual variables in the current ACK interval is exactly $c$. 

\begin{algorithm}[!t]
\LinesNumbered
\caption{Primal-dual-based Online Algorithm (PDOA)}
\label{alg:primal-dual}
\Input{$c$, $\mathbf{s}$ (revealed in an online manner)}
\Output{$d(t),z_i(t), y_i(t)$}
\Init{$d(t),z_i(t), y_i(t), L,M \gets 0$ for all $i$ and $t$}
\For{$t = 1$ \KwTo $T$}{
    { \small{\Comments{Iterate all the packets arriving since the latest ACK time $L$.}}
    \For{$i = L + 1$ \KwTo $t$ \label{line:inner_interation_begins}}{
       \If(\tcc*[f]{Not ready to ACK}){$M  < 1$ \label{line:d_less_than_1_begin}}{
        ${z_i}(t) \leftarrow 1$\; \label{line:online_z_update1}
        $M \leftarrow M + 1/c$\;
        ${y_i}(t) \leftarrow \min\{1, c- c \cdot M\}$\;
       }\label{line:d_less_than_1_end}
       \If(\tcc*[f]{Ready to ACK}){$M \ge 1$}{
           \eIf(\tcc*[f]{ON channel}){$s(t) = 1$  \label{line:d_larger_than_1_begin} }{
           $d(t) \leftarrow 1$\;
           $M \leftarrow 0$\;
           $L \leftarrow t$\;
           break and go to the next slot (i.e., $t+1$)\; \label{line:slot_ON}
           }(\tcc*[f]{OFF channel}){ 
           ${z_i}(t) \leftarrow 1$\;  \label{line:online_z_update2}
           \label{line:OFF_primal}
           }
        }\label{line:d_larger_than_1_end}
    } \label{line:inner_interation_end} 
    \small{\Comments{At the end of slot $t$, update dual variable $y_i(t)$ with $s(i) = 0$ as:}}
    \For{$i = t$ \textbf{\textit{decrease}} \KwTo $L + 1$ \label{line:final_update_begin}}{
    \eIf{$s(i) = 0$}{
    \If{${y_i}(t) = 0$}{${y_i}(t) \leftarrow 1$\;}
    } {break and go to the next slot;}
    }
    \label{line:final_update_end}}
}
\end{algorithm}

In addition, at the end of slot $t$ (i.e., Lines~\ref{line:final_update_begin}-\ref{line:final_update_end}), 
if the channel is ON at slot $t$, PDOA  will skip slot $t$ and go to slot $t+1$. Otherwise,  the channel is OFF at slot $t$, and assuming that the most recent ON slot is slot $t^\dag \triangleq \max \{\tau: \tau < t \ \text{and} \  s(\tau) = 1 \}$ ($t^\dag \in [L,t)$), then the channels are OFF during $[t^\dag + 1, t]$. To maximize the dual objective function, PDOA updates the dual variables of packet $t^\dag + 1$ to packet $t$ to be $1$ since the channels are OFF during $[t^\dag + 1, t]$ and the updating of their dual variables does not violate constraint~\eqref{eq:relax_dual-con1} 
(see an illustration in Fig.~\ref{fig:online_dual}).
Note that there may be some OFF channels before slot $t^\dag$, but PDOA does not update their dual variables to avoid the violation of constraint~\eqref{eq:relax_dual-con1}. For example, assuming that slot $t'$ ($t' < t^\dag$) is an OFF channel and we let ${y_{t'}}(t) = 1$. Letting ${y_{t'}}(t) = 1$ has no effect on the constraint~\eqref{eq:relax_dual-con1} at slot $t'$ since we always have $s(t')\sum_{i=1}^{t'} \sum_{\tau=t'}^{T} y_{i}(\tau) = 0$, but doing this does impact the constraint~\eqref{eq:relax_dual-con1} at slot $t^\dag$ (i.e., increasing $s(t^\dag)\sum_{i=1}^{t^\dag} \sum_{\tau=t^\dag}^{T} y_{i}(\tau)$ by $1$ because ${y_{t'}}(t)$ is a part of $s(t^\dag)\sum_{i=1}^{t^\dag} \sum_{\tau=t^\dag}^{T} y_{i}(\tau)$ as $t' < t^\dag$ and $t^\dag \le t$), possibly making constraint~\eqref{eq:relax_dual-con1} at slot $t^\dag$ violated.

\begin{figure}[!t]
	\centering
	\subfigure[Primal variables $z_i(t)$ updates.]{
         \includegraphics[scale = 0.47]{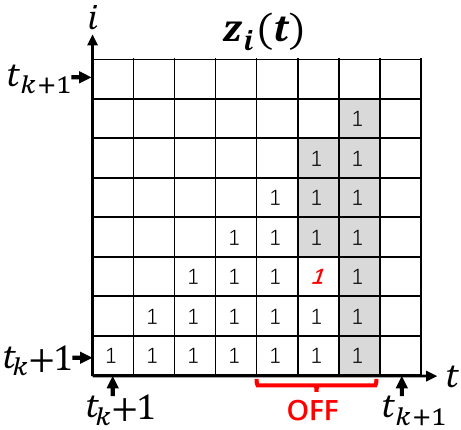}
         \label{fig:online_primal}}
         \qquad
         \subfigure[Dual variables $y_i(t)$ updates. ]{\includegraphics[scale = 0.47]{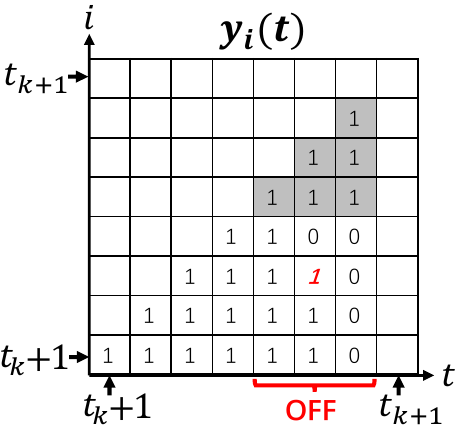}
         \label{fig:online_dual}}
         \caption{The updates of primal variables $z_i(t)$ and dual variables $y_i(t)$ in the $k$-th ACK interval $[t_k+1, t_{k+1}]$ of PDOA, where channels are OFF during $[t_k+5, t_k+7]$. The x-axis represents time and the y-axis represents the packet id. PDOA makes two ACKs at slot $t_k$ and slot $t_{k+1}$, where the ACK cost $c = 18$. 
         The primal variables $z_i(t)$ and dual variables $y_i(t)$ are updated from slot $t_k+1$ to slot $t_{k+1}$; and in slot $t$, packets are updated from packet $t_k+1$ to packet $t$. 
         The red bold italic $1$ denotes when the ACK marker equals or is larger than $1$. 
         In Fig.~\ref{fig:online_primal}, the grey areas denote the updates due to Line~\ref{line:OFF_primal}; In Fig.~\ref{fig:online_dual}, the grey areas denote the updates due to Lines~\ref{line:final_update_begin}-\ref{line:final_update_end}.}
         \label{fig:online_primal_dual_updates}
         \vspace{-10pt}
\end{figure}

\begin{theorem}
    PDOA is 3-competitive. 
    \label{them:online_3}
\end{theorem}
We provide detailed proof in Section~\ref{sec:keyProof} and explain
the key ideas as follows. We first show that given any channel state $\bf{s}$, 
PDOA produces a feasible solution to primal Problem \eqref{prob:integerl-program_relax} and dual Problem \eqref{problem:dual-program_relax}. 
Then, we show that in any $k$-th ACK interval, the ratio between the primal objective value and the dual objective value (denoted by $P(k)$ and $D(k)$, respectively) is at most $3$, i.e., $P(k)/D(k) \le 3$. This implies that the ratio between the total primal objective value (denoted by $P$) and the total dual objective value (denoted by $D$) is also at most $3$, i.e., $P/D \le 3$.  
By the weak duality \cite{buchbinder2009design}, PDOA is 3-competitive.

\begin{remark}
\label{remark:wrong_cr} 
In \cite{8849808}, the authors also propose a primal-dual-based online algorithm and show that their algorithm achieves a CR of $e/(e-1)$. 
However, this CR is achieved only in an (unrealistic) asymptotic setting (i.e., when the transmission cost $c$ goes to infinity but the time horizon $T$ is finite). 
Specifically, in their algorithm, to maximize the dual objective function, at the end of the last slot $T$, they update certain dual variables $y_i(t)$ that arrived at the OFF slot to be $1$. 
In their analysis (the proof of their Theorem~$7$), they show that because of those dual variables updates at the end of slot $T$, the primal objective value satisfies 
$P \le (1 + 1/({(1 + 1/c)^{\left\lfloor c \right\rfloor }} - 1)) \cdot D + (T(T + 1)/2) \cdot (D/c)$.
When $c$ goes to infinity, their CR becomes $P/D \le e/(e - 1)$ as $(T(T + 1)/2)/c$ goes to $0$ since $T$ is finite. 
However, the optimization problem becomes trivial in this setting since an optimal algorithm is simply not to transmit at all given that the transmission cost $c$ can significantly exceed the total staleness cost (at most $(T(T + 1)/2)$). 
Furthermore, when $c$ is finite, their CR is a quadratic function of the time horizon $T$, which can be very large when $T$ is large. 
Instead, our analysis holds for any $T$ and $c$. In our algorithm, rather than updating these dual variables $y_i(t)$ at the end of slot $T$, we directly update them only 
in the current ACK interval (i.e., Lines~\ref{line:final_update_begin}-\ref{line:final_update_end}), ensuring that the dual constraint~\eqref{eq:relax_dual-con1} is satisfied and the dual objective function is as large as possible.
This enables us to focus on the analysis of $P(k)/D(k)$ in the current ACK interval and show that $P(k)/D(k) \le 3$ for any $k$-th ACK interval, which implies that PDOA is 3-competitive.
\end{remark}
\section{Learning-augmented Online Algorithm}
\label{sec:ml_augmented_algorithm}
Online algorithms are known for their robustness against worst-case scenarios, but they can be overly conservative and may have a poor average performance in typical scenarios. 
In contrast, ML algorithms leverage historical data to train models that excel in average cases. However, they typically lack worst-case performance guarantees when facing distribution shifts or outliers. To attain the best of both worlds, we design a learning-augmented online algorithm that achieves both consistency and robustness.

\subsection{Machine Learning Predictions}

We consider a scenario where the device is provided with ML predictions generated by an external ML algorithm (i.e., the device does not need to run the ML model or generate the predictions itself).
The ML prediction $\mathcal{P} \triangleq \{ {p_1},{p_2},\dots,{p_n}\} $ represents the times to transmit an ACK for the destination (i.e., the prediction $\mathcal{P}$ makes a total of $n$ ACKs and sends the $i$-th ACK at slot $p_i$).
The prediction $\mathcal{P}$ is unaware of the channel state pattern $\mathbf{s}$ and can be provided either in full in the beginning (i.e., $t =0$) or one-by-one in each slot. 
Furthermore, when the prediction $\mathcal{P}$ decides to send an ACK at an OFF slot, we will simply ignore the decision for this particular slot. 

Provided with the prediction $\mathcal{P}$, we specify a trust parameter $\lambda \in (0,1]$ to reflect our confidence in the prediction $\mathcal{P}$: a smaller $\lambda$ means higher confidence. 
The learning-augmented online algorithm takes a prediction $\mathcal{P}$, a trust parameter $\lambda$, and a channel state pattern $\mathbf{s}$ (revealed in an online manner) as inputs, and outputs a solution with a cost of $C(\mathbf{s}, \mathcal{P},\lambda)$.
A learning-augmented online algorithm is said  $\beta(\lambda)$-\textit{robust} ($\beta(\lambda) \ge 1$) and $\gamma(\lambda)$-\textit{consistent} ($\gamma(\lambda) \ge 1$) if its cost satisfies
\begin{equation}
   C(\mathbf{s}, \mathcal{P},\lambda) \le \min \{ \beta(\lambda ) \cdot OPT(\mathbf{s}), \gamma(\lambda ) \cdot C(\mathbf{s},\mathcal{P})\},
\end{equation}
where $OPT(\mathbf{s})$ and $C(\mathbf{s},\mathcal{P})$ is the cost of the optimal offline algorithm and the cost of purely following the prediction $\mathcal{P}$ under the channel state pattern $\mathbf{s}$, respectively.

We aim to design a learning-augmented online algorithm for primal Problem~\eqref{prob:integerl-program_relax} that exhibits two desired properties (i) \textit{consistency}: when the ML prediction $\mathcal{P}$ is accurate ($C(\mathbf{s},\mathcal{P}) \approx OPT(\mathbf{s})$) and trusted, our learning-augmented online algorithm performs closely to the optimal offline algorithm (i.e., $\gamma (\lambda ) \to 1$ as $\lambda \to 0$); 
and (ii) \textit{robustness}: even if the ML prediction $\mathcal{P}$ is inaccurate, our learning-augmented algorithm still retains a worst-case guarantee (i.e., $C(\mathbf{s}, \mathcal{P},\lambda) \le \beta(\lambda ) \cdot OPT(\mathbf{s})$ for any prediction $\mathcal{P}$).

\begin{algorithm}[t!]
\LinesNumbered
\caption{Learning-augmented Primal-dual-based Online Algorithm (LAPDOA)}
\label{alg:ml_primal-dual}
\Input{$c$, $\mathcal{P}, \lambda$, $\mathbf{s}$ (revealed in an online manner)}
\Output{$d(t),z_i(t), y_i(t)$}
\Init{$d(t),z_i(t), y_i(t), L,M \gets 0$ for all $i$ and $t$}
\For{$t = 1$ \KwTo $T$}{
    { \small{\Comments{Iterate all the packets arriving since the most recent ACK time $L$.}}
    \For{$i = L + 1$ \KwTo $t$ \label{line:ml_inner_interation_begins}}{
       \If(\tcc*[f]{Not ready to ACK}){$M  < 1$ \label{line:ml_d_less_than_1_begin}}{
        \eIf{$t \ge \alpha (i)$} {
        \Comments{\small{Big update: prediction already acked packet $i$}}
       $M' \leftarrow 1/{\lambda c}$, $y' \leftarrow 1$\; 
       } {
       \Comments{\small{Small update: prediction did not ack packet $i$ yet}}
       $M' \leftarrow  \lambda / c$, $y' \leftarrow \lambda$\;
       }
        ${z_i}(t) \leftarrow 1$\;
        $M \leftarrow M + M'$\;
        ${y_i}(t) \leftarrow y'$\;
       \label{line:ml_d_less_than_1_end}}
       \If(\tcc*[f]{Ready to ACK}){$M  \ge 1$}{
       \eIf(\tcc*[f]{ON channel}){$s(t) = 1$}{
       $d(t) \leftarrow 1$\;  \label{line:d_t_update}
       $M \leftarrow 0$\;
       $L \leftarrow t$\;
       break and go to the next slot (i.e., $t+1$)\;
       }(\tcc*[f]{OFF channel}){\If{${z_i}(t) \ne 1$}{
       \Comments{\small{Zero update}}
       ${z_i}(t) \leftarrow 1$\;
       \label{line:ml_OFF_primal}
       }}
       }}      \label{line:ml_inner_interation_end} 
    \small{\Comments{At the end of slot $t$, update dual variable $y_i(t)$ with $s(i) = 0$ as:}}
    \For{$i = t$ \textbf{\textit{decrease}} \KwTo $L + 1$ \label{line:ml_final_update_begin}}{
    \eIf{$s(i) = 0$}{
    \If{${y_i}(t) = 0$}{${y_i}(t) \leftarrow 1$\;}
    } {break and go to the next slot;}
    }
    \label{line:ml_final_update_end}}
}
\end{algorithm}

\subsection{Learning-augmented Online  Algorithm Design}

We present our Learning-augmented Primal-dual-based Online Algorithm (LAPDOA) in Algorithm~\ref{alg:ml_primal-dual}. LAPDOA behaves similarly to PDOA, but the updates of primal variables and dual variables incorporate the ML prediction $\mathcal{P}$.

In LAPDOA, two additional auxiliary variables $M'$ and $y'$ are used, where $M'$ denotes the increment of the ACK marker $M$ and $y'$ denotes the increment of the dual variables $y_i(t)$ in each iteration of update.
Assuming that the current time is $t$, let $\alpha(t)$ denote the next time when the prediction $\mathcal{P}$ sends an ACK (i.e., $\alpha (t) \triangleq \min \{ {p_i}:{p_i} \ge t\} $ and $\alpha (t) = \infty $ if $t > {p_n}$). 
For the updates of primal and dual variables of an unacked packet $i$ at slot $t$, based on the relationship between the current time $t$ and $\alpha(i)$ (which is also the time when the prediction $\mathcal{P}$ makes an ACK for packet $i$ because packet $i$ arrives at slot $i$),  we classify them into three types: 
\begin{itemize}
    \item Big updates: those updates make $M' \gets 1/\lambda c$, $y' \gets 1$, and $z_i(t) \gets 1$. The big updates are made when LAPDOA is behind the ACK scheduled by the prediction $\mathcal{P}$ (i.e., $t \ge \alpha (i)$), and it tries to catch up the prediction $\mathcal{P}$ by making a big increase in the ACK marker. 
    \item Small updates: those updates make $M' \gets \lambda / c$, $y' \gets \lambda$, and $z_i(t) \gets 1$. The small updates are made when LAPDOA is ahead of the ACK scheduled by the prediction $\mathcal{P}$ (i.e., $t < \alpha (i)$), and LAPDOA tries to slow down its ACK rate by making a small increase in the ACK marker.
    \item Zero updates: those updates make $M' \gets 0$, $y' \gets 0$, and $z_i(t) \gets 1$. The zero updates are made when LAPDOA is supposed to ACK at some slot $t'$ but finds that slot $t'$ is OFF, and it has to delay its ACK to the next ON slot and pay the holding cost (i.e., $z_i(t) =1$) along the way. 
\end{itemize} 
An illustration of these three types of updates is in Fig.~\ref{fig:ml_primal_dual_updates}.

\begin{figure*}[!t]
	\centering
	\subfigure[Auxiliary variables $M'$ updates.]{
         \includegraphics[scale = 0.55]{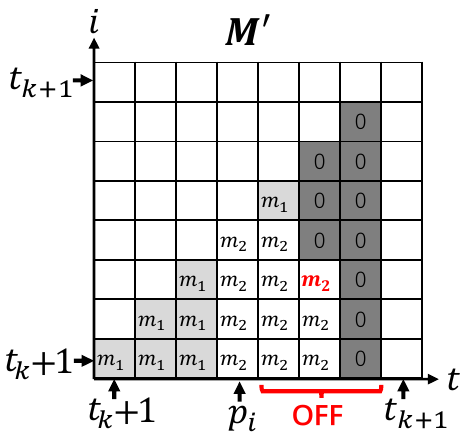}
         \label{fig:ml_primal_d}}
         \quad \quad
         \subfigure[Primal variables $z_i(t)$ 
 updates. ]{\includegraphics[scale = 0.55]{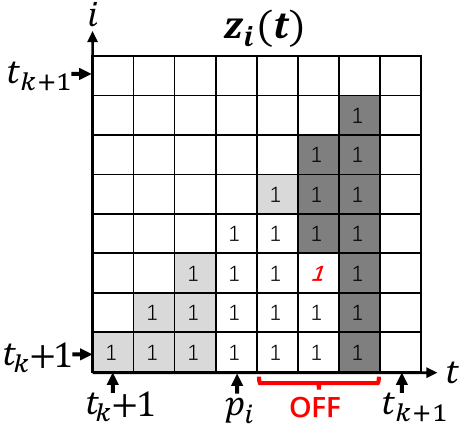}
         \label{fig:ml_primal_z}}
         \quad \quad
         \subfigure[Dual variables $y_i(t)$ updates.]{
         \includegraphics[scale = 0.55]{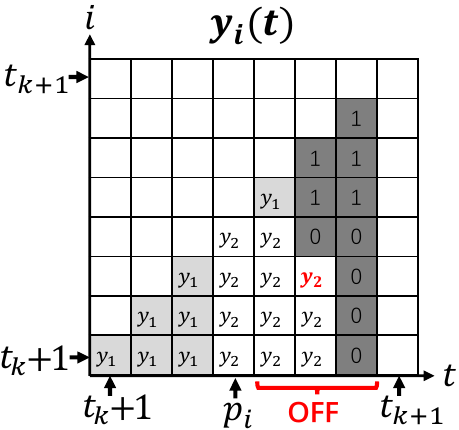}
         \label{fig:ml_dual}}
         \caption{The updates of variables in the $k$-th ACK interval $[t_k+1, t_{k+1}]$ of LAPDOA, where channels are OFF during  $[t_k+5, t_k+7]$. LAPDOA makes two ACKs at $t_k$ and $t_{k+1}$, and the ML prediction $\mathcal{P}$ makes its $i$-th ACK at slot $t_k+4$. The red bold italic value denotes when the ACK marker $M \ge 1$. 
         Let ${m_1} = \lambda /c$, ${m_2} = 1/\lambda c$, ${y_1} = \lambda $, and ${y_2} = 1$.
         The light grey area denotes the small updates, the white area (without background) denotes the big updates, and the dark grey area denotes the zero updates.
         }   \label{fig:ml_primal_dual_updates}
\end{figure*}

In addition, at the end of slot $t$ (i.e., Lines~\ref{line:ml_final_update_begin}-\ref{line:ml_final_update_end}), if the channel is OFF at slot $t$, similar to PDOA, to maximize the dual objective function, LAPDOA updates the dual variables of the packets that arrive after the most recent ON slot (which is assumed to be slot $t^\dag \triangleq \max \{\tau: \tau < t \ \text{and} \  s(\tau) = 1 \}$ ($t^\dag \in [L,t)$)) and the updating of their dual variables does not violate constraint~\eqref{eq:relax_dual-con1} since the channels are OFF during $[t^\dag + 1, t]$
(see an illustration in Fig.~\ref{fig:ml_dual}).

\subsection{Learning-augmented Online Algorithm Analysis}
In this subsection, we focus on the consistency and robustness analysis of LAPDOA with $\lambda \in (0, 1]$. The special cases of LAPDOA with $\lambda=0$ and $\lambda=1$ correspond to the cases that LAPDOA follows the prediction $\mathcal{P}$ purely and PDOA, respectively.
It is noteworthy that by choosing different values of $\lambda$, LAPDOA exhibits a crucial trade-off between consistency and robustness.

\begin{theorem}
    For any channel state pattern $\mathbf{s}$,  any prediction $\mathcal{P}$, any parameter $\lambda \in (0,1]$, and any ACK cost $c$, LAPDOA outputs an almost feasible solution (within a factor of $c/(c+1)$) with a cost of:
    when $\lambda \in (0, 1/c]$, 
    \begin{align}
        \nonumber {C}({\mathbf{s}},{\cal P},\lambda ) \le & \min \{ ({3}/{\lambda }) \cdot (({{c + 1}})/{c}) \cdot OPT({\mathbf{s}}),\\
        & (1 + \lambda ){C_H}({\mathbf{s}},{\cal P}) + {C_A}({\mathbf{s}},{\cal P})\},
    \end{align}
    and when $\lambda \in (1/c, 1]$, 
    \begin{align}
        & C(\mathbf{s}, \mathcal{P}, \lambda) \le \nonumber \min \{ ({3}/{\lambda }) \cdot (({{c + 1}})/{c})  \cdot OPT({\mathbf{s}}), \\
        & (\lambda  + 2)C_H(\mathbf{s},\mathcal{P}) + (1 / \lambda  + 2) \cdot \left\lceil {\lambda c} \right\rceil \cdot C_A(\mathbf{s},\mathcal{P})/{c}\},
    \end{align}
    where $C_A(\mathbf{s},\mathcal{P})$ and $C_H(\mathbf{s},\mathcal{P})$ denote the total ACK costs and total holding costs of prediction $\mathcal{P}$ under $\mathbf{s}$, respectively;
    and $C(\mathbf{s},\mathcal{P}) = C_A(\mathbf{s},\mathcal{P}) + C_H(\mathbf{s},\mathcal{P})$.
    \label{them:ml_primal_dual}
\end{theorem}

Next, we show that LAPDOA has the robustness guarantee in Lemma~\ref{lemma:ml_robustness} and the consistency guarantee in Lemma~\ref{lemma:ml_consistency}.
Combining Lemmas~\ref{lemma:ml_robustness} and \ref{lemma:ml_consistency}, we can conclude Theorem~\ref{them:ml_primal_dual}.

\begin{lemma}
    (Robustness)  For any ON/OFF input instance $\mathbf{s}$,  any prediction $\mathcal{P}$, any parameter $\lambda \in (0,1]$, and any ACK cost $c$, LAPDOA outputs a solution which has a cost of
    \begin{equation}
        C(\mathbf{s}, \mathcal{P}, \lambda)  \le (3/\lambda ) \cdot ((c + 1)/c) OPT(\mathbf{s}).
    \end{equation}
    \label{lemma:ml_robustness}
\end{lemma}

We provide detailed proof in Appendix~\ref{appendix:ml_robustness_proof} and explain
the key ideas as follows. 
We first show that LAPDOA produces a feasible primal solution and an almost feasible dual solution (with a factor of $c/(c+1)$). 
Then, we show that in any $k$-th ACK interval,
LAPDOA  achieves $P(k)/{D}(k) \le 3/\lambda $.
This implies that LAPDOA also achieves $ P/ {D} \le 3/\lambda $ on the entire instance. 
Finally, by scaling down all dual variables ${y_i}(t)$ generated by LAPDOA by a factor of $c/(c + 1)$, we obtain a feasible dual solution with a dual objective value of $(c/(c + 1)) \cdot D$. By the weak duality, we have $P/OPT \le P/((c/(c + 1)) \cdot D) = (P/D) \cdot ((c + 1)/c)  \le  (3/\lambda) \cdot ((c + 1)/c)  $.

\begin{lemma}
    (Consistency)
    For any channel state pattern $\mathbf{s}$,  any prediction $\mathcal{P}$, any parameter $\lambda \in (0,1]$, and any ACK cost $c$, LAPDOA outputs a solution with a cost of:
    when $\lambda \in (0, 1/c]$, 
    \begin{equation}
        C(\mathbf{s}, \mathcal{P}, \lambda) \le (1+\lambda)C_H(\mathbf{s},\mathcal{P}) + C_A(\mathbf{s},\mathcal{P}),
        \label{eq:consistency_1}
    \end{equation}
    and when $\lambda \in (1/c, 1]$, 
    \begin{align}
        & \nonumber C(\mathbf{s}, \mathcal{P}, \lambda)  \le  \\
        & (\lambda  + 2)C_H(\mathbf{s},\mathcal{P}) + (1 / \lambda  + 2) \cdot \left\lceil {\lambda c} \right\rceil \cdot {C_A(\mathbf{s},\mathcal{P})}/{c},\label{eq:consistency_2}
    \end{align}
    where $C_A(\mathbf{s},\mathcal{P})$ and $C_H(\mathbf{s},\mathcal{P})$ denote the total ACK cost and total holding cost of prediction $\mathcal{P}$ under $\mathbf{s}$, respectively; and $C(\mathbf{s},\mathcal{P}) = C_A(\mathbf{s},\mathcal{P}) + C_H(\mathbf{s},\mathcal{P})$.
    \label{lemma:ml_consistency}
\end{lemma} 

We provide detailed proof in Appendix~\ref{appendix:ml_consistency_proof} and give a proof sketch as follows. 
In general,  LAPDOA generates three types of updates: big updates, small updates, and zero updates. 
Our idea is to bound the total cost of each type of update by the cost of the algorithm that purely follows the prediction.
In the case of $\lambda \in (0, 1/c]$, we show that the total number of big updates is $C_A(\mathbf{s},\mathcal{P})/c$, and each big update increases the primal objective value by $c$, so the total cost of big updates in LAPDOA is $C_A(\mathbf{s},\mathcal{P})$. In addition, the total number of small updates and zero updates can be shown to be bounded by $C_H(\mathbf{s},\mathcal{P})$, and each small update or zero update incurs a cost at most $1+\lambda$, thus the total cost of small and zero updates in LAPDOA is at most $(1+\lambda)C_H(\mathbf{s},\mathcal{P})$. In summary, the total cost of LAPDOA in the case of $\lambda \in (0, 1/c]$ is bounded by $(1+\lambda)C_H(\mathbf{s},\mathcal{P}) + ((c+1)/c)C_A(\mathbf{s},\mathcal{P})$. A similar bound can be also obtained for the case of $\lambda \in (1/c, 1]$.

\begin{remark}
     When we trust the ML prediction (i.e., $\lambda \to 0$) and the ML prediction is accurate at the same time ($C(\mathbf{s},\mathcal{P}) \approx OPT(\mathbf{s})$), our learning-augmented algorithm also performs nearly to the optimal offline algorithm, achieving consistency. 
\end{remark}

\begin{remark}
With any $\lambda \in (0,1]$, the CR of LAPDOA is at most $(3/\lambda ) \cdot ((c + 1)/c)$, regardless of the prediction quality. This indicates that our learning-augmented algorithm has
the worst-case performance guarantees, achieving robustness. 
\end{remark}

\section{Numerical Results}
\label{sec:simulation}
In this section, we perform simulations using both synthetic data and real trace data to show that our online algorithm PDOA outperforms the State-of-the-Art online algorithm and that our learning-augmented online algorithm LAPDOA achieves consistency and robustness.

\begin{figure*}[!t]
        \begin{minipage}[t]{0.6\textwidth}
	\centering
	\subfigure[Bernoulli process]{
         \includegraphics[scale = 0.35]{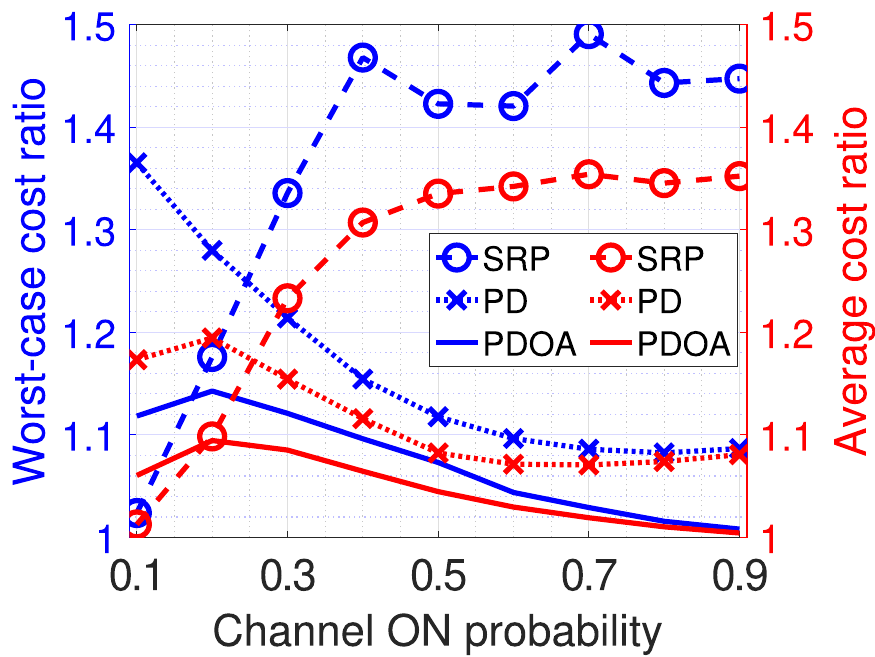}
         \label{fig:simu_online_ber}}
         \subfigure[Real trace dataset]{\includegraphics[scale = 0.35]{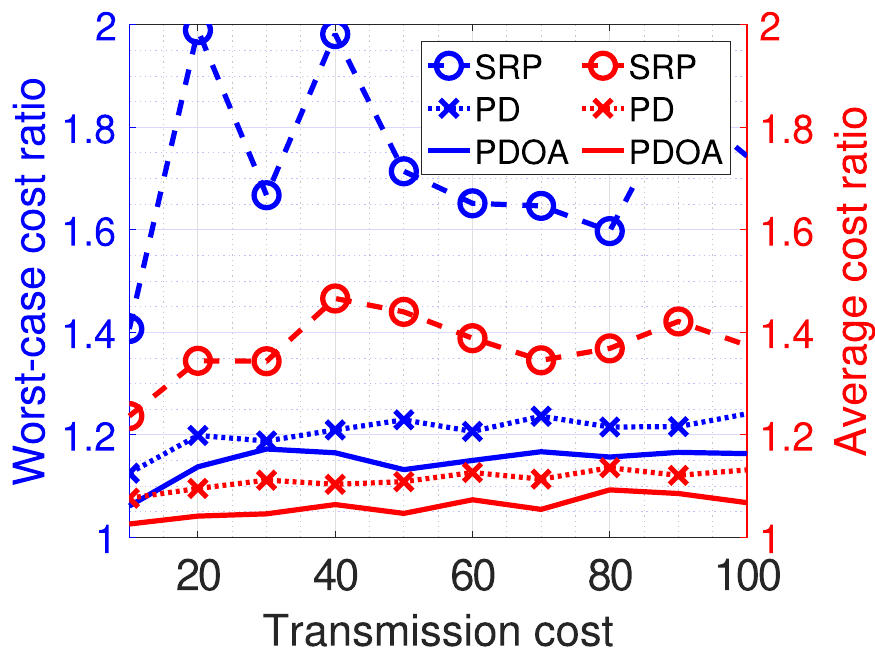}
         \label{fig:simu_online_realTrace}}
         \caption{Performance comparison of online algorithms under different datasets.}
         \label{fig:simu_online}
         \end{minipage}
         \begin{minipage}[t]{0.4\textwidth}
         \centering
         \subfigure{
         \includegraphics[scale = 0.37]{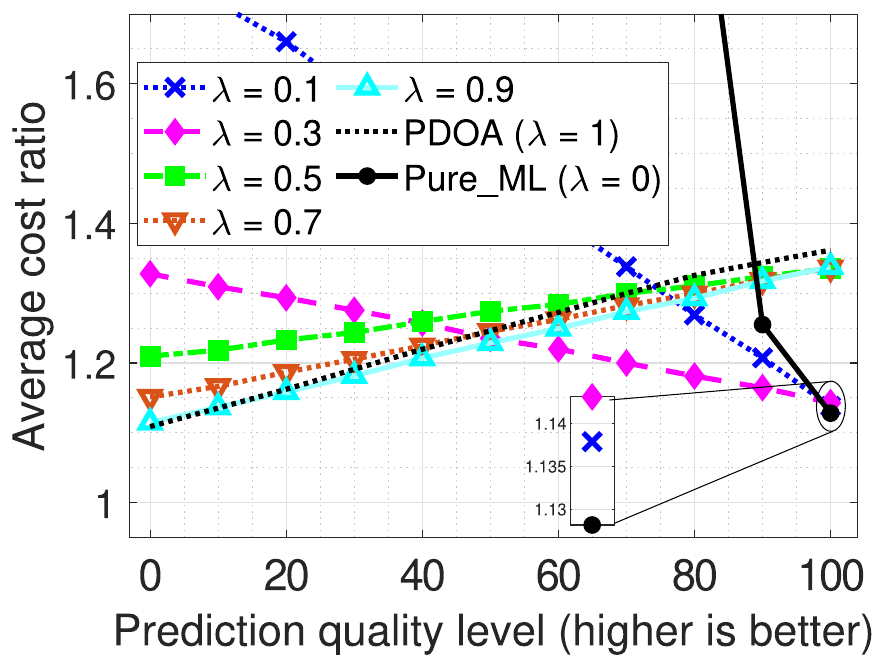}
         }
         \caption{Performance comparison of LAPDOA under different trust parameter $\lambda$ using synthetic dataset (higher prediction quality levels lead to better prediction accuracy).}
         \label{fig:ml_syn}
         \end{minipage}
\end{figure*}

\begin{figure*}[!t]
        \begin{minipage}[t]{0.37\textwidth}
         \centering
         \subfigure{
         \includegraphics[scale = 0.375]{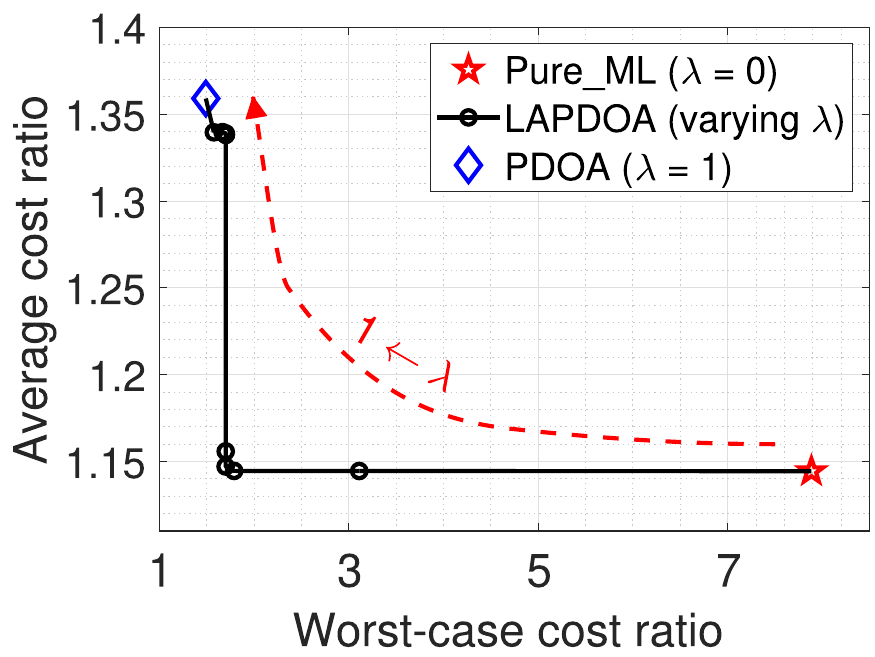}
         }
         \caption{Average cost ratio vs. worst-case cost ratio of LAPDOA when the prediction quality level is $99$. The direction of the dotted arrow line indicates that $\lambda$ becomes larger. }
         \label{fig:ml_em_cr_syn}
         \end{minipage}
         \quad
        \begin{minipage}[t]{0.6\textwidth}
	\centering
	\subfigure[Driving tests (prediction is accurate)]{
         \includegraphics[scale = 0.355]{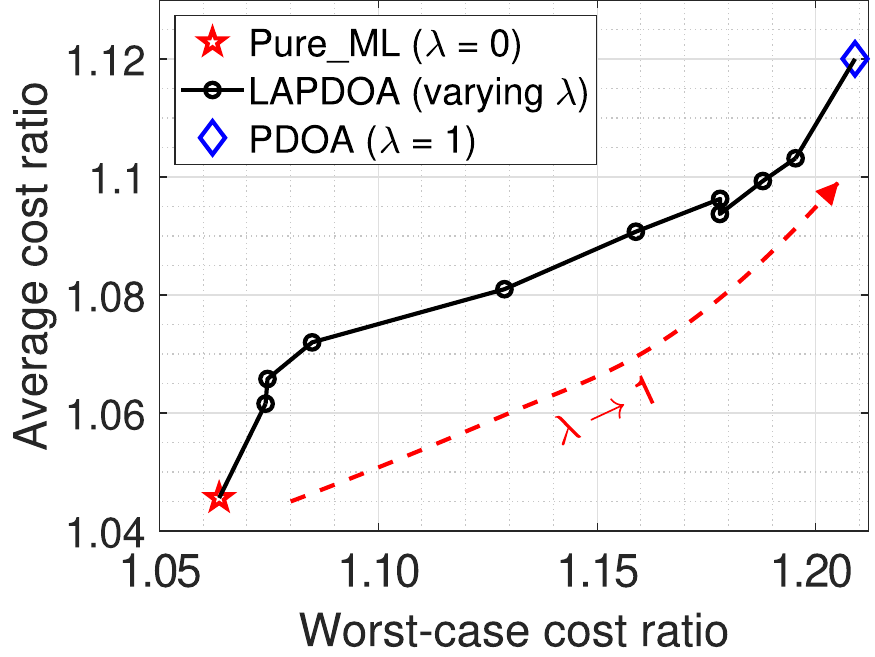}
        \label{fig:ml_trace_small_shift}}
         \subfigure[Walking tests (prediction is inaccurate)]{\includegraphics[scale = 0.355]{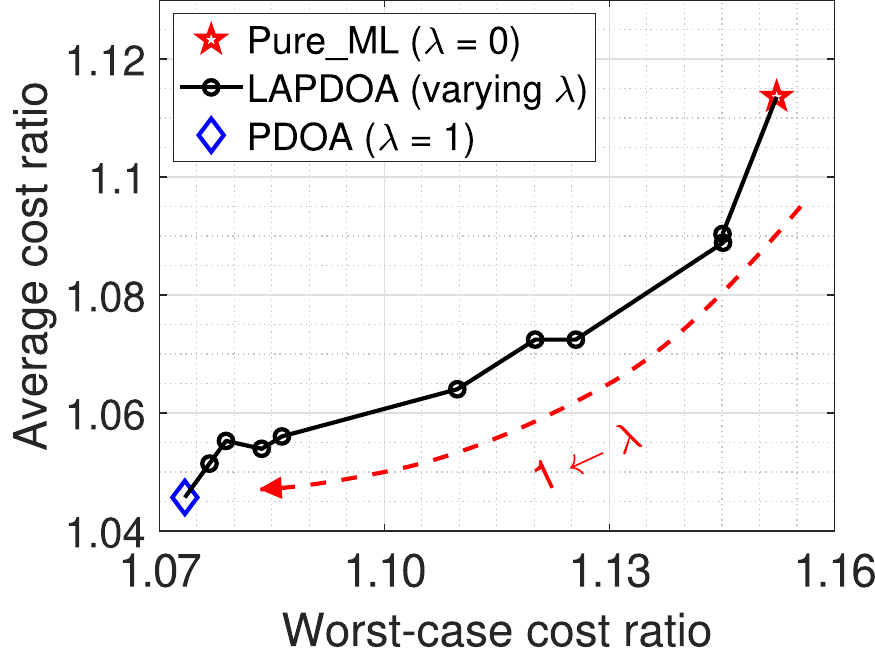}
         \label{fig:ml_trace_big_shift}}
         \caption{Average cost ratio vs. worst-case cost ratio of LAPDOA using real trace dataset. The direction of the dotted arrow line indicates that $\lambda$ becomes larger. The prediction is accurate in Fig.~\ref{fig:ml_trace_small_shift} since the ML model was trained on the driving test dataset.}
         \label{fig:simu_la_cr}
         \end{minipage}
\end{figure*}

\subsection{Online Algorithm}
In Fig.~\ref{fig:simu_online}, we compare PDOA with two online algorithms: (i) the State-of-the-Art online algorithm proposed in \cite{8849808} (which is referred to as ``PD'' in Fig.~\ref{fig:simu_online}), and (ii) the stationary randomized policy proposed in \cite{9488746} (which is referred to as ``SRP'' in Fig.~\ref{fig:simu_online}). Note that SRP in \cite{9488746} cannot be directly applied to our setting due to several differences. First, regarding the system model, \cite{9488746} considers a continuous system where updates are generated with an inter-generation time following a known continuous distribution. Moreover, updates can be transmitted at any time (i.e., there are no OFF channels). Second, in terms of algorithm design and analysis, SRP requires knowledge of the inter-generation time distribution (e.g., mean and variance). Specifically, when there is an update generated at the source, SRP transmits it with a probability $p^* = \min \{ \mu /\sqrt c ,1\}$, where $\mu $ is the mean of inter-generation time, and $c$ is the transmission cost. 
To adapt SRP to our setting, we treat the slots with ON channels as those where updates are generated, requiring SRP to make decisions only during these ON slots. Consequently, the mean of inter-generation time can be determined as $\mu  = T/\sum\nolimits_{t = 1}^T {s(t)} $, which corresponds to the average length between two ON slots.

Two datasets are considered: (i) The synthetic dataset in Fig.~\ref{fig:simu_online_ber}. We adopt the same settings as in \cite{8849808}, where the channel state is a Bernoulli process with varying channel ON probability and the transmission cost $c = 15$. The number of simulation runs is 100, each has 1000 slots; (ii) The real trace dataset \cite{narayanan2020lumos5g} in Fig.~\ref{fig:simu_online_realTrace}.
This dataset contains the channel measurement (i.e.,  reference signal received quality (RSRQ)) of the commercial
mmWave 5G services in a major U.S. city. Specifically, located in the Minneapolis downtown region, the researchers in \cite{narayanan2020lumos5g} repeatedly conduct walking tests on the 1300m loop area.  
Throughout these walking tests, they utilized a 5G monitoring tool installed on an Android smartphone to collect RSRQ information. 
The RSRQ values fluctuate as the tester moves, being higher in proximity to the mmWave 5G tower and decreasing as the tester moves away. 
There are 11 simulation runs, each with 750 slots.
For the ON/OFF channel determination, a threshold is established for the RSRQ (-13dB): the channel is considered ON when the RSRQ exceeds the threshold; otherwise, it is deemed OFF.
Here we vary the transmission cost from $10$ to $100$. 
In both datasets, the performance metrics are the worst-cast cost ratio and the average cost ratio (i.e., the worst cost ratio and the average cost ratio under the online algorithm and the optimal offline algorithm over multiple simulation runs). The worst-case cost ratio serves as a proxy for the CR, as it is impractical to examine the worst-case cost ratio in all possible scenarios. Instead, we use the worst-case cost ratio observed over multiple simulation runs to represent the CR.

Fig.~\ref{fig:simu_online} illustrates that our online algorithm PDOA consistently outperforms the State-of-the-Art online algorithm PD in both datasets, i.e., PDOA achieves a lower worst-case cost ratio and a lower average cost ratio than PD.
Compared to PDOA, SRP slightly outperforms PDOA in Fig.~\ref{fig:simu_online_ber} only when the channel ON probability is low (i.e., $0.1$).  This occurs because, under such conditions, SRP will transmit at every ON slot (as the transmission probability ${p^*} = \min \{ \mu /\sqrt c ,1\}  = \min \{ 10/\sqrt {15} ,1\}  = 1$). In contrast, PDOA may skip transmissions in some ON slots if they are close to the previous ON slot, resulting in a significant increase in AoI since the next ON slot could be much later.
In addition, the worst-case cost ratio of PDOA outperforms the theoretical analysis (with a CR of 3), validating our theoretical results.  

\subsection{Learning-augmented Online Algorithm}

In this subsection, we study the performance of LAPDOA
under different prediction qualities using both synthetic and real trace datasets.
We begin by presenting the training dataset and explaining the process of generating ML predictions based on it. Next, we shift the distribution of the testing dataset to deviate from the training dataset, demonstrating the performance of LAPDOA on these testing datasets.

\subsubsection{Training Dataset and ML Prediction Generation}

\textbf{\textit{The Synthetic Dataset.}}
In Fig.~\ref{fig:simu_online_ber}, PDOA demonstrates strong performance under the Bernoulli process.
However, a specific training dataset reveals its suboptimal performance. 
In this training dataset, the transmission cost $c = 15$, and the channel state sequence is constituted by an independently repeating pattern [$X \times$OFF, $Y\times$ON], where $X \sim B(13,0.9)$ and $Y \sim B(6,0.9)$
($B(n,p)$ represents the binomial distribution with parameters $n$ and $p$). The sequence starts with $X$ OFF slots, followed by $Y$ ON slots, then $X$ OFF slots again, and so on. This pattern characterizes a bursty channel condition.
Under this pattern, in most cases, PDOA only makes one transmission at the first ON slot of these $Y$ ON slots (i.e., after a long consecutive $X$ OFF slots, the ACK marker $M$ will be larger than $1$, and PDOA will transmit at the first ON slot. However, after this transmission, during the short remaining $(Y-1)$ ON slots, the ACK marker $M$ may not be able to be increased to $1$). This results in a high AoI increase for these next $X$ OFF slots. 
For the optimal offline algorithm, to have a lower AoI during OFF slots, it will transmit at both the first ON slot and the last ON slot among those $Y$ ON slots.
To generate a sequence of channel states of the required length, we repeat the pattern enough times independently and concatenate them together.

Recall that LAPDOA incorporates an ML prediction $\mathcal{P}$ that provides the transmission decision at each slot. To generate such an ML prediction $\mathcal{P}$, we train an \textit{Long Short-term Memory
(LSTM)} network, which has three LSTM layers (each layer has $20$ hidden states) followed by one
fully connected layer.
The input of our LSTM network is the current channel state, and its output is the transmission probability at that slot.  
For training, we manually create 300 sequences, each with a length of 100 slots consisting of repeating patterns introduced earlier (we call these constructed sequences ``pattern sequences''). 
Optimal offline transmission decisions for the training datasets are obtained through dynamic programming. 
We use the mean squared error between the LSTM network output and the optimal offline algorithm output as the loss function and employ the Adam optimizer to train the weights.
In the end, to convert the output of our LSTM network (i.e., transmission probability) to the real transmission decisions, a threshold (i.e., $0.5$) is set, and transmission occurs when the output of the LSTM network exceeds the threshold. \footnote{Our ML prediction algorithm generates the full set of predictions $\mathcal{P}$ at the beginning (i.e., $t = 0$).  However, more advanced adaptive ML prediction algorithms could be employed (e.g., predictions could be generated sequentially in each slot based on previous channel states).}

\textbf{\textit{The Real Trace Dataset.}} 
We still use the real trace dataset \cite{narayanan2020lumos5g}. In addition to the walking tests we introduced before, this dataset also contains the RSRQ measurement of the driving tests. Throughout these driving tests, the researchers mounted the smartphone on the car’s windshield and repeatedly drove on the same 1300m loop area to collect RSRQ information. 
Again, we set a threshold for RSRQ ($-13$dB) to determine the ON/OFF channels and let the transmission cost $c = 15$.
The differences between the driving datasets and the walking datasets are that: (i) the time length of one driving loop is much shorter than that of one walking loop (i.e., 250 seconds vs 750 seconds); (ii) the proportion 
of ON slots in the driving dataset is less than that in the walking dataset. As explained in \cite{narayanan2020lumos5g}, this phenomenon primarily arises due to signal attenuation caused by the car's body components, such as windshields or side windows. Additionally, the swift movement of the car leads to frequent handoffs between 5G panels and towers, which further degrades signal strength.

Similar to the synthetic dataset, the ML prediction $\mathcal{P}$ is also generated by an LSTM network (with the same architecture as introduced in the synthetic dataset). To train this LSTM network, we use a $5$ loop of driving tests as our training dataset.

\subsubsection{Results Analysis}

\textbf{\textit{The Synthetic Dataset.}} 
In Fig.~\ref{fig:ml_syn}, we illustrate LAPDOA's performance under varying prediction qualities, influenced by a distribution shift between the training and testing datasets.
The training dataset only contains the sequences fully composed of the pattern (i.e., the percentage of the pattern sequences is $100\%$). However, in the testing dataset, the percentage of the pattern sequence is reduced by replacing some pattern sequences with a Bernoulli process sequence of a length of $100$ with an ON probability of $0.32$ (close to the pattern ON probability).
While the training dataset and the testing dataset share the same channel ON probability, they exhibit shifts in distribution. The magnitude of this shift amplifies as the percentage of the pattern sequence decreases.
To quantify these shifts, we introduce the term ``prediction quality level” in Fig.~\ref{fig:ml_syn}, ranging from $0$ to $100$. The prediction quality level represents the percentage of pattern sequences in the testing dataset (e.g., a prediction quality level of $80$ indicates that $80\%$ of the sequences in the testing dataset are pattern sequences). Thus, a higher prediction quality level means better prediction accuracy.

As we can observe in Fig.~\ref{fig:ml_syn}, when the prediction is accurate (prediction quality level of $100$ or $90$), our trained ML algorithm (``Pure\_ML" in the figure) outperforms PDOA (recall that Pure\_ML is a special case of LAPDOA with $\lambda=0$, and PDOA is a special case of LAPDOA with $\lambda=1$). Learning-augmented algorithms trusting the prediction ($\lambda \in \{0.1, 0.3\}$) closely match the ML algorithm's performance.
Conversely, with an inaccurate prediction (prediction quality level of $0$ or $10$), the ML algorithm performs poorly while PDOA performs well. In this case, learning-augmented algorithms not trusting the prediction ($\lambda \in \{0.7, 0.9\}$)  closely resemble PDOA. Furthermore, with different values of $\lambda$,
LAPDOA provides different tradeoff curves for consistency
and robustness.

Though the trained ML algorithm performs well in the average case when the distribution shift is small (i.e., Pure\_ML achieves a low average cost ratio in Fig.~\ref{fig:ml_syn} when the prediction quality level is high), it may lack worst-case performance guarantees. 
In Fig.~\ref{fig:ml_em_cr_syn}, we show the average cost ratio and the worst-case cost ratio performance of LAPDOA when the prediction quality level is $99$ (i.e., there exists at least a sequence that is not the pattern sequence). 
Here we consider the LAPDOA algorithms with 
$\lambda \in \{0, 0.1,..., 0.9, 1\}$. 
Pure\_ML achieves the smallest average cost ratio; however, its worst-case cost ratio significantly surpasses that of PDOA, indicating that it lacks performance robustness. 
In addition, as the trust parameter  $\lambda$ increases, the worst-case cost ratio performance improves, while the performance of the average cost ratio worsens. 
In this scenario, selecting $\lambda$ as $0.3$ appears to be beneficial, as it not only yields a low worst-case cost ratio but also sustains a low average cost ratio concurrently.

\textbf{\textit{The Real Trace Dataset.}} 
We consider two different testing datasets: (i) a $3$ loop of driving tests in Fig.~\ref{fig:ml_trace_small_shift}, and (ii) a $3$ loop of walking tests in Fig.~\ref{fig:ml_trace_big_shift}. The distribution shift between the first testing dataset and the training dataset is small as the data is collected under the same scenario (i.e., driving), while the distribution shift between the second testing dataset and the training dataset is large as the data is collected under the two different scenarios (i.e., walking vs. driving). 
In Fig.~\ref{fig:ml_trace_small_shift}, when the testing dataset is the driving dataset (indicating accurate predictions due to a small distribution shift from the training dataset), the Pure\_ML demonstrates superior performance not only for average cost ratio but also worst-case cost ratio. 
We conjecture that Pure\_ML achieves a low worst-case cost ratio because this testing dataset is highly identical to the training datasets 
(i.e., with the same $1300$m loop area, the signal strength measured in one driving loop does not appear to change dramatically in another driving loop, and thus those driving loops share a similar signal strength pattern). 
However, in Fig.~\ref{fig:ml_trace_big_shift}, when the predictions are less accurate (due to a significant distribution shift between the training and testing datasets), the performance of Pure\_ML diminishes, resulting in both a high average cost ratio and a high worst-case cost ratio. In contrast, the PDOA online algorithm excels in this scenario in terms of both the average cost ratio and the worst-case cost ratio.
Upon analyzing these two testing datasets, we learn that, on the one hand, 
if we understand the characteristics of the testing dataset, 
we can select our trust parameters correspondingly. For example, if we are aware that the testing dataset deviates from the training dataset greatly, we should choose a lower trust parameter. 
On the other hand, when uncertainty shrouds the testing dataset, selecting an appropriate trust parameter (e.g., $\lambda = 0.4$) enables LAPDOA to strike a good trade-off between consistency and robustness.
\section{\revise{Limitations}}
\label{sec:limitations}

In this section, we discuss the limitations of this work and outline potential improvements for future research. The main limitations lie in the following two categories: (i) system model and (ii) algorithm design. 

\textbf{\textit{Limitations in System Model.}} 
Our system model makes several simplifying assumptions to facilitate algorithm development and performance analysis. However, these assumptions may not fully capture real-world constraints.

(\romannumeral 1) Channel probing overhead. 
We assume the device probes the channel at the beginning of each time slot to determine its state. However, frequent probing is energy-intensive, particularly in resource-constrained environments we consider. Previous studies have shown that adaptive probing strategies, such as Markov decision processes (MDPs) and threshold-based strategies, can leverage historical data and environmental conditions to minimize unnecessary probing. Machine learning techniques have also shown promise in predicting wireless channel states \cite{8395053}. By integrating these adaptive probing techniques or predictive methods into our learning-augmented online algorithm, we may achieve significant energy savings.

(\romannumeral 2) Fixed transmission cost. We consider a scenario where the device operates with a limited battery, and each transmission incurs a fixed cost.  However, a more realistic approach would allow the device to adjust its transmission power dynamically, taking into account factors such as distance, battery level, and network conditions. In this case, the transmission cost would vary based on the transmission power.
Furthermore, in certain cases (e.g., outdoor sensors or cameras), the device can replenish its energy from external sources (e.g., solar, wind, or water). 
This energy replenishment introduces time-varying constraints; for example, the data transmission rate can be higher when energy is abundant and lower when it is scarce.
Consequently, the transmission policy could be adapted to be more dynamic.

(\romannumeral 3) Staleness cost. The staleness cost is captured by the AoI, which increases linearly over time. However, this linear growth may not accurately capture scenarios where prolonged periods without transmission have a disproportionately greater impact on the receiver. In such cases, modeling the staleness cost as a rapidly increasing convex function of AoI could provide a more appropriate representation~\cite{8000687}.

\textbf{\textit{Limitations in Algorithm Design.}} 
We proposed a learning-augmented online algorithm LAPDOA that leverages the strengths of both traditional online algorithms and machine learning techniques. Specifically, we highlighted that by adjusting the trust parameter $\lambda$, LAPDOA achieves a critical balance between consistency and robustness. For example, setting $\lambda = 0$ results in LAPDOA relying solely on the prediction $\mathcal{P}$, while $\lambda = 1$ corresponds to the case where LAPDOA behaves like PDOA. However, the trust parameter $\lambda$ must be predetermined in this work. 
In real-world systems,  the quality of ML predictions may vary, requiring adaptive adjustments to the trust parameter. An intriguing avenue for future research lies in developing methods to dynamically select $\lambda$ to optimize performance. One possible approach is to incorporate feedback loops that continuously evaluate prediction errors and adjust $\lambda$ accordingly, ensuring the model adapts to changing conditions.

\section{Conclusion}
\label{sec:conclusion}
In this paper, we studied the minimization of data freshness and transmission costs under a time-varying wireless channel.
After reformulating our original problem to a TCP ACK problem, we developed a $3$-competitive primal-dual-based online algorithm.
Realizing the pros and cons of online algorithms and ML algorithms, we designed a learning-augmented online algorithm that takes advantage of both approaches and achieves consistency and robustness. 
Finally, simulation results validate the superiority of our online algorithm and highlight the consistency and robustness achieved by our learning-augmented algorithm. 
For future work, one interesting direction would be to consider how to adaptively select the trust parameter $\lambda$ to achieve the best performance.
\bibliographystyle{IEEEtran}
\bibliography{reference}

\section{Practical Applications and Case Studies}
\label{sec:case_study}
In this section, we present two real-world examples to illustrate the tradeoff between transmission cost and data freshness, highlighting the critical need to balance this tradeoff effectively.
We also provide justifications for our modeling assumptions, particularly the fixed transmission cost and the requirement for channel probing at each time slot, demonstrating their relevance in practical scenarios.

\textbf{\textit{Wildfires Real-time Monitoring.}}
Wildfires in the U.S. and around the world are becoming increasingly frequent, costly, and dangerous (for example, the 2025 LA wildfire in the U.S. set a record with $135$ billion dollar in damages \cite{BBC}).  
To address this growing threat, wildfire sensors are typically installed in high-risk areas, such as forests, grasslands, mountainous regions, and the wildland-urban interface. Early detection of wildfires in these areas significantly increases the chances of timely containment and suppression.
In response, the Department of Homeland Security (DHS) Science and Technology Directorate (S\&T)  began efforts at developing wildfire sensors in late 2019 \cite{DHS_ST}.  
Research on wildfire sensors aims to achieve real-time, continuous identification of key elements associated with wildfire conditions, such as temperatures, humidity, particulate matter, volatile organic compounds, and gases. 
Although wildfire monitoring systems may deploy multiple sensors across a region, each sensor functions independently, collecting and transmitting data based on local conditions. Our analysis focuses on a single sensor.
Typically, the wildfire sensor is equipped with internal batteries and utilize LTE cellular network communications for data transmission \cite{DHS_ST}.  
In this case, the wildfire sensor must transmit data strategically to balance the tradeoff between transmission cost (e.g., energy cost) and data freshness. Frequent transmissions can quickly deplete the battery, while infrequent transmissions may result in delayed detection of a wildfire, potentially allowing it to grow into a large-scale disaster. 

For system modeling, we assume that each transmission consumes a fixed amount of energy, resulting in a fixed transmission cost. This assumption holds in wildfire monitoring systems, where the data packets transmitted are typically small and uniform in size.  Additionally, most wildfire sensors often use pre-configured communication modules, such as LTE or low-power wide-area networks (LPWANs) \cite{7815384}, which operate at a fixed power level for each transmission. 
Furthermore, we assume that the sensor probes the channel state at each time slot. Wildfire monitoring systems must ensure that the transmitted data is received successfully. Since network conditions fluctuate due to signal attenuation or congestion, regular channel probing enables real-time monitoring and informed transmission scheduling, ensuring reliable data delivery.

\textbf{\textit{Unmanned Aerial Vehicles (UAVs) System.}} UAVs are versatile tools used for applications such as aerial surveillance, delivery, disaster response, agriculture, environmental monitoring, and media production, offering efficiency and precision in various industries \cite{8682048}. While UAV networks are often used in large-scale operations, we consider a single UAV, which is typically equipped with a limited battery and relies on wireless communication to transmit data to ground controllers.
On the one hand, frequent data transmissions consume significant energy, which is a critical resource for the UAV due to its limited battery capacity. High transmission rates may also cause increased bandwidth usage, potential interference, and higher risks of packet collisions, especially in multi-UAV networks \cite{8961096}.
On the other hand, fresh data is essential for making accurate and timely decisions. 
A lower frequency of data transmission can lead to outdated information at the controller, potentially resulting in suboptimal or even harmful decisions. 
For instance, in surveillance, delayed images or sensor data could result in missed critical events, while outdated GPS data in delivery drones may cause navigation errors.  
Therefore, the UAV needs to transmit data strategically to balance the tradeoff between transmission cost (e.g., energy cost and bandwidth usage) and data freshness.

Similarly, in UAV-based systems, we can assume that each transmission consumes a fixed amount of energy. This is because UAVs often transmit data packets of relatively uniform size, particularly when sending periodic updates (e.g., sensor readings, images, or GPS coordinates).  Additionally, we assume that the sensor probes the channel state at each time slot. Maintaining a reliable communication link with the ground controller is critical for UAV operations. Given the UAV’s mobility and potential environmental interferences, real-time evaluation of the channel state is necessary to determine optimal transmission decisions. Periodic channel probing allows the UAVs to detect signal degradation, assess bandwidth availability, and ensure that critical updates reach the controller with minimal delay.

\section{Proof of Theorem~\ref{them:online_3}}
\label{sec:keyProof}

Our proof outline is as follows. We first demonstrate that PDOA produces a feasible solution to primal Problem \eqref{prob:integerl-program_relax} and dual Problem \eqref{problem:dual-program_relax} in Lemma~\ref{lemma:online_feasible_sol}. Then, we explain the usefulness of the primal-dual problem \cite{buchbinder2009design} for competitive analysis. Finally, we establish that our online primal-dual-based PDOA is $3$-competitive.

To begin with, we first introduce two key observations of PDOA that will be widely used in the proofs.

\begin{observation}
Assuming that PDOA makes the latest ACK at some ON slot $L$ and the current time slot is $t$ ($t > L$), then at slot $t$, before the threshold is achieved ($M < 1$), PDOA updates the primal variable $z_i(t)$ and dual variable $y_i(t)$ of packet $(L+1)$ to packet $t$; however, once the threshold is achieved ($M \ge 1$), PDOA only updates the primal variable $z_i(t)$ of the unacked packets if the channel is OFF at slot $t$.
\label{obs:within_interval}
\end{observation}

\begin{observation}
Once PDOA makes an ACK at some ON slot $t$, all the packets arriving no later than slot $t$ (packet $1$ to packet $t$) are acked forever after slot $t$, and their primal variables and dual variables will never be changed after slot $t$,  i.e.,  $z_i(\tau) = 0$ and $y_i(\tau) = 0$ for all $i \le t$ and all $\tau > t$.
\label{obs:after_ACK}
\end{observation}

With Observations~\ref{obs:within_interval}
and \ref{obs:after_ACK}, we can prove the feasibility of the solutions produced by PDOA in the following lemma.

\begin{lemma}
    PDOA  produces a feasible solution to both primal Problem \eqref{prob:integerl-program_relax} and dual Problem \eqref{problem:dual-program_relax}.
    \label{lemma:online_feasible_sol}
\end{lemma}

\begin{proof}
The primal constraint~\eqref{eq:relax_primal-con2} and the dual constraint~\eqref{eq:relax_dual-con2} are clearly satisfied. 
For the primal constraint~\eqref{eq:relax_primal-con1}, it is easy to verify that for the $i$-th packet at slot $t$ ($i \le t$), if PDOA made an ACK during $[i,t]$, then constraint~\eqref{eq:relax_primal-con1} is satisfied; otherwise, since the $i$-th packet is not acked by slot $t$, PDOA will update ${z_i}(t)$ to be $1$, so constraint~\eqref{eq:relax_primal-con1} is also satisfied.
When the channels are OFF, the dual constraints~\eqref{eq:relax_dual-con1} are automatically satisfied.  
Now, consider an ON slot $t$ and its dual constraint~\eqref{eq:relax_dual-con1}
$\sum_{i=1}^{t} \sum_{\tau=t}^{T} y_{i}(\tau) \leq c$. This constraint requires that for all the packets arriving no later than slot $t$ (packet $1$ to packet $t$),  the sum of their dual variables beyond slot $t$ should not exceed $c$.  
Assuming that this ON slot $t$ falls into the $k$-th ACK interval $[t_k + 1, t_{k+1}]$ of PDOA, i.e., ${t_k} + 1 \le t \le {t_{k + 1}}$, where PDOA makes two ACKs at the ON slot $t_{k+1}$ ($t_{k+1} > t_k + 1$) and the ON slot $t_k$ (when the ON slot $t$ falls into the last ACK interval $[t_K + 1, T]$, where PDOA makes the last ACK at the ON slot $t_K$, our following analysis can be easily extended to this case).  
According to Observations~\ref{obs:within_interval}
and \ref{obs:after_ACK}, packet $1$ to packet $t_k$ are not updated after slot $t_k$, and packet $(t_k+1)$ to packet $t$ are not updated after slot $t_{k+1}$, then we have
\begin{equation}
\begin{aligned}
& \sum\nolimits_{i = 1}^t {\sum\nolimits_{\tau  = t}^T {{y_i}} } (\tau ) \\
= & \sum\nolimits_{i = 1}^{{t_k}} {\sum\nolimits_{\tau  = t}^T {{y_i}} } (\tau ) + \sum\nolimits_{i = {t_k} + 1}^t {\sum\nolimits_{\tau  = t}^{{t_{k + 1}}} {{y_i}} } (\tau ) \\
& + \sum\nolimits_{i = {t_k} + 1}^t {\sum\nolimits_{\tau  = {t_{k + 1}+1}}^T {{y_i}} } (\tau ) \\ 
= & 0 + \sum\nolimits_{i = {t_k} + 1}^t {\sum\nolimits_{\tau  = t}^{{t_{k + 1}}} {{y_i}} } (\tau ) + 0 \\
= & \sum\nolimits_{i = {t_k} + 1}^t {\sum\nolimits_{\tau  = t}^{{t_{k + 1}}} {{y_i}} } (\tau ).
\end{aligned}
\label{eq:dual_var_reduction}
\end{equation}
Next, we discuss the value of $\sum\nolimits_{i = {t_k} + 1}^t {\sum\nolimits_{\tau  = t}^{{t_{k + 1}}} {{y_i}} } (\tau )$ in two cases: (i)  $t = t_k + 1$,  and (ii) $t_k + 1 < t \le t_{k+1}$.

Case i): $t = t_k + 1$. In this case, we have $\sum\nolimits_{i = {t_k} + 1}^t {\sum\nolimits_{\tau  = t}^{{t_{k + 1}}} {{y_i}} } (\tau ) = \sum\nolimits_{\tau  = {t_k} + 1}^{{t_{k + 1}}} {{y_{{t_k} + 1}}(\tau )} $. 
We assume that the threshold is achieved after the updating packet $j$ ($t_k + 1 \le j \le t_{k+1}$) at slot ${t^\dag}$ ($t_k + 1 < {t^\dag} \le t_{k+1}$), i.e., 
\begin{equation}
    \begin{aligned}
        & \sum\nolimits_{\tau  = {t_k} + 1}^{{t_{k + 1}}} {{y_{{t_k} + 1}}(\tau )} \\
        = & \sum\nolimits_{\tau  = {t_k} + 1}^{{t^\dag }} {{y_{{t_k} + 1}}(\tau )}  + \sum\nolimits_{\tau  = {t^\dag } + 1}^{{t_{k + 1}}} {{y_{{t_k} + 1}}(\tau )} \\
        \buildrel (a) \over = & \sum\nolimits_{\tau  = {t_k} + 1}^{{t^\dag }} {{y_{{t_k} + 1}}(\tau )}  + 0 \\
        \le & \sum\nolimits_{\tau  = {t_k} + 1}^{{t^\dag } - 1} {\sum\nolimits_{i = {t_k} + 1}^\tau  {{y_i}(\tau )} }  + \sum\nolimits_{i = {t_k} + 1}^j {{y_i}({t^\dag })}  \\
        = & c, 
    \end{aligned}
\end{equation}
where $(a)$ is because the threshold is achieved at slot  ${t^\dag}$ and packet $(t_k+1)$ is never updated after slot  ${t^\dag}$.

Case ii): $t_k + 1 < t \le t_{k+1}$. 
We assume that during the interval $[t,{t_{k + 1}}]$, all packets arriving between $[t_k+1,t]$ (packet $t_k+1$ to packet $t$) make a total number $m$ updates, i.e.,  $\sum_{i={t_k} + 1}^{t} \sum_{\tau=t}^{{{t_{k + 1}}}} y_{i}(\tau) \le m$.
Next, we discuss the value of $m$ regarding the value of the ACK marker $M$.
On the one hand, consider the case where before $m$ increases to $\left\lceil c \right\rceil  - 1$, the ACK marker $M$ already becomes no smaller than $1$. In this case, all the dual variables of packet $t_k+1$ to packet $t$ are not updated according to Observation~\ref{obs:within_interval}, and we have $\sum\nolimits_{i = {t_k} + 1}^t {\sum\nolimits_{\tau  = t}^{{t_{k + 1}}} {{y_i}} } (\tau ) \le m < \left\lceil c \right\rceil  - 1 < (c + 1) - 1 = c$. 
On the other hand, consider the case where before $m$ increases to $\left\lceil c \right\rceil  - 1$, the ACK marker $M$ is smaller than $1$. 
At the point when $m = \left\lceil c \right\rceil  - 1$, the increment of the ACK marker $M$  due to those $m$ updates (denoted by $M(m)$) is 
$M(m) = (\left\lceil c \right\rceil  - 1)/c$.
Now the ACK marker $M$ becomes $M = N' + M(m) \ge 1/c + M(m) = \left\lceil c \right\rceil /c \ge 1$, where $N'$ is the increment of the ACK marker $M$ due to packet $(t_k + 1)$ to packet $t-1$ before slot $t$ ($N'$ is at least $1/c$ since packet $(t_k + 1)$ is not acked at slot $(t_k + 1)$, which increases  $N'$ by $1/c$). 
Given that the ACK marker $M$ now is no smaller than $1$, all the dual variables of packet $t_k+1$ to packet $t$ are not updated according to Observation~\ref{obs:within_interval}. 
Therefore, we have $\sum\nolimits_{i = {t_k} + 1}^t {\sum\nolimits_{\tau  = t}^{{t_{k + 1}}} {{y_i}} } (\tau ) \le m = \left\lceil c \right\rceil  - 1 < (c + 1) - 1 = c$.

In summary, we have $\sum_{i=1}^{t} \sum_{\tau=t}^{T} y_{i}(\tau) \leq c$ in both cases, thus the dual constraint~\eqref{eq:relax_dual-con1} is satisfied.
\end{proof}

The primal-dual problem allows us to analyze the CR of our online algorithm without knowing the optimal offline solution. As shown in Lemma~\ref{lemma:online_feasible_sol}, for a given channel state $\bf{s}$, our online algorithm outputs an integer feasible solution (denoted by $\pi$) to both primal Problem~\eqref{prob:integerl-program_relax} and dual Problem~\eqref{problem:dual-program_relax}. We use $P(\bf{s}, \pi)$ and $D(\bf{s}, \pi)$ to denote the primal objective value and the dual objective value under $\pi$, respectively.
In addition, this integer solution $\pi$ is also a feasible solution to Problem~\eqref{prob:integerl-program}, and we use $C_{\ref{prob:integerl-program}}({\bf{s}},\pi)$ to denote the objective value of Problem~\eqref{prob:integerl-program} under $\pi$.
Because primal Problem~\eqref{prob:integerl-program_relax} and Problem~\eqref{prob:integerl-program} have the same objective
function, we have $C_{\ref{prob:integerl-program}}({\bf{s}},\pi) = P({\bf{s}}, \pi)$.
The CR of our online algorithm $\pi$ for primal Problem~\eqref{prob:integerl-program_relax} satisfies
\begin{equation}
    \underbrace {\frac{{C_{\ref{prob:integerl-program}}({\bf{s}},\pi)}}{{OPT_{\ref{prob:integerl-program}}({\bf{s}})}}}_{\text{CR of Problem~\eqref{prob:integerl-program}}} \le \underbrace {\frac{{P({\bf{s}},{\pi})}}{{OPT_{\ref{prob:integerl-program_relax}}({\bf{s}})}}}_{\text{CR of primal Problem~\eqref{prob:integerl-program_relax}}} \le \frac{{P({\bf{s}},{\pi})}}{{D({\bf{s}},{\pi})}},
\end{equation}
where $OPT_{\ref{prob:integerl-program}}({\bf{s}})$ and $OPT_{\ref{prob:integerl-program_relax}}({\bf{s}})$ is the cost of the optimal offline algorithm for Problem~\eqref{prob:integerl-program} and primal Problem~\eqref{prob:integerl-program_relax}, respectively. 
Here we have $OPT_{\ref{prob:integerl-program}}({\bf{s}}) \ge OPT_{\ref{prob:integerl-program_relax}}({\bf{s}})$ because the search space of the optimal solution in primal Problem~\eqref{prob:integerl-program_relax} is larger than that in Problem~\eqref{prob:integerl-program}. 
The second inequality comes from the weak duality \cite{buchbinder2009design}. 
Furthermore, if we can show that there exists a constant $\beta$ such that $P({\bf{s}},\pi)/D({\bf{s}},\pi) \le \beta $ holds for any channel state $\mathbf{s}$, then our online algorithm is $\beta$-competitive
for primal Problem~\eqref{prob:integerl-program_relax} and Problem~\eqref{prob:integerl-program}.

Now with Lemma~\ref{lemma:online_feasible_sol}, we are ready to show our main result that PDOA is $3$-competitive in the following.

\begin{proof}[Proof of Theorem~\ref{them:online_3}]
For notational simplicity, let $P$ and $D$ be the value of the objective function of the
primal and the dual solutions produced by PDOA under a given channel state $\bf{s}$, respectively. 
In the following, we show that $ P/ {D} \le 3$. 
We assume that PDOA makes a sequence of ACKs $\pi = \{t_1, t_2, \dots, t_K\}$, where PDOA makes the $i$-th ACK at the ON slot $t_i$
(i.e., $s(t_i)d(t_i) = 1$). 
Our goal is to show that for any $k$-th  ($k \in [0,K]$) ACK interval $[{t_k} + 1,{t_{k+1}}]$ (where the first ACK interval is $[1,{t_{1}}]$ when $k = 0$ and the last ACK interval is $[{t_K} + 1,T]$ when $k= K$), the ratio between the primal objective value and the dual objective value in this $k$-th ACK interval (denoted by $P(k)$ and $D(k)$, respectively) is at most $3$, i.e., $P(k)/D(k) \le 3$. According to Observation~\ref{obs:after_ACK}, $P(k)$ and $D(k)$ are never changed when this ACK interval ends at slot $t_{k+1}$. 
This implies that PDOA also achieves $ P/ {D} \le 3$ on the entire instance $\mathbf{s}$. 
This, along with Lemma~\ref{lemma:online_feasible_sol}, concludes that PDOA is $3$-competitive based on the weak duality \cite{buchbinder2009design}.

We first discuss the relation between $P(k)$ and $D(k)$ in the first $K$ ACK interval $[{t_k} + 1,{t_{k+1}}]$ (i.e., $k \in [0,K - 1]$, where there is always an ACK made at slot $t_{k+1}$), and then discuss the relation between $P(K)$ and $D(K)$ in the last ACK interval $[{t_K} + 1,T]$ (i.e., $k = K$, where it is possible that no ACK is made at slot $T$) in the end.

Consider any $k$-th ($k \in [0,K-1]$) ACK interval $[{t_k} + 1,{t_{k+1}}]$ (denoted by $I_k$), where PDOA makes two ACKs at the ON slots $t_k$ and $t_{k+1}$, respectively. 
There are two cases when making an ACK at $t_{k+1}$: $1)$ the ACK marker $M$ equals or is larger than $1$ at $t_{k+1}$; $2)$ the ACK marker $M$ equals or is larger than $1$ at some OFF slot $t'$ $(t_k + 1<t'< t_{k+1})$ and $t_{k+1}$ is the very first ON slot after $t'$ (the channels are OFF during $[t', t_{k+1} - 1]$). 
Note that $t' \ne {t_k} + 1$ because the ACK cost $c > 1$.
An illustration of Case~$2$ is provided in Fig.~\ref{fig:online_proof_primal_dual_updates}.
We emphasize that according to Observation~\ref{obs:after_ACK}, we only need to consider the primal variable update and dual variable updates of packet $(t_k + 1)$ to packet $t_{k+1}$, since all previous packets (packet $1$ to packet $t_k$) are never updated after slot $t_k$.

Case $1)$: The ACK marker $M$ equals or is larger than $1$ at $t_{k+1}$. 
In this case, Lines~\ref{line:d_less_than_1_begin}-\ref{line:d_less_than_1_end} in PDOA are repeated $\left\lceil c \right\rceil $ times, and the total holding cost in $I_k$ is $\sum\nolimits_{t = {t_k} + 1}^{{t_{k + 1}}} {\sum\nolimits_{i = {t_k} + 1}^t {{z_i}(t)} } =  \left\lceil c \right\rceil$ and the total ACK cost in $I_k$ is $\sum\nolimits_{t = {t_k} + 1}^{{t_{k + 1}}} {c \cdot d(t)}  = c \cdot d({t_{k + 1}}) = c$. 
Thus, the primal objective value is  $P(k) = \left\lceil c \right\rceil  + c$. 
Similarly, the dual objective value is the sum of the dual variables $y_i(t)$ in $I_k$, which is $D(k) = \sum\nolimits_{t = {t_k} + 1}^{{t_{k + 1}}} {\sum\nolimits_{i = {t_k} + 1}^t {{y_i}(t)} }  = c$. Therefore, we have $P(k)/D(k) = (\left\lceil c \right\rceil  + c)/c \le (c + 1 + c)/c < 3$.

    \begin{figure}
        \centering
        \includegraphics[scale = 0.5]{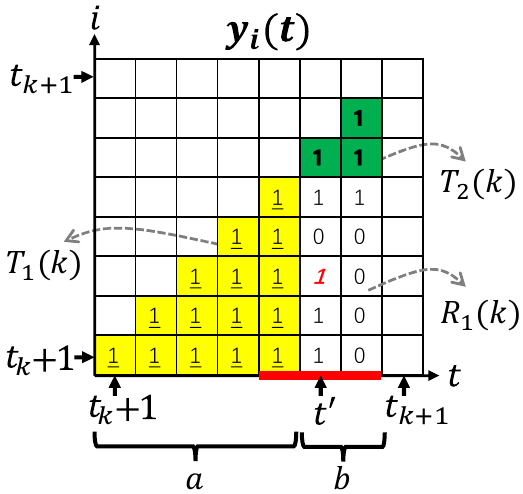}
        \caption{
        The updates of primal variables $z_i(t)$ and dual variables $y_i(t)$ in the $k$-th ACK interval $[t_k+1, t_{k+1}]$, where channels are OFF during $[t'-1, t'+1]$. The red bold italic $1$ denotes when the ACK marker equals or is larger than $1$.
        In addition, ${T_1}(k)$ is an equilateral triangle made of 1 (the underlined $1$'s with yellow background), ${T_2}(k)$ is an equilateral triangle made of $1$ (the bold $1$'s with green background), and
        ${R_1}(k)$ is a rectangle made of $1$ and $0$ (the regular $1$'s and $0$'s without background).} \label{fig:online_proof_primal_dual_updates}
        \vspace{-10pt}
    \end{figure}

Case $2)$: The ACK marker $M$ equals or is larger than $1$ at some OFF slot $t'$  $(t_k + 1<t'< t_{k+1})$ and $t_{k+1}$ is the very first ON slot after $t'$. 
We use $C_A(k)$ to denote the total ACK cost and use $C_H(k)$ to denote the total holding cost in $I_k$.  Here $P(k) = C_A(k) + C_H(k)$.
We have
    \begin{equation}
        \begin{aligned}
            {P(k)}/{D(k)} &  = ({{C_A(k) + C_H(k)}})/{D(k)} \\
            & \buildrel (a) \over = {c}/{D(k)} + {C_H(k)}/{D(k)}\\ 
            & \buildrel (b) \over \le {c}/{c} + {C_H(k)}/{D(k)}\\
            & \buildrel (c) \over \le {c}/{c} + {{2\Delta D}(k)}/{D(k)}\\
            & =3,
        \end{aligned}
        \label{eq:online_CR_3}
    \end{equation}
where $(a)$ is because PDOA makes only one ACK at slot ${{t_{k + 1}}}$ during $I_k$, 
i.e., ${C_A}(k) = \sum\nolimits_{t = {t_k} + 1}^{{t_{k + 1}}} {c \cdot d(t)}  = c \cdot d({t_{k + 1}}) = c$; $(b)$ is due to the dual objective value $D(k)$ is at least $c$ (i.e., when the ACK markter $M$ equals or is larger than $1$, $D(k)$ equals $c$, and $D(k)$ can be larger than $c$ due to the additional updates of the dual variables in Lines~\ref{line:final_update_begin}-\ref{line:final_update_end}); and we prove $(c)$ as follows. 
The total holding cost in $I_k$ is 
\begin{equation}
\begin{aligned}
& {C_H}(k) \\
= & \sum\nolimits_{\tau  = {t_k} + 1}^{{t_{k + 1}}} {\sum\nolimits_{i = {t_k} + 1}^\tau  {{z_i}(\tau )} } \\ 
= & \sum\nolimits_{\tau  = {t_k} + 1}^{{t_{k + 1}} - 1} {\sum\nolimits_{i = {t_k} + 1}^\tau  {{z_i}(\tau )} }  + \sum\nolimits_{i = {t_k} + 1}^{{t_{k + 1}}} {{z_i}({t_{k + 1}})}  \\
\buildrel (d) \over = & \sum\nolimits_{\tau  = {t_k} + 1}^{{t_{k + 1}} - 1} {\sum\nolimits_{i = {t_k} + 1}^\tau  {{z_i}(\tau )} }  + 0  \\
= & \sum\nolimits_{\tau  = {t_k} + 1}^{{t_{k + 1}} - 1} {\sum\nolimits_{i = {t_k} + 1}^\tau  {{z_i}(\tau )} } \\
\buildrel (e) \over = & \sum\nolimits_{\tau  = {t_k} + 1}^{{t_{k + 1}} - 1} {\sum\nolimits_{i = {t_k} + 1}^\tau  1 } \\
= & ({t_{k + 1}} - {t_k} - 1)({t_{k + 1}} - {t_k})/2,
\end{aligned}
\end{equation}
where in $(d)$,  ${z_i}({t_{k + 1}}) = 0$ for any $i \in [{t_k} + 1, {t_{k + 1}}]$ is because all the packets in $I_k$ are acked at slot $t_{k+1}$; and in $(e)$, ${{z_i}(\tau )} = 1$ for any $\tau  \in [{t_k} + 1,{t_{k + 1}} - 1]$ and $i \in [{t_k} + 1,\tau ]$ is because the packets in $I_k$ are not acked until slot $t_{k+1}$, and each of them needs to pay a holding cost, i.e., ${{z_i}(\tau )} = 1$. 
Similarly, the dual objective value in $I_k$ can be computed as $D(k) = \sum\nolimits_{\tau  = {t_k} + 1}^{{t_{k + 1}} - 1} {\sum\nolimits_{i = {t_k} + 1}^\tau  {{y_i}(\tau )} } $. 
We split $D(k)$ into three parts: the triangle ${T_1}(k) = \sum\nolimits_{\tau  = {t_k} + 1}^{t' - 1} {\sum\nolimits_{i = {t_k} + 1}^\tau  {{y_i}(\tau )} } $, the triangle  ${T_2}(k) = \sum\nolimits_{\tau  = t'}^{{t_{k + 1}} - 1} {\sum\nolimits_{i = t'}^\tau  {{y_i}(\tau )} } $, and the rectangle ${R_1}(k) = \sum\nolimits_{\tau  = t'}^{{t_{k + 1}} - 1} {\sum\nolimits_{i = {t_k} + 1}^{t' - 1} {{y_i}(\tau )} } $ (see an illustration in Fig.~\ref{fig:online_proof_primal_dual_updates}). Here $D(k) = {T_1}(k) + {T_2}(k) + {R_1}(k)$.

Our goal is to show that ${C_H}(k) \le 2({T_1}(k) + {T_2}(k) + {R_1}(k) ) = 2D(k)$.
Here, we can compute ${T_1}(k) = \sum\nolimits_{\tau  = {t_k} + 1}^{t' - 1} {\sum\nolimits_{i = {t_k} + 1}^\tau  {{y_i}(\tau )} }  = \sum\nolimits_{\tau  = {t_k} + 1}^{t' - 1} {\sum\nolimits_{i = {t_k} + 1}^\tau  1 }  = (t' - {t_k} - 1)(t' - {t_k})/2$, where ${{y_i}(\tau )} = 1$ for any $\tau  \in [{t_k} + 1,t' - 1]$ and $i \in [{t_k} + 1,\tau ]$ comes from  Lines~\ref{line:d_less_than_1_begin}-\ref{line:d_less_than_1_end} in PDOA since the ACK marker $M$ equals or is larger than $1$ until $t'$. 
In addition, we can compute ${T_2}(k) = \sum\nolimits_{\tau  = t'}^{{t_{k + 1}} - 1} {\sum\nolimits_{i = t'}^\tau  {{y_i}(\tau )} }  = \sum\nolimits_{\tau  = t' + 1}^{{t_{k + 1}} - 1} {\sum\nolimits_{i = t'}^\tau  {{y_i}(\tau )} }  + {y_{t'}}(t') = \sum\nolimits_{\tau  = t' + 1}^{{t_{k + 1}} - 1} {\sum\nolimits_{i = t'}^\tau  1 }  + {y_{t'}}(t') = [({t_{k + 1}} - t')({t_{k + 1}} - t' + 1)/2 - 1] + {y_{t'}}(t')$, where ${{y_i}(\tau )} = 1$ for any $\tau  \in [t' + 1, t_{k+1}-1]$ and $i \in [t',\tau ]$ comes from Lines~\ref{line:final_update_begin}-\ref{line:final_update_end} in PDOA since the channels are OFF during $[t' + 1, t_{k+1}-1]$. 
Next, we discuss the value of ${T_2}(k) + {R_1}(k)$ based on the value of ${y_{t'}}(t')$. On the one hand, if ${y_{t'}}(t') < 1$ (i.e., right after the update of packet $t'$ at slot $t'$, the ACK marker $M$ becomes no smaller than $1$), then we have ${R_1}(k) = \sum\nolimits_{\tau  = t'}^{{t_{k + 1}} - 1} {\sum\nolimits_{i = {t_k} + 1}^{t' - 1} {{y_i}(\tau )} }  \ge {y_{{t_k} + 1}}(t') = 1$, where ${y_{{t_k} + 1}}(t') = 1$ because at slot $t'$, ${y_i}(t') = 1$ for any $i \in [{{t_k} + 1}, t')$ except that ${y_{t'}}(t') < 1$. In this case, we have ${T_2}(k) + {R_1}(k) \ge ({t_{k + 1}} - t')({t_{k + 1}} - t' + 1)/2$. 
On the other hand, if ${y_{t'}}(t') = 1$ (i.e., the ACK marker $M$ becomes no smaller than $1$ before the update of packet $t'$ at slot $t'$, and ${y_{t'}}(t') = 1$ due to Lines~\ref{line:final_update_begin}-\ref{line:final_update_end}), we have ${T_2}(k) + {R_1}(k) \ge {T_2}(k) = ({t_{k + 1}} - t')({t_{k + 1}} - t' + 1)/2$. In both cases, we always have ${T_2}(k) + {R_1}(k) \ge ({t_{k + 1}} - t')({t_{k + 1}} - t' + 1)/2$.
Let $a$ be the length of $[{t_k} + 1,t' - 1]$ (i.e., $a = t' - {t_k} - 1$) and $b$ be the length of $[t', t_{k+1} -1]$ (i.e., $b = {t_{k + 1}} - t'$). Now we have ${T_1}(k) = a(a + 1)/2$, ${T_2}(k) + {R_1}(k) \ge b(b + 1)/2$, and ${C_H}(k) = (a + b)(a + b + 1)/2$.  Clearly, we have
      \begin{equation}
       \begin{aligned}
            & 2D(k) - {C_H}(k) \\
            & = 2({T_1}(k) + {T_2}(k) + {R_1}(k)) - {C_H}(k) \\
            & \ge  2 \cdot [a(a + 1)/2 + b(b + 1)/2] - (a + b)(a + b + 1)/2 \\
            & = [{(a - b)^2} + a + b]/2 \\
            & \ge 0,
        \end{aligned}
        \label{eq:triangles_larger_than_rectangle}
   \end{equation}
which completes $(c)$ in Eq.~\eqref{eq:online_CR_3}.

In the end, we consider the last time interval $[{t_K} + 1,T]$ (denoted by $I_{K}$).
If the last slot is an ON slot and PDOA makes the last ACK exactly at the last slot, i.e., $t_K= T$, then our previous analysis in Cases $1$ and $2$ still holds. Next, we consider the scenario where $t_K < T$. There are two cases at slot $T$: $1)$ $T$ is the slot before the ACK marker $M$ equals or is larger than $1$; $2)$ $T$ is the slot when or after the ACK marker $M$ equals or is larger than $1$. In both cases, there is no ACK made during $I_{K}$.

Case $1)$: $T$ is the slot before the ACK marker $M$ equals or is larger than $1$.
In this case, the total holding cost and the total ACK cost in $I_K$ are $(T - {t_K})(T - {t_K} + 1)/2$ and $0$, respectively. Thus, the primal objective value is $ P(K)=  (T - {t_K})(T - {t_K} + 1)/2 + 0 \cdot c = (T - {t_K})(T - {t_K} + 1)/2$. 
The dual objective value is the sum of total dual variables $y_i(t)$ in $I_{K+1}$, which is $(T - {t_K})(T - {t_K} + 1)/2$, i.e., $D(K)=(T - {t_K})(T - {t_K} + 1)/2$.  Therefore, we have $P(K)/D(K) = 1$. 

Case $2)$: $T$ is the slot when or after the ACK marker $M$ equals or is larger than $1$.
We assume that the ACK marker $M$ equals or is larger than $1$ at some slot $t'$ $( t_K < t' \le T)$. 
We claim that the channels are OFF during the interval $[t', T]$. Otherwise, an ACK will be made during $[t', T]$, which contradicts the fact that there is no ACK made during $I_{K}$.
We use $C_A(K)$ to denote the total ACK cost and use $C_H(K)$ to denote the total holding cost.  Here $P(K) = C_A(K) + C_H(K) = 0 + C_H(K)$, where $C_A(K) = 0$ because there is no ACK made during $I_{K}$.
We have
  \begin{equation}
        \begin{aligned}
            {P(K)}/{D(K)} &  = ({{C_A(K) + C_H(K)}})/{D(K)} \\
            & = 0 + {C_H(K)}/{D(K)}\\ 
            & \buildrel (a) \over \le {{2\Delta D}(K)}/{D(K)}\\
            & =2,
        \end{aligned}
    \end{equation}
where the analysis in $(a)$ mirrors  that of $(c)$ in Eq.~\eqref{eq:online_CR_3}.

In summary, given any channel state $\mathbf{s}$, PDOA achieves $P(k)/D(k) \le 3$ for any ACK interval $I_k$ ($k \in [0,K]$), and thus PDOA achieves
$P/D \le 3$ on the entire instance $\mathbf{s}$.
This, along with Lemma~\ref{lemma:online_feasible_sol} and leveraging the weak duality \cite{buchbinder2009design}, confirms that PDOA is $3$-competitive.
\end{proof}

\vspace{-1cm}
\begin{IEEEbiography}[{\includegraphics[width=1in,height=1.25in,clip,keepaspectratio]{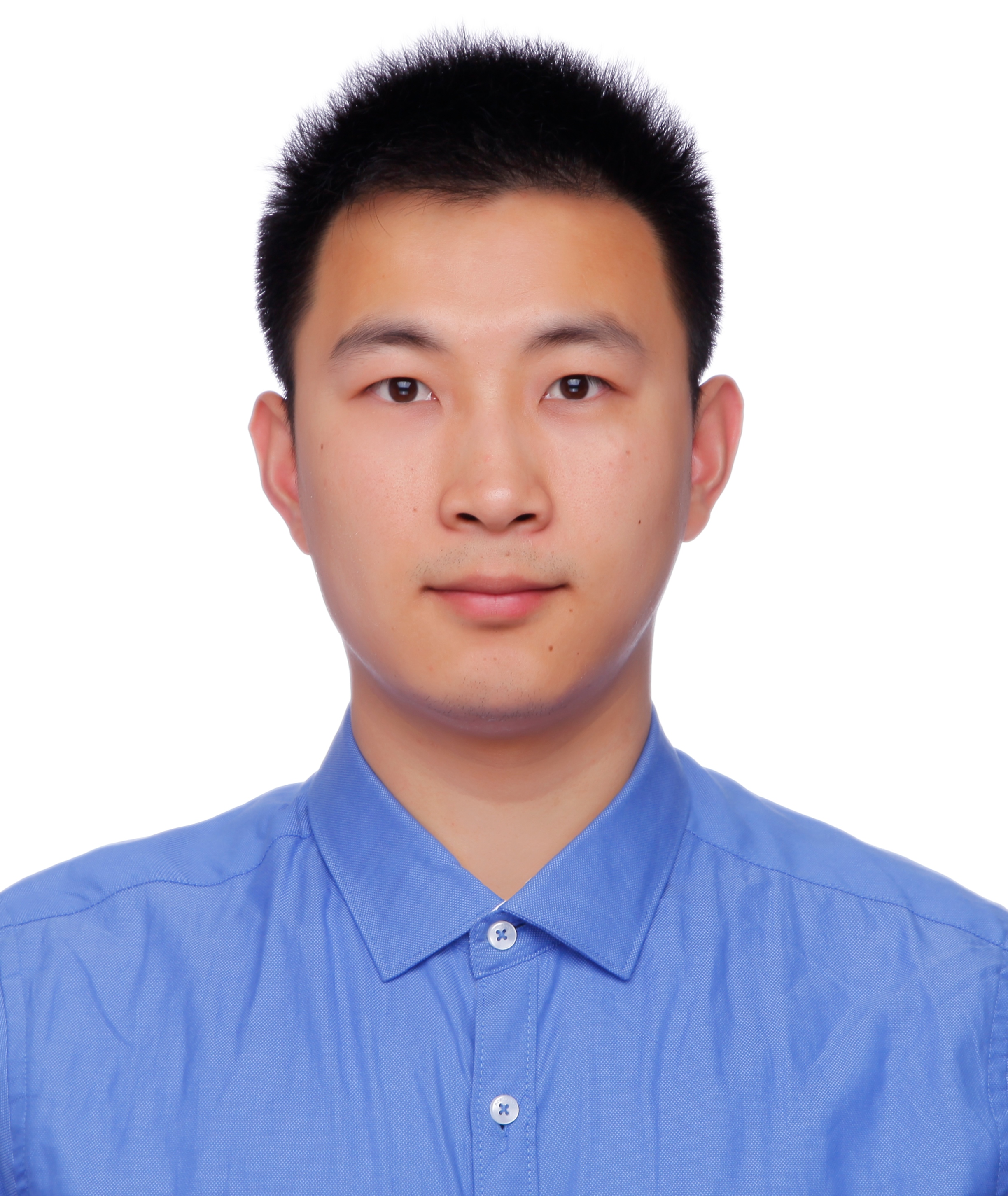}}]{Zhongdong Liu} received his B.S. degree in Mathematics and Applied Mathematics with honor from Northeast Forestry University, Harbin, China, in 2016,  and his Ph.D. degree in Computer Science and Application from Virginia Tech, Blacksburg, VA, USA, in 2024. He is currently an instructor in the Department of Computer Science at Virginia Tech, Blacksburg, VA, USA. His research interests are in the modeling, analysis, control, and optimization of complex network systems.
\end{IEEEbiography}

\begin{IEEEbiography}[{\includegraphics[width=1in,height=1.25in,clip,keepaspectratio]{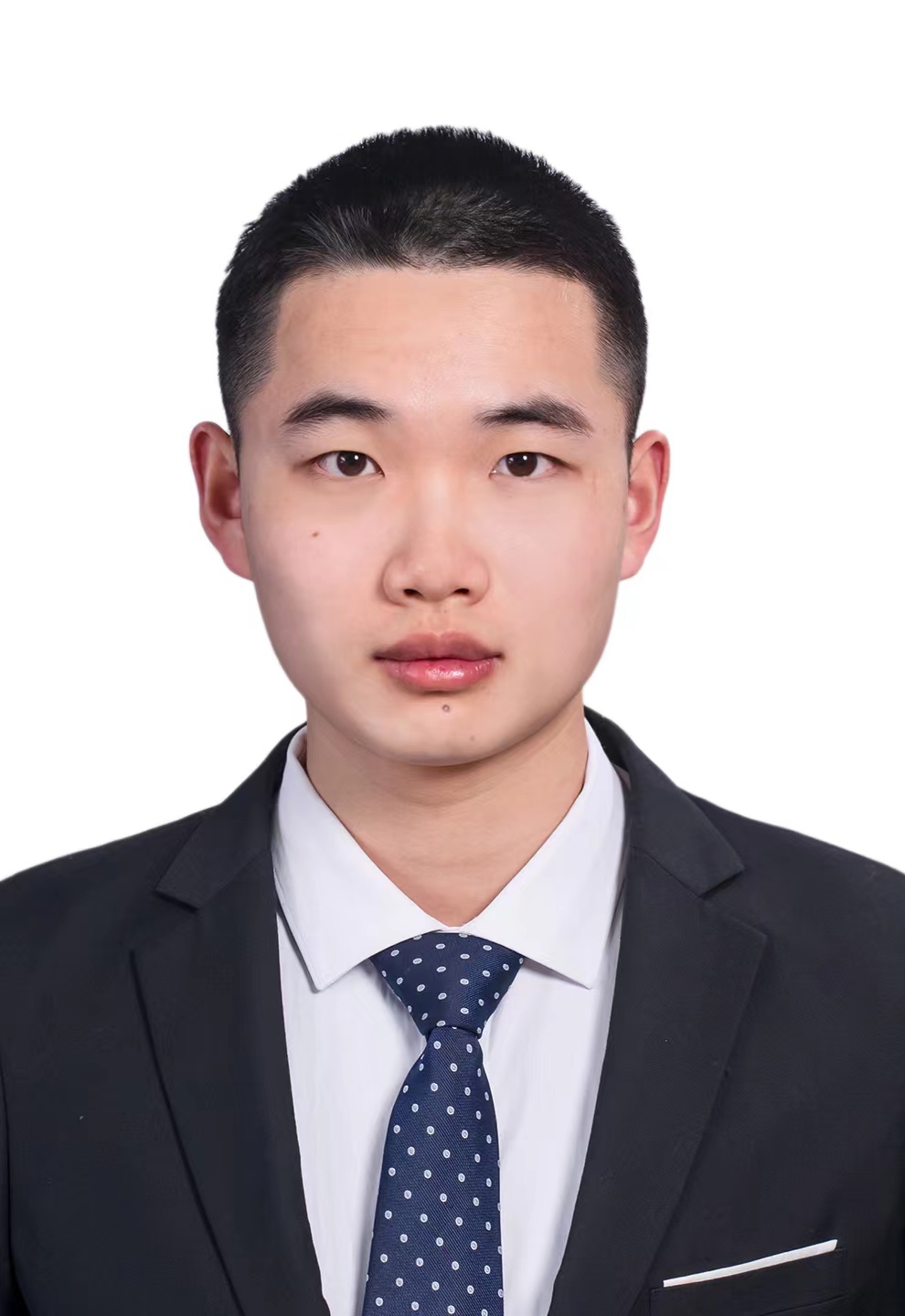}}]{Keyuan Zhang}(S'22) receive his B.S. degree in Computer Science and Engineering from Southern University of Science and Technology (SUSTech), Shenzhen, China, in 2022. He is currently pursuing his Ph.D. degree with the Department of Computer Science, Virginia Tech, Blacksburg, VA, USA. His research interests are modeling, analysis, and algorithm design for machine learning and network systems. 
\end{IEEEbiography}

\vspace{-0.6cm}
\begin{IEEEbiography}[{\includegraphics[width=1in,height=1.25in,clip,keepaspectratio]{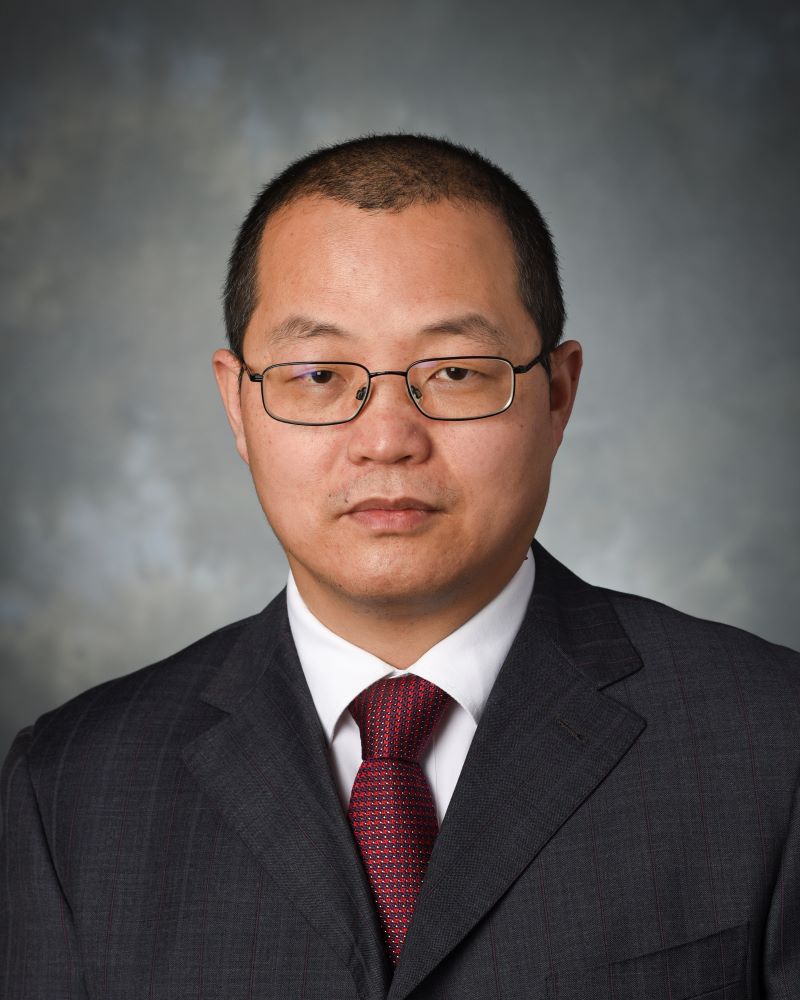}}]{Bin Li}(S'11-M'16-SM'20)
received the B.S. degree in Electronic and Information Engineering, M.S. degree in Communication and Information Engineering, both from Xiamen University, China, and 
Ph.D. degree in Electrical and  Computer Engineering from The Ohio State University. He is currently an associate professor in the Department of Electrical Engineering at the Pennsylvania State University, University Park, PA, USA. His research focuses on the intersection of networking, machine learning, and system developments, and their applications in networking for virtual/augmented reality, mobile edge computing, mobile crowd-learning, and Internet-of-Things. He is a senior member of the IEEE and a member of the ACM. He received both the National Science Foundation (NSF) CAREER Award and the Google Faculty Research Award.
\end{IEEEbiography}

\begin{IEEEbiography}[{\includegraphics[width=1in,height=1.25in,clip,keepaspectratio]{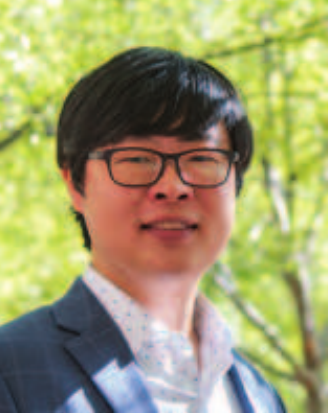}}]{Yin Sun}(Senior Member, IEEE) is the Bryghte D. and Patricia M. Godbold Endowed Associate Professor in the Department of Electrical and Computer Engineering at Auburn University, Alabama. He received his B.Eng. and Ph.D. degrees in Electronic Engineering from Tsinghua University, in 2006 and 2011, respectively. He was an Assistant Professor in the Department of Electrical and Computer Engineering at Auburn University from 2017 to 2023 and a Postdoctoral Scholar and Research Associate at the Ohio State University from 2011 to 2017. His research interests include Networking, Machine Learning, Semantic Communications, Age of Information, and Information Theory. His articles received the Best Student Paper Award of the IEEE/IFIP WiOpt 2013, the Best Paper Award of the IEEE/IFIP WiOpt 2019, runner-up for the Best Paper Award of ACM MobiHoc 2020, and the Journal of Communications and Networks (JCN) Best Paper Award in 2021. He received the Auburn Author Award in 2020, the National Science Foundation (NSF) CAREER Award in 2023, the Bryghte D. and Patricia M. Godbold Endowed Professorship in 2023, the Ginn Faculty Achievement Fellowship in 2023, and the College of Engineering's Research Award for Excellence (Senior Faculty) in 2024. 
\end{IEEEbiography}

\vspace{-1cm}
\begin{IEEEbiography}[{\includegraphics[width=1in,height=1.25in,clip,keepaspectratio]
{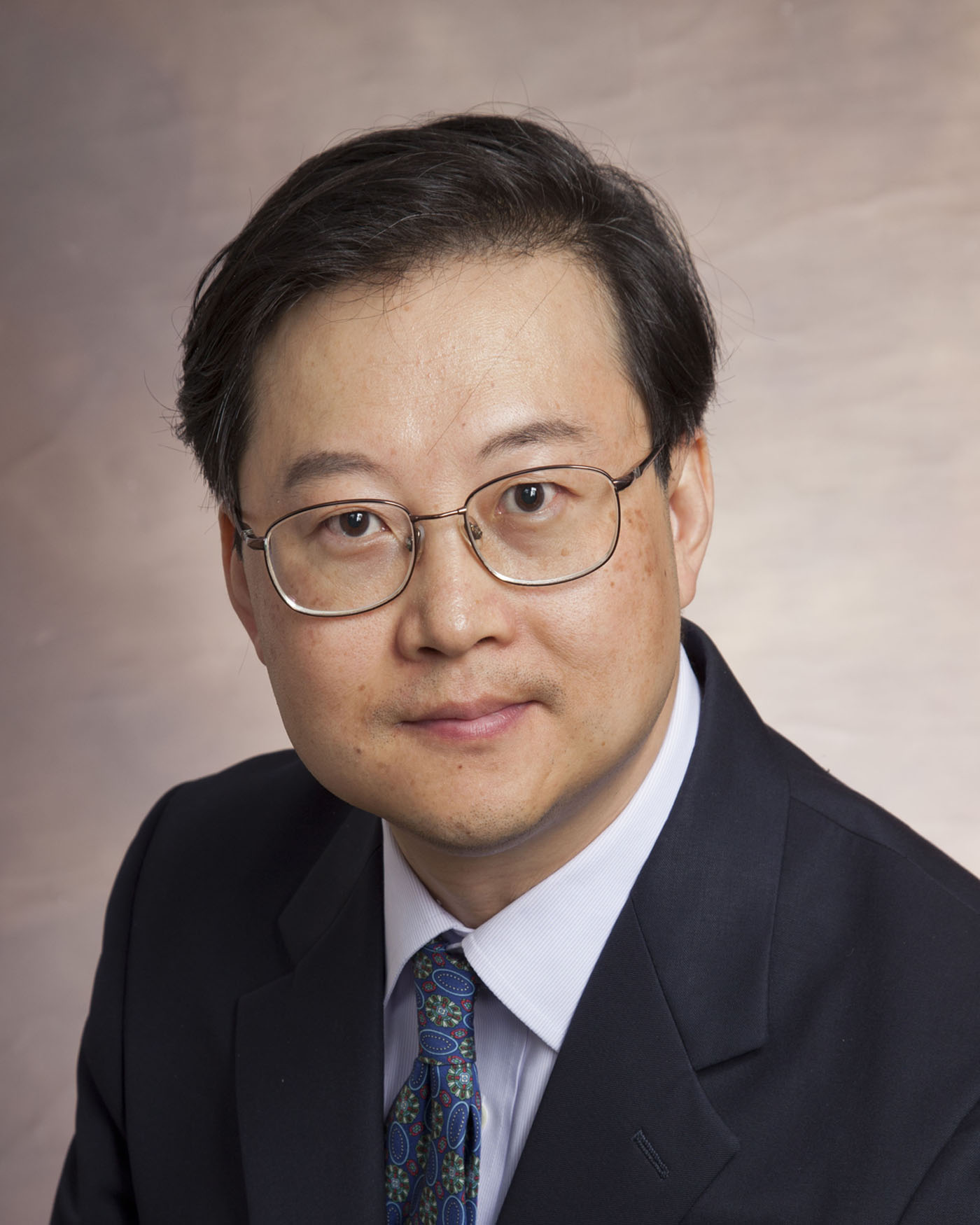}}]{Y. Thomas Hou}(Fellow, IEEE) received his Ph.D. from NYU Tandon School of Engineering in 1998. He is currently Bradley Distinguished Professor of Electrical and Computer Engineering at Virginia Tech, Blacksburg, VA, USA, which he joined in 2002. He was a Member of Research Staff at Fujitsu Laboratories of America in Sunnyvale, CA from 1997 to 2002. 
His current research focuses on developing real-time optimal solutions to complex science and engineering problems arising from wireless and mobile networks. 
He is also interested in wireless security.  
He authored/co-authored two textbooks and has published over 400 papers in IEEE/ACM journals and conferences. 
His papers were recognized by 12 best paper awards from IEEE and ACM, including an IEEE INFOCOM Test of Time Paper Award in 2023. 
He holds six U.S. patents.  
Prof. Hou was named an IEEE Fellow for contributions to modeling and optimization of wireless networks.  
He was/is on the editorial boards of a number of IEEE and ACM transactions and journals. 
He was Steering Committee Chair of IEEE INFOCOM conference and was a member of the IEEE Communications Society Board of Governors.  He was also a Distinguished Lecturer of the IEEE Communications Society.  
\end{IEEEbiography}

\vspace{-1cm}
\begin{IEEEbiography}[{\includegraphics[width=1in,height=1.25in,clip,keepaspectratio]{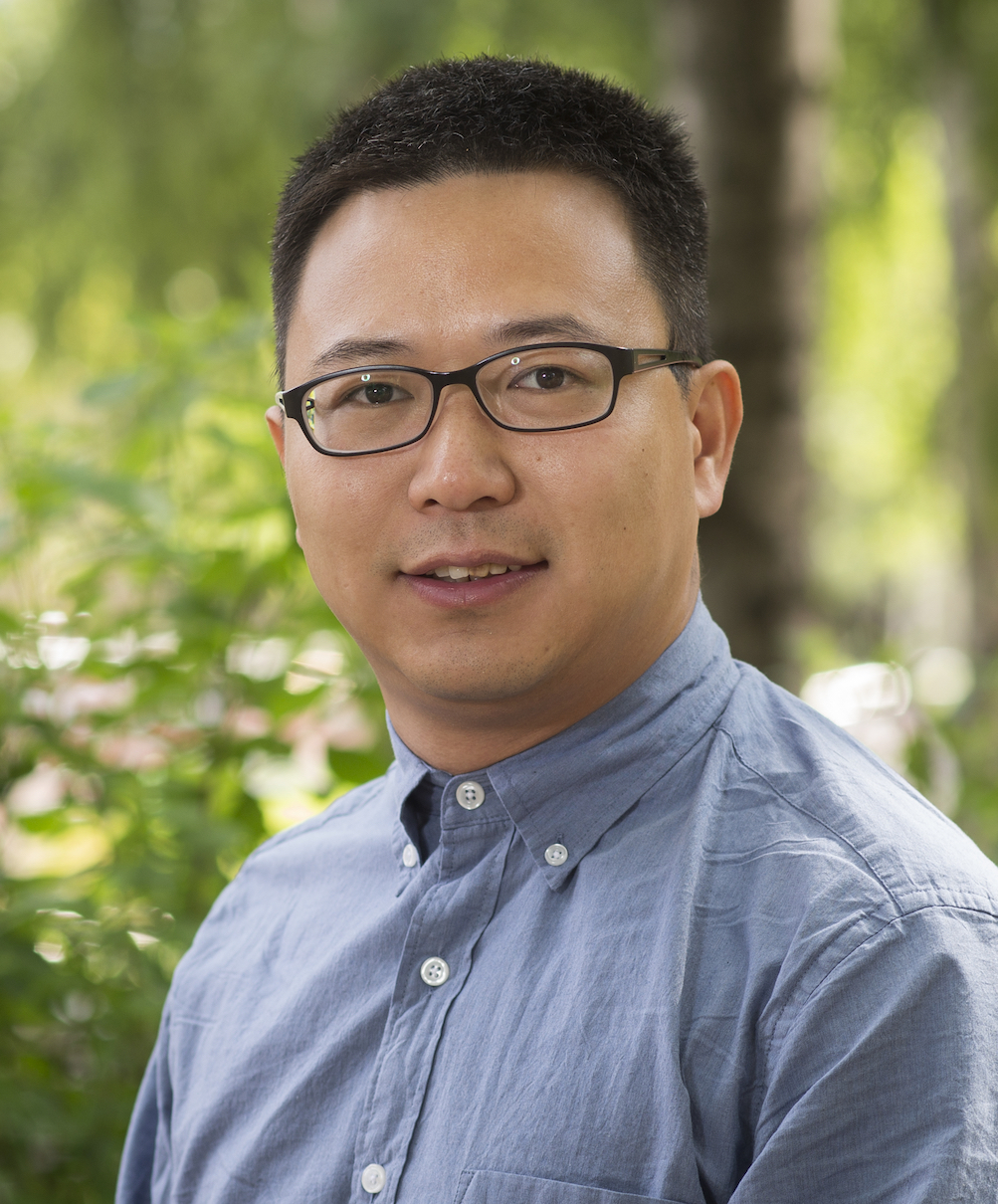}}]{Bo Ji}(S'11-M'12-SM'18)
received his B.E. and M.E. degrees in Information Science and Electronic Engineering from Zhejiang University, Hangzhou, China, in 2004 and 2006, respectively, and his Ph.D. degree in Electrical and Computer Engineering from The Ohio State University, Columbus, OH, USA, in 2012. Dr. Ji is an Associate Professor of Computer Science and a College of Engineering Faculty Fellow at Virginia Tech. Prior to joining Virginia Tech, he was an Associate Professor in the Department of Computer and Information Sciences at Temple University, where he was an Assistant Professor from July 2014 to June 2020. He was also a Senior Member of Technical Staff at AT\&T Labs, San Ramon, CA, from January 2013 to June 2014.
His research interests include the multidisciplinary intersections of Computing and Networking Systems, Artificial Intelligence and Machine Learning, Security and Privacy, and Extended Reality.
He has been the general co-chair of IEEE/IFIP WiOpt 2021 and the technical program co-chair of ACM MobiHoc 2023 and ITC 2021, and he has also served on the editorial boards of various IEEE and ACM journals (IEEE/ACM Transactions on Networking, ACM SIGMETRICS Performance Evaluation Review, IEEE Transactions on Network Science and Engineering, IEEE Internet of Things Journal, and IEEE Open Journal of the Communications Society). Dr. Ji is a senior member of the IEEE and the ACM and a member of the AAAS. He was a recipient of the National Science Foundation (NSF) CAREER Award in 2017, the NSF CISE Research Initiation Initiative Award in 2017, the IEEE INFOCOM 2019 Best Paper Award, the IEEE/IFIP WiOpt 2022 Best Student Paper Award, the IEEE TNSE Excellent Editor Award in 2021, 2022, and 2024, and the Dean's Faculty Fellow Award from the College of Engineering at Virginia Tech in 2023.
\end{IEEEbiography}
\clearpage
\appendices

\twocolumn[
\begin{center}
\section*{\normalfont\Huge\sffamily \textcolor{gray}{Supplemental Material of} \\ Learning-augmented Online Minimization of Age
of Information and Transmission Costs}

\vspace{0.6cm}

{\normalfont\large\sffamily Zhongdong Liu, Keyuan Zhang, Bin Li, Yin Sun, Y. Thomas Hou, and Bo Ji}

\vspace{1cm}

\textbf{DESCRIPTION}: This document is supporting content for the manuscript (Learning-augmented Online Minimization of Age
of Information and Transmission Costs), given the
limited number of pages allowed by the publisher.
\end{center}

\vspace{1.2cm}
]

\section{Proof of Lemma~\ref{lemma:problem_equivalence}} 
\label{appendix:proof_of_equivalence}
\begin{proof}
Our goal is to show that: (i) any feasible solution to Problem~\eqref{problem} can be converted to a feasible solution to Problem~\eqref{prob:integerl-program}, and these two solutions have the same total costs; (ii) any feasible solution to Problem~\eqref{prob:integerl-program} can be converted to a feasible solution to Problem~\eqref{problem}, and the total cost of the converted solution to Problem~\eqref{problem} is no greater than the total cost of the solution to Problem~\eqref{prob:integerl-program}. This implies that any optimal solution to Problem~\eqref{problem} is also an optimal solution to Problem~\eqref{prob:integerl-program}, and vice versa. Therefore, these two problems are equivalent \cite[Sec. 4.1.3]{boyd2004convex}.

We first show that any feasible solution to Problem~\eqref{problem} can be converted to a feasible solution to Problem~\eqref{prob:integerl-program}.
Given a channel state pattern $\mathbf{s}$, 
we assume that the solution $\pi  = \{ {t_1},{t_2}, \dots ,{t_n}\}$ is a feasible solution to Problem~\eqref{problem},
where this solution makes the $i$-th transmission at the ON slot $t_i$ (i.e., $s(t_i)d(t_i) = 1$), and the total number of transmission is $n$. We can compute the total cost of the solution $\pi$ to Problem~\eqref{problem} as ${C_{\ref{problem}}}(s,\pi ) = cn + ({t_1} - 1){t_1}/2 + \sum\nolimits_{i = 2}^n {({t_i} - {t_{i - 1}} - 1)({t_i} - {t_{i - 1}})/2}  + (T - {t_n} - 1)(T - {t_n})/2$, where $cn$ is the total transmission cost and the rest is the total staleness cost.
Based on the solution $\pi$ to Problem~\eqref{problem}, we construct a solution $\pi'$ to Problem~\eqref{prob:integerl-program} in the following way: (i) solution $\pi'$ sends the ACKs at the same time when solution $\pi$ transmits, i.e., solution $\pi'$ sends ACKs at a sequence of slots {$\{ t{_1},t{_2}, \dots ,t{_n}\} $};
(ii) based on the ACK decisions in step (i), for any slot $t$ and any packet $i \le t$,  solution $\pi'$ lets $z_i(t) = 0$ if  $\sum\nolimits_{\tau  = i}^t s (\tau )d(\tau ) \ge 1$ and $z_i(t) = 1$ if $\sum\nolimits_{\tau  = i}^t s (\tau )d(\tau ) = 0$. We denote the solution $\pi'$ to  Problem~\eqref{prob:integerl-program} by $\pi ' = \{ \{ {t_1},{t_2}, \dots ,{t_n}\} ,\{ \{ {z_i}(t)\} _{i = 1}^t\} _{t = 1}^T\} $. 
We can easily verify that solution $\pi'$ is a feasible solution to Problem~\eqref{prob:integerl-program} because both constraints~\eqref{eq:primal-con1} and \eqref{eq:primal-con2} are satisfied. Furthermore, according to the construction of $z_i(t)$ in step (ii), we know that once the ACK is made at some ON slot $t \in \{ {t_1},{t_2}, \dots,{t_n}\}$, then all previously arrived packets (packet $1$ to packet $t$) are acked forever after slot $t$, i.e.,  $z_i(\tau) = 0$ for all $i \le t$ and all $\tau \ge t$. This indicates that we can compute the total holding cost of the solution $\pi'$ as 
\begin{equation}
    \begin{aligned}
        & \sum\nolimits_{t = 1}^T {\sum\nolimits_{i = 1}^t {{z_i}(t)} } \\
        = & \underbrace {\sum\nolimits_{t = 1}^{{t_1} - 1} {\sum\nolimits_{i = 1}^t {{z_i}(t)} }  + \sum\nolimits_{t = {t_1}}^T {\sum\nolimits_{i = 1}^{{t_1}} {{z_i}(t)} } }_{\text{Total holding cost of packet $1$ to packet $t_1$}} \\
        & + \underbrace {\sum\nolimits_{t = {t_1} + 1}^{{t_2} - 1} {\sum\nolimits_{i = {t_1} + 1}^t {{z_i}(t)} }  + \sum\nolimits_{t = {t_2}}^T {\sum\nolimits_{i = {t_1} + 1}^{{t_2}} {{z_i}(t)} } }_{\text{Total holding cost of packet ($t_1+1$) to packet $t_2$}} \\
        & +  \dots  + \underbrace {\sum\nolimits_{t = {t_n} + 1}^T {\sum\nolimits_{i = {t_n} + 1}^t {{z_i}(t)} } }_{\text{Total holding cost of packet ($t_n+1$) to packet $T$}} \\
        \buildrel (a) \over  = & \sum\nolimits_{t = 1}^{{t_1} - 1} {\sum\nolimits_{i = 1}^t 1 }  + \sum\nolimits_{t = {t_1}}^T {\sum\nolimits_{i = 1}^{{t_1}} 0 } \\
        & + \sum\nolimits_{t = {t_1} + 1}^{{t_2} - 1} {\sum\nolimits_{i = {t_1} + 1}^t {1} }  + \sum\nolimits_{t = {t_2}}^T {\sum\nolimits_{i = {t_1} + 1}^{{t_2}} {0} } \\
        & +  \dots  + \sum\nolimits_{t = {t_n} + 1}^T {\sum\nolimits_{i = {t_n} + 1}^t {1} }  \\
        = & ({t_1} - 1){t_1}/2 + \sum\nolimits_{i = 2}^n {({t_i} - {t_{i - 1}} - 1)({t_i} - {t_{i - 1}})/2} \\ 
        & + (T - {t_n} - 1)(T - {t_n})/2,
    \end{aligned}
    \label{eq:total_holding_cost}
\end{equation}
where in (a), ${z_i}(t) = 1$ because packet $i$ is not acked by slot $t$ and ${z_i}(t) = 0$ otherwise.
In addition, the total ACK cost of the solution $\pi'$ is $cn$. Therefore, the total cost of the solution $\pi'$ to Problem~\eqref{prob:integerl-program} is ${C_{\ref{prob:integerl-program}}}(s,\pi') = cn + ({t_1} - 1){t_1}/2 + \sum\nolimits_{i = 2}^n {({t_i} - {t_{i - 1}} - 1)({t_i} - {t_{i - 1}})/2}  + (T - {t_n} - 1)(T - {t_n})/2$, which is the same as the total cost of solution $\pi$ to Problem~\eqref{problem}. 
In summary, any feasible solution $\pi$ to Problem~\eqref{problem} can be converted to a feasible solution $\pi'$ to Problem~\eqref{prob:integerl-program}, and those two solutions have the same total cost, i.e., ${C_{\ref{problem}}}(s,\pi ) = {C_{\ref{prob:integerl-program}}}(s,\pi ')$.

Next, we show that any feasible solution to Problem~\eqref{prob:integerl-program} can be converted to a feasible solution to Problem~\eqref{problem}.
Given a channel state pattern $\mathbf{s}$, 
we assume that the solution $\pi = \{ \{ {t_1},{t_2}, \dots ,{t_n}\} ,\{ \{ {z_i}(t)\} _{i = 1}^t\} _{t = 1}^T\}$ is a feasible solution to Problem~\eqref{prob:integerl-program}, where this solution makes the $i$-th ACK at the ON slot $t_i$ (i.e., $s(t_i)d(t_i) = 1$). 
Though the solution $\pi$ is a feasible solution to Problem~\eqref{prob:integerl-program}, it is possible that this solution makes unnecessary cost of $z_i(t)$ (i.e., letting $z_i(t) = 1$ even though $\sum\nolimits_{\tau=i}^{t}s(\tau) d(\tau) \ge 1$).
In this case, we can always find another feasible solution $\hat \pi  = \{ \{ {t_1},{t_2}, \dots ,{t_n}\} ,\{ \{ {{\hat z}_i}(t)\} _{i = 1}^t\} _{t = 1}^T\} $ to Problem~\eqref{prob:integerl-program} that makes the same ACK decisions as the solution $\pi$ but never makes the unnecessary cost of ${\hat z}_i(t)$ (i.e., letting ${\hat z}_i(t) = 1$ only when $\sum\nolimits_{\tau=i}^{t}s(\tau) d(\tau) = 0$ and letting ${\hat z}_i(t) = 0$ only when $\sum\nolimits_{\tau=i}^{t}s(\tau) d(\tau) \ge 1$), and their total cost in Problem~\eqref{prob:integerl-program}  satisfies $C_{\ref{prob:integerl-program}}(\mathbf{s}, {\hat \pi}) \le C_{\ref{prob:integerl-program}}(\mathbf{s}, \pi)$. 
Similar to the previous analysis (i.e., Eq.~\eqref{eq:total_holding_cost}), the total cost of the solution $\hat \pi$ is ${C_{\ref{prob:integerl-program}}}(s,{\hat \pi}) = cn + ({t_1} - 1){t_1}/2 + \sum\nolimits_{i = 2}^n {({t_i} - {t_{i - 1}} - 1)({t_i} - {t_{i - 1}})/2}  + (T - {t_n} - 1)(T - {t_n})/2$.
Based on the feasible solution ${\hat \pi }$ to Problem~\eqref{prob:integerl-program}, we can construct a solution $\pi'$ to Problem~\eqref{problem} in the following way: 
(i) solution $\pi'$ transmits at the same time when solution $\pi$ sends the ACKs, i.e., solution $\pi'$ transmits at a sequence of slot {$\{ t{_1},t{_2}, \dots ,t{_n}\} $}.
We denote the solution $\pi'$ to  Problem~\eqref{problem} by $\pi ' = \{ {t_1},{t_2}, \dots ,{t_n}\}$. We can easily check that the solution $\pi'$ is a feasible solution to Problem~\eqref{problem} since constraint~\eqref{eq:problem_cons1} is satisfied. In addition, we can compute the total cost of the solution $\pi'$ to Problem~\eqref{problem} as ${C_{\ref{problem}}}(s,\pi') = cn + ({t_1} - 1){t_1}/2 + \sum\nolimits_{i = 2}^n {({t_i} - {t_{i - 1}} - 1)({t_i} - {t_{i - 1}})/2}  + (T - {t_n} - 1)(T - {t_n})/2$, which is the same as the total cost of solution ${\hat \pi }$ to Problem~\eqref{prob:integerl-program}.
In conclusion, any feasible solution $\pi$ to Problem~\eqref{prob:integerl-program} can be converted to a feasible solution $\pi'$ to Problem~\eqref{problem}, and their total cost satisfies ${C_{\ref{prob:integerl-program}}}(s,\pi ) \ge {C_{\ref{problem}}}(s,\pi')$.

Finally, we show that for any optimal solution of Problem~\eqref{problem}, it can be converted to an optimal solution to Problem~\eqref{prob:integerl-program}, and vice versa. 
Assuming that the solution $\pi_*$ is an optimal solution to Problem~\eqref{problem}.
From the above analysis, we can construct a feasible solution $\pi'_*$ to Problem~\eqref{prob:integerl-program}, and their total cost satisfies ${C_{\ref{problem}}}(\mathbf{s},\pi_*) = {C_{\ref{prob:integerl-program}}}(\mathbf{s},\pi'_*)$. 
We claim that $\pi'_*$ is also an optimal solution to Problem~\eqref{prob:integerl-program}. Otherwise, there must be an optimal solution $\pi''_*$ to Problem~\eqref{prob:integerl-program} such that $\pi''_* \ne \pi'_*$ and ${C_{\ref{prob:integerl-program}}}(\mathbf{s}, \pi'_*) > {C_{\ref{prob:integerl-program}}}(\mathbf{s}, \pi''_*)$. 
Again, from the previous analysis, we know that the optimal solution $\pi''_*$ to Problem~\eqref{prob:integerl-program} can be converted to a feasible solution $\pi _*^\dag $ to Problem~\eqref{problem}, and their total cost satisfies ${C_{\ref{prob:integerl-program}}}(\mathbf{s},\pi''_*) \ge {C_{\ref{problem}}}(\mathbf{s},\pi _*^\dag)$.
However, this indicates the solution $\pi_*$ is not an optimal solution to Problem~\eqref{problem} since we have ${C_{\ref{problem}}}(\mathbf{s},\pi _*^\dag) < {C_{\ref{problem}}}(\mathbf{s},\pi_*)$, which contradicts with our assumption. Therefore, the solution $\pi_*$ is also an optimal solution to Problem~\eqref{prob:integerl-program}. Similarly, we can show that any optimal solution to Problem~\eqref{prob:integerl-program} is also an optimal solution to Problem~\eqref{problem}. This completes the proof. 
\end{proof}

\section{Proof of Lemma~\ref{lemma:ml_robustness}}
\label{appendix:ml_robustness_proof}
Our proof outline is as follows.
We first show that LAPDOA produces a feasible primal solution and an almost feasible dual solution (with a factor of $c/(c+1)$) in Lemma~\ref{lemma:ml_almost_feasible_solution}.
Then, we show that the ratio between the total primal objective value (denoted by $P$) and the total dual objective value (denoted by $D$) is at most $3/ \lambda$, i.e., $P/D \le 3/\lambda$.  
Scaling down all dual variables ${y_i}(t)$ generated by LAPDOA by a multiplicative factor of $c/(c + 1)$, we obtain a feasible dual solution with a dual objective value of $(c/(c + 1)) \cdot D$. By the weak duality \cite{buchbinder2009design}, we have $P/OPT \le P/((c/(c + 1)) \cdot D) = ((c + 1)/c) \cdot P/D \le ((c + 1)/c) \cdot 3/\lambda $, completing the proof.

To begin with, we introduce two key observations of LAPDOA that will be widely used in the following proofs.

\begin{observation}
Assuming that LAPDOA made the latest ACK at some ON slot $L$ and the current time slot is $t$ ($t > L$), then at slot $t$, before the threshold is achieved ($M < 1$), LAPDOA updates the primal variable $z_i(t)$ and the dual variable $y_i(t)$ of packet $(L+1)$ to packet $t$; however, once the threshold is achieved ($M \ge 1$), LAPDOA only updates the primal variable $z_i(t)$ of the unacked packets if the channel is OFF at slot $t$.
\label{obs:ml_within_interval}
\end{observation}

\begin{observation}
Once LAPDOA makes an ACK at some ON slot $t$, all the packets arriving no later than slot $t$ (packet $1$ to packet $t$) are acked forever after slot $t$, and their primal variables and dual variables will never be changed after slot $t$,  i.e.,  $z_i(\tau) = 0$ and $y_i(\tau) = 0$ for all $i \le t$ and all $\tau > t$.
\label{obs:ml_after_ACK}
\end{observation}

Then, we show that LAPDOA gives an almost feasible solution in Lemma~\ref{lemma:ml_almost_feasible_solution}.

\begin{lemma}    LAPDOA produces a feasible solution to primal Problem \eqref{prob:integerl-program_relax}. In addition, let $\mathbf{y} \triangleq \{ \{ {y_i}(t)\} _{i = 1}^t\} _{t = 1}^T$ be the solution produced by LAPDOA to  dual Problem \eqref{problem:dual-program_relax}, then $(c/(c + 1))\mathbf{y}$ is a feasible solution to dual Problem \eqref{problem:dual-program_relax}.
\label{lemma:ml_almost_feasible_solution}
\end{lemma}

\begin{proof}
We omit the proof for primal constraints~\eqref{eq:relax_primal-con1}-\eqref{eq:relax_primal-con2} and dual constraint~\eqref{eq:relax_dual-con2} because they are similar to the proof in Lemma~\ref{lemma:online_feasible_sol}  and provide the proof for dual constraint~\eqref{eq:relax_dual-con1}. 
Consider an ON slot $t$ and its dual constraint~\eqref{eq:relax_dual-con1}
    $\sum_{i=1}^{t} \sum_{\tau=t}^{T} y_{i}(\tau) \leq c$.
Recall that this dual constraint requires that for all the packets arriving no later than slot $t$,  the sum of their dual variables beyond slot $t$ should not exceed $c$. 
We assume that this ON slot $t$ falls into the $k$-th ACK interval $[t_k + 1, t_{k+1}]$, i.e., ${t_k} + 1 \le t \le {t_{k + 1}}$, 
where LAPDOA makes two ACKs at the ON slots $t_k$ and $t_{k+1}$ (in a special case that the ON slot $t$ falls into the last interval $[t_K + 1, T]$, i.e., ${t_K} + 1 \le t \le T$, where LAPDOA makes the last ACK at the ON slot $t_K$, our following analysis can be extended to this case). 
According to Observations~\ref{obs:ml_within_interval}
and \ref{obs:ml_after_ACK}, packet $1$ to packet $t_k$ are not updated after slot $t_k$, and packet $(t_k+1)$ to packet $t$ are not updated after slot $t_{k+1}$, similar to Eq.~\eqref{eq:dual_var_reduction}, we have 
$\sum_{i={1}}^{t} \sum_{\tau=t}^{T} y_{i}(\tau) = \sum_{i={t_k} + 1}^{t} \sum_{\tau=t}^{{{t_{k + 1}}}} y_{i}(\tau)$.
Furthermore, we assume that during the interval $[t,{t_{k + 1}}]$, all packets arriving between $[t_k+1,t]$ (packet $t_k+1$ to packet $t$) make $m$ big updates and $n$ small updates (some zero updates can also be made but we ignore them since they cannot increase the dual variables). In other words, we have $\sum_{i={t_k} + 1}^{t} \sum_{\tau=t}^{{{t_{k + 1}}}} y_{i}(\tau) = 1 \cdot m + \lambda \cdot n  = m + \lambda n$.
We claim that if $m + \lambda n \ge c$,  then the increment of ACK marker $M$  due to those $m$ big and $n$ small updates (denoted by $M(m,n)$) will be larger than or equal to $1$.
This is true since we have $M(m,n) = m \cdot (1/\lambda c) + n \cdot (\lambda /c) = m/\lambda c + \lambda n/c \ge m/c + \lambda n/c = (m + \lambda n)/c \ge c/c = 1$. 
This claim, in turn, implies that $\sum_{i={t_k} + 1}^{t} \sum_{\tau=t}^{{{t_{k + 1}}}} y_{i}(\tau) = m + \lambda n < c+1$. 
To see this, consider the edge case where there are $m'$ big updates and $n'$ small updates made by all packets arriving between $[t_k+1,t]$ since slot $t$, and they satisfy: (i) their sum of dual variables is smaller than $c$, i.e.,  $m' + \lambda n' < c$; (ii) with one more update (either big or small), the sum of their dual variables is no less than $c$, i.e., either $(m' + 1) + \lambda n' \ge c$ or $m' +  \lambda (n' + 1) \ge c$ holds. From condition (ii) and our claim we know that with the arrival of one more update, the ACK marker $M$ will be larger than or equal to be $1$ (because we have $M \ge M(m,n) \ge 1$) at some slot $t'$ ($t' \le t_{k+1}$). When this happens, the sum of dual variables is at most $\max \{ (m' + 1) + \lambda n',m' + \lambda (n' + 1)\}  = (m' + 1) + \lambda n' = (m' + \lambda n') + 1 < c + 1$, and those dual variables will never be updated after slot $t'$ according to Observation~\ref{obs:ml_within_interval}.  
Therefore, we have $\sum_{i={1}}^{t} \sum_{\tau=t}^{T} y_{i}(\tau) = \sum_{i={t_k} + 1}^{t} \sum_{\tau=t}^{{{t_{k + 1}}}} y_{i}(\tau) < c+ 1$.
Now scaling down the dual solution $\bf{y}$ by a factor of $c/(c + 1)$, we obtain a feasible dual solution $(c/(c + 1))\mathbf{y}$.
\end{proof}

Now with Lemma~\ref{lemma:ml_almost_feasible_solution}, we are ready to show our main result that LAPDOA is
$(((c + 1)/c) \cdot 3/\lambda)$-competitive.

\begin{proof}[Proof of Lemma~\ref{lemma:ml_robustness}]
In the following,  we assume that $d(t)$ is updated to the ACK marker $M$ $(M \ge 1)$ rather than $1$ in Line~\ref{line:d_t_update}, which possibly makes LAPDOA perform worse (i.e., has a larger total cost since the one-time ACK cost now is $c \cdot M$, which is larger than or equal to $c \cdot 1$). 
We show that our CR analysis holds for this worse setting (i.e., $d(t) = M$), and thus our CR analysis also holds for LAPDOA. The benefit of considering this worse setting is that this allows us to allocate the ACK costs to large and small updates.
Specifically, under the worse setting, suppose that LAPDOA
makes an ACK at slot $t_k$, and after $m$ ($m \ge 0$) big updates and $n$ ($n \ge 0$) small updates, LAPDOA is ready to make another ACK at some slot $t_{k+1}$. At this point, the ACK marker is $M = m / (\lambda c) + n  \lambda / c$, and the ACK cost is $c \cdot M = m / \lambda + n  \lambda$. However, instead of calculating the ACK cost at slot $t_{k+1}$, we can distribute the ACK cost to the updates in $[t_k + 1,t_{k+1}]$, that is, each big update gets an ACK cost of $1 /\lambda$ and each small update gets an ACK cost of $\lambda$. The total ACK cost of those $m$ big updates and $n$ small updates is still $m / \lambda + n  \lambda$. Doing this does not change the ACK cost, but now every big update or small update has a contribution to the ACK cost, which helps our analysis in the following when we compute the primal increment (i.e., the sum of ACK cost and holding cost) of each update.

We assume that LAPDOA makes a sequence of ACKs $\pi = \{t_1, t_2, \dots, t_K\}$, where LAPDOA makes the $i$-th ACK at the ON slot $t_i$. 
Our goal is to show that for any $k$-th  ($k \in [0,K]$) ACK interval $[{t_k} + 1,{t_{k+1}}]$ (where the first ACK interval is $[1,{t_{1}}]$ when $k = 0$ and the last ACK interval is $[{t_K} + 1,T]$ when $k= K$), the ratio between the primal objective value and the dual objective value in this $k$-th ACK interval is at most $3/\lambda$, i.e., $P(k)/D(k) \le 3/\lambda$. According to Observation~\ref{obs:ml_after_ACK}, when this ACK interval ends at slot $t_{k+1}$, $P(k)$ and $D(k)$ are never changed.  
This implies that LAPDOA also achieves $ P/ {D} \le 3 / \lambda$ on the entire instance $\mathbf{s}$. 
This, along with Lemma~\ref{lemma:ml_almost_feasible_solution}, concludes that PDOA is $(((c + 1)/c) \cdot 3/\lambda)$-competitive based on the weak duality \cite{buchbinder2009design}.

We first discuss the relation between $P(k)$ and $D(k)$ in the first $K$ ACK interval $[{t_k} + 1,{t_{k+1}}]$ (i.e., $k \in [0,K - 1]$, where there is always an ACK made at slot $t_{k+1}$), and then discuss the relation between $P(K)$ and $D(K)$ in the last ACK interval $[{t_K} + 1,T]$ (i.e., $k = K$, where it is possible that no ACK is made at slot $T$) in the end.

Consider the $k$-th ($k \in [0,K-1]$) ACK interval $[{t_k} + 1,{t_{k+1}}]$ (denoted by $I_k$), where LAPDOA makes two ACKs at the ON slots $t_k$ and $t_{k+1}$ (the first ACK interval is the $0$-th ACK interval $[{t_0} + 1,{t_{1}}]$, where $t_0 = 0$ and LAPDOA only makes one ACK at slot $t_1$), respectively. There are two cases when we make an ACK at slot $t_{k+1}$: $1)$ the ACK marker $M$ is equal to or larger than $1$ at $t_{k+1}$; $2)$ the ACK marker $M$ equals or is larger than $1$ 
at some OFF slot $t'$ $(t'< t_{k+1})$ and $t_{k+1}$ is the very first ON slot after slot $t'$ (in this case, the channels are OFF during $[t', t_{k+1} - 1]$). We analyze the performance of LAPDOA  in these two cases of $t_{k+1}$.

Case $1)$: The ACK marker $M$ is equal to or larger than $1$ at $t_{k+1}$. In this case, we do not have zero updates in $I_k$.  Let $\Delta P$ and $\Delta D$ denote the increment of the primal objective value and the increment of the dual objective value when we make an update, respectively. In the case of a small update, we have $\Delta P = \lambda  + 1$ and $\Delta D = \lambda $, that is, $\Delta P/\Delta D = 1 + 1/\lambda $. In the case of a big update, $\Delta P =  1/\lambda  + 1$ and $\Delta D = 1 $, and we still have $\Delta P/\Delta D = 1 + 1/\lambda $. Obviously, for this $k$-th ACK interval, we have $P(k)/D(k) = 1 + 1/\lambda $.

Case $2)$:  The ACK marker $M$ equals or is larger than $1$ at some OFF slot $t'$ $(t'< t_{k+1})$ and $t_{k+1}$ is the very first ON slot after slot $t'$. 
An illustration is shown in Fig.~\ref{fig:ml_primal_dual_updates}.
We use $C_A(k)$ and $C_H(k)$ to denote the ACK costs and holding costs in $I_k$, respectively. Here $P(k) = C_A(k) + C_H(k)$.
In addition, we assume that there are $m$ ($m \ge 0$) big updates and $n$ ($n \ge 0$) small updates in $I_k$ (some zero updates can also be made but we ignore them since they cannot increase the ACK cost and the dual variable). Given that each big update has an ACK cost of $1/\lambda $ and increases the dual variable by $1$, and each small update has an ACK cost of $\lambda$ and increases the dual variable by $\lambda $, we have ${C_A}(k) = m/\lambda  + \lambda n$ and $D(k) \ge m + \lambda n$ (i.e., $D(k)$ can be larger than $m + \lambda n$ due to the additional updates of dual variables in Lines~\ref{line:ml_final_update_begin}-\ref{line:ml_final_update_end}).
\begin{figure}[!t]
	\centering
	\subfigure[Primal variables $z_i(t)$ updates.]{
         \includegraphics[scale = 0.45]{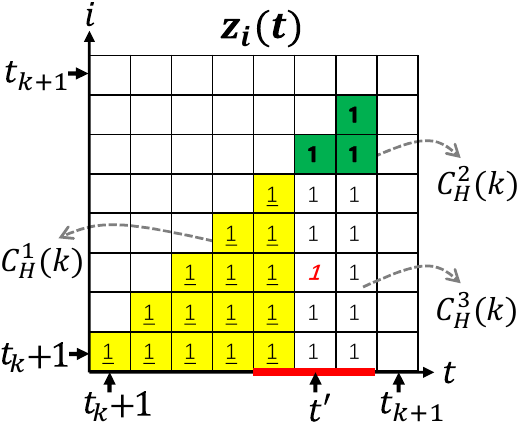}
         \label{fig:online_proof_ml_primal_1}}
         \quad
         \subfigure[Dual variables $y_i(t)$ updates. ]{\includegraphics[scale = 0.45]{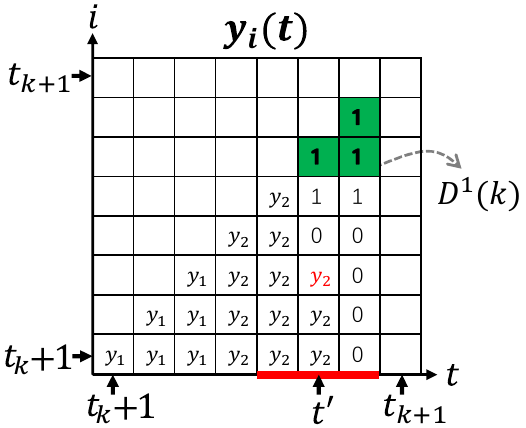}
         \label{fig:online_proof_ml_dual_1}}
         \caption{An illustration of $C_H^1(k), C_H^2(k), C_H^3(k)$ and ${D^1}(k)$ when $j \ne t'$. $C_H^1(k)$ is an equilateral triangle made of 1 (the underlined $1$'s with yellow background in Fig.~\ref{fig:online_proof_ml_primal_1}), $C_H^2(k)$ is an equilateral triangle made of $1$ (the bold $1$'s with green background in Fig.~\ref{fig:online_proof_ml_primal_1}), 
         $C_H^3(k)$ is a rectangle made of $1$ (the regular $1$'s without background in Fig.~\ref{fig:online_proof_ml_primal_1}),
         and 
         ${D^1}(k)$ is an equilateral triangle made of 1 (the bold $1$'s with green background in Fig.~\ref{fig:online_proof_ml_dual_1}).}
\label{fig:online_proof_ml_primal_dual_updates_1}
         \vspace{-10pt}
\end{figure}
Now we can compute         
\begin{equation}
    \begin{aligned}
            & P(k)/D(k)   \\
            & =({C_A(k) + C_H(k)})/{D(k)} \\
            &  =  (m/\lambda  + \lambda n + C_H(k))/ {D(k)} \\
            & \le (m/\lambda  + \lambda n)/({m + \lambda n}) + C_H(k)/{D(k)} \\ 
            & \le 1/{\lambda} + C_H(k)/D(k) \\
            & \buildrel (a) \over \le 1 / \lambda + 2 / \lambda \\
            & =3 / \lambda,
        \end{aligned}
        \label{eq:ml_CR_3_lambda}
\end{equation}
where $(a)$ is proven in the following. 
Similar to the analysis of Case-$(2)$ in the proof of Theorem~\ref{them:online_3}, we can first compute the total holding cost in $I_k$ as ${C_H}(k) = \sum\nolimits_{\tau  = {t_k} + 1}^{{t_{k + 1}} - 1} {\sum\nolimits_{i = {t_k} + 1}^\tau  {{z_i}(\tau )} }  = \sum\nolimits_{\tau  = {t_k} + 1}^{{t_{k + 1}} - 1} {\sum\nolimits_{i = {t_k} + 1}^\tau  1 }  = ({t_{k + 1}} - {t_k} - 1)({t_{k + 1}} - {t_k})/2$, where  ${{z_i}(\tau )} = 1$ for any $\tau  \in [{t_k} + 1,{t_{k + 1}} - 1]$ and $i \in [{t_k} + 1,\tau ]$ is because the packets in $I_k$ are not acked until slot $t_{k+1}$, and they need to pay a holding cost, i.e., ${{z_i}(\tau )} = 1$.  
Next, we split ${C_H}(k)$ into three parts under two different cases. 
Assuming that the ACK marker $M$ is equal to or larger than $1$ after the updating of $j$-th packet at slot $t'$. There are two cases for packet $j$: $1)$ packet $j$ is not packet $t'$ ($j \ne t'$), and $2)$ packet $j$ is packet $t'$ ($j = t'$).
In the first case ($j \ne t'$), we can split ${C_H}(k)$ into 
$C_H^1(k) = \sum\nolimits_{\tau  = {t_k} + 1}^{t' - 1} {\sum\nolimits_{i = {t_k} + 1}^\tau  {{z_i}(\tau )} } $, $C_H^2(k) = \sum\nolimits_{\tau  = t'}^{{t_{k + 1}} - 1} {\sum\nolimits_{i = t'}^\tau  {{z_i}(\tau )} } $, and $C_H^3(k) = \sum\nolimits_{\tau  = t'}^{{t_{k + 1}} - 1} {\sum\nolimits_{i = {t_k} + 1}^{t' - 1} {{z_i}(\tau )} } $ (see an illustration in Fig.~\ref{fig:online_proof_ml_primal_dual_updates_1}). In addition, when $j \ne t'$, we know that the total number of big updates and small updates satisfies $m + n = \sum\nolimits_{\tau  = {t_k} + 1}^{t' - 1} {\sum\nolimits_{i = {t_k} + 1}^\tau  {{z_i}(\tau )}  + } \sum\nolimits_{i = {t_k} + 1}^{j} {{z_i}(t')}  \ge \sum\nolimits_{\tau  = {t_k} + 1}^{t' - 1} {\sum\nolimits_{i = {t_k} + 1}^\tau  {{z_i}(\tau )}  = } C_H^1(k)$. 
\begin{figure}[!t]
	\centering
	\subfigure[Primal variables $z_i(t)$ updates.]{
         \includegraphics[scale = 0.45]{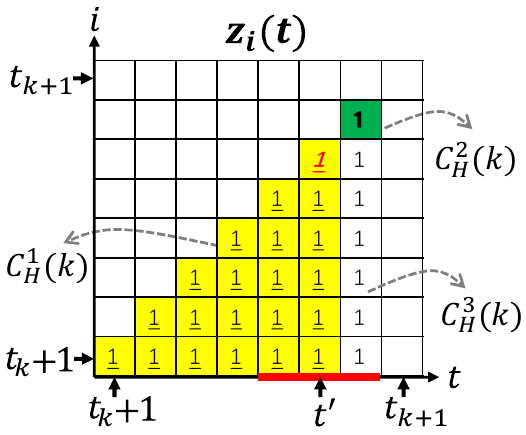}
         \label{fig:online_proof_ml_primal_2}}
         \quad
         \subfigure[Dual variables $y_i(t)$ updates. ]{\includegraphics[scale = 0.45]{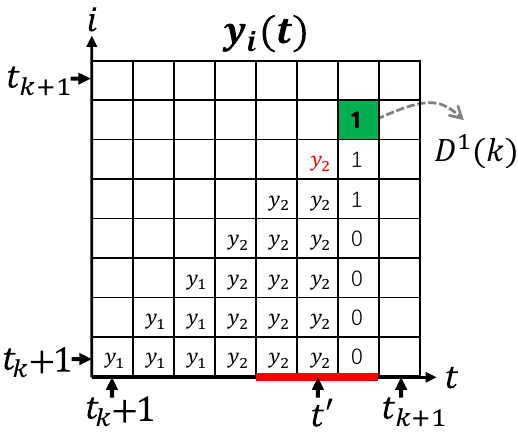}
         \label{fig:online_proof_ml_dual_2}}
         \caption{An illustration of $C_H^1(k), C_H^2(k), C_H^3(k)$ and ${D^1}(k)$ when $j = t'$. $C_H^1(k)$ is an equilateral triangle made of 1 (the underlined $1$'s with yellow background in Fig.~\ref{fig:online_proof_ml_primal_2}), $C_H^2(k)$ is an equilateral triangle made of $1$ (the bold $1$'s with green background in Fig.~\ref{fig:online_proof_ml_primal_2}), 
         $C_H^3(k)$ is a rectangle made of $1$ (the regular $1$'s without background in Fig.~\ref{fig:online_proof_ml_primal_2}),
         and 
         ${D^1}(k)$ is an equilateral triangle made of 1 (the bold $1$'s with green background in Fig.~\ref{fig:online_proof_ml_dual_2}).}
\label{fig:online_proof_ml_primal_dual_updates_2}
\end{figure}
In the second case ($j = t'$), we can split ${C_H}(k)$ into $C_H^1(k) = \sum\nolimits_{\tau  = {t_k} + 1}^{t'} {\sum\nolimits_{i = {t_k} + 1}^\tau  {{z_i}(\tau )} } $, $C_H^2(k) = \sum\nolimits_{\tau  = t' + 1}^{{t_{k + 1}} - 1} {\sum\nolimits_{i = t' + 1}^\tau  {{z_i}(\tau )} } $, and $C_H^3(k) = \sum\nolimits_{\tau  = t' + 1}^{{t_{k + 1}} - 1} {\sum\nolimits_{i = {t_k} + 1}^{t'} {{z_i}(\tau )} } $ (see an illustration in Fig.~\ref{fig:online_proof_ml_primal_dual_updates_2}). Furthermore, when $j = t'$, we know that the total number of big updates and small updates satisfies $m + n = \sum\nolimits_{\tau  = {t_k} + 1}^{t'} {\sum\nolimits_{i = {t_k} + 1}^\tau  {{z_i}(\tau )} }  = C_H^1(k)$.
In the following, we only focus on the analysis of $D(k)$ in the first case since the analysis in the second case is very similar. 
According to Lines~\ref{line:ml_final_update_begin}-\ref{line:ml_final_update_end}, for the dual variables, we have an equilateral triangle made of $1$, which has the same shape as $C_H^2(k)$, we denote it by ${D^1}(k)$ (see an illustration in Fig.~\ref{fig:online_proof_ml_primal_dual_updates_1}).
We can calculate that ${D^1}(k) = \sum\nolimits_{\tau  = t'}^{{t_{k + 1}} - 1} {\sum\nolimits_{i = t'}^\tau  {{y_i}(\tau )} }  = \sum\nolimits_{\tau  = t'}^{{t_{k + 1}} - 1} {\sum\nolimits_{i = t'}^\tau  1 }  = C_H^2(k)$.
Now we can compute 
\begin{equation}
    \begin{aligned}
            C_H(k) &= C_H^1(k) + C_H^2(k) + C_H^3(k) \\
            & \buildrel (a) \over \le 2(C_H^1(k)+ C_H^2(k)) \\
            & \buildrel (b) \over \le 2((m + n) + C_H^2(k)) \\
            & \buildrel \over= 2((m + n) + {D^1}(k)) \\
            &  \le 2(({m + \lambda n})/{\lambda } + ({{D^1}(k)})/{\lambda }) \\
            &  = 2/\lambda  \cdot (m + \lambda n + {D^1}(k)) \\
            & \buildrel (c) \over \le 2/\lambda  \cdot D(k),
        \end{aligned}
        \label{eq:robustness_proof_holding_dual_ratio}
\end{equation}
where $(a)$ can be proven using the same techniques in Eq.~\eqref{eq:triangles_larger_than_rectangle}, $(b)$ is because $C_H^1(k)$ is at least $m+n$ as we analyzed above, and  $(c)$ is because $D(k)$ is least $m + \lambda n + {D^1}(k)$ (i.e., $D(k)$ can be larger than $m + \lambda n + {D^1}(k)$ due to Lines~\ref{line:ml_final_update_begin}-\ref{line:ml_final_update_end}).
This completes (a) in Eq.~\eqref{eq:ml_CR_3_lambda}.

In the end, we consider the last time interval $[{t_K} + 1,T]$ (denoted by $I_K$).
If the last slot is an ON slot and LAPDOA makes the last ACK exactly at the last slot, i.e., $t_K= T$, then our previous analysis in Cases $1$ and $2$ still holds. Next, we consider the scenario where $t_K < T$. There are two cases at slot $T$: $1)$ $T$ is the slot before the ACK marker $M$ equals or is larger than $1$; $2)$ $T$ is the slot when or after the ACK marker $M$ equals or is larger than $1$. In both cases, there is no ACK made during $I_K$.
Let $P(K)$ and $D(K)$ be the primal and dual objective value in $I_{K}$, respectively.

Case $1)$: $T$ is the slot before the ACK marker $M$ equals or is larger than $1$. In this case, we do not have zero updates in $I_K$.  Let $\Delta P$ and $\Delta D$ denote the increment of the primal objective value and the increment of the dual objective value when we make an update, respectively. In the case of a small update, we have $\Delta P = c \cdot 0 + 1 =  1$ and $\Delta D = \lambda $, that is, $\Delta P/\Delta D = 1/\lambda $. In the case of a big update, $\Delta P = c \cdot 0 + 1 = 1$ and $\Delta D = 1 $, so we have $\Delta P/\Delta D = 1$. Obviously, for this $K$-th ACK interval, we have $P(K)/D(K) \le 1/\lambda $. 

Case $2)$: $T$ is the slot when or after the ACK marker $M$ equals or is larger than $1$. We assume that the ACK marker $M$ equals or is larger than $1$ at slot $t'$ $(t'> t_K)$. According to the definition of Case $2$, we have $T \ge t'$.
We claim that the channels are OFF during the interval $[t', T]$. Otherwise, an ACK will be made during $[t', T]$, which contradicts the fact that there is no ACK during $I_K$.
We use $C_A(K)$ to denote the total ACK cost and use $C_H(K)$ to denote the total holding cost in $I_K$.  Here $P(K) = C_A(K) + C_H(K) = 0 + C_H(K)$, where $C_A(K) = 0$ because there is no ACK made during $I_K$.
We have
  \begin{equation}
        \begin{aligned}
            {P(K)}/{D(K)} &  = ({{C_A(K) + C_H(K)}})/{D(K)} \\
            & = 0 + {C_H(K)}/{D(K)}\\ 
            & \buildrel (a) \over \le 2/\lambda,
        \end{aligned}
    \end{equation}
where the analysis in $(a)$ is the same as the (a) in Eq.~\eqref{eq:ml_CR_3_lambda}.

In summary, LAPDOA  achieves $P(k)/{D}(k) \le 3/\lambda $ for any ACK interval, and thus LAPDOA also achieves $ P/ {D} \le 3/\lambda $ on the entire instance. 
\end{proof}

\section{Proof of Lemma~\ref{lemma:ml_consistency}}
\label{appendix:ml_consistency_proof}
\begin{proof}
In this proof, similar to the proof of Lemma~\ref{lemma:ml_robustness}, we still assume that $d(t)$ is updated to the ACK marker $M$ $(M \ge 1)$ rather than $1$ in Line~\ref{line:d_t_update} (except the analysis of big updates in the case of $\lambda \in (0, 1/c]$, where we still update $d(t)$ to be $1$). Therefore, for the big updates and small updates in any ACK interval, each big update can be charged with an ACK cost of $1/ \lambda$, and each small update can be charged with an ACK cost of $\lambda$. 
In particular, for the big updates and small updates in the last time interval $[t_K, T]$ (we assume that LAPDOA makes the last ACK at slot $t_K$), though there is no ACK made during $[t_K, T]$, we still charge each big update and each small update with an ACK cost of $1/ \lambda$ and $\lambda$, respectively. 
Doing this can possibly increase the total cost of LAPDOA, but we can show that the upper bound we derived still holds in this case.

We first consider $\lambda \in (0, 1/c]$. 
In this case, 
LAPDOA has three types of updates: big updates, small updates, and zero updates. 
Consider the total cost of big updates first. 
Once the prediction $\mathcal{P}$ makes an ACK at some ON slot $t$, LAPDOA will make a big update immediately, the ACK marker becomes  $M=1/\lambda c = 1/c \cdot 1/\lambda  \ge 1/c \cdot c = 1$, and thus LAPDOA will also make an ACK at the beginning of slot $t$. 
Since the prediction $\mathcal{P}$ has a total number of $C_A(\mathbf{s},\mathcal{P})/c$ ACKs, then LAPDOA has also $C_A(\mathbf{s},\mathcal{P})/c$ big updates. 
Each big update leads to an ACK, which results in an ACK cost of $c$. Therefore, the total cost of the big updates in LAPDOA is $C_A(\mathbf{s},\mathcal{P})/c \cdot c = C_A(\mathbf{s},\mathcal{P})$.
By the definition of small update and zero update, each small update or zero update in LAPDOA corresponds to one packet in the prediction $\mathcal{P}$ that has not been acked yet, which requires the prediction $\mathcal{P}$ to pay a holding cost of $1$, so the total number of small updates and zero udpates is at most $C_H(\mathbf{s},\mathcal{P}) / 1 = C_H(\mathbf{s},\mathcal{P})$.
For each of the small updates, the increase in the primal is  $\Delta P =  \lambda  + 1$, and for each of the zero updates, the increase in the primal is $\Delta P = 0 + 1 = 1$, so the total cost of small updates and zero updates is at most $(1+\lambda)C_H(\mathbf{s},\mathcal{P})$.
In summary, the total cost of big updates, small updates, and zero updates is $C_A(\mathbf{s},\mathcal{P}) + (1+\lambda)C_H(\mathbf{s},\mathcal{P})$. 
This concludes Eq.~\eqref{eq:consistency_1}.

Next, we analyze $\lambda \in (1/c,1]$. We consider two cases: $1)$ the channels are ON all the time, and $2)$ there are some OFF channels. We show that Eq.~\eqref{eq:consistency_2} holds for both cases.
    
Case $1)$: The channels are ON all the time. In this case, LAPDOA will generate only two types of updates: small updates and big updates. 
By the definition of small updates, for each of them, there is one corresponding packet in the prediction $\mathcal{P}$ that has not been acked yet, which requires the prediction $\mathcal{P}$ to pay a holding cost of $1$ for this packet, so the total number of the small updates is at most $C_H(\mathbf{s},\mathcal{P})/1=C_H(\mathbf{s},\mathcal{P})$. 
Each small update contributes $(\lambda + 1)$ to the primal objective value, thus the total cost of small updates is at most $(1+\lambda)C_H(\mathbf{s},\mathcal{P})$.
Next, we analyze the total cost of big updates. 
We claim that for any ACK made by the prediction $\mathcal{P}$, LAPDOA makes at most $\left\lceil {\lambda c} \right\rceil $ big updates for this ACK. 
To see this, assuming that prediction $\mathcal{P}$ makes an ACK at slot $t$, and consider all the big updates due to this ACK (i.e., those big updates produced by the packets in LAPDOA that have not been acked yet and arrives before or at slot $t$). 
After at most 
$\left\lceil {\lambda c} \right\rceil $ such big updates, the ACK marker will become $M = \left\lceil {\lambda c} \right\rceil  \cdot 1/ \lambda c \ge 1$ at some slot $t' \ge t$. 
Once the ACK marker $M$ equals or is larger than $1$, 
no more big updates will be made for the ACK made by prediction $\mathcal{P}$ at slot $t$.
Given that prediction $\mathcal{P}$ makes $C_A(\mathbf{s},\mathcal{P})/c$ ACKs, then LAPDOA makes at most $\left\lceil {\lambda c} \right\rceil \cdot C_A(\mathbf{s},\mathcal{P})/c$ big updates. For each big update, the increment in the primal objective value is $\Delta P = 1/ \lambda  + 1$. Therefore, the total cost of big updates is at most $(1/ \lambda  + 1) \cdot \left\lceil {\lambda c} \right\rceil \cdot C_A(\mathbf{s},\mathcal{P})/c$.
In summary, when the channels are always ON, the total cost of LAPDOA is upper bounded by $C(\mathbf{s},\mathcal{P}, \lambda)   \le(\lambda  + 1)C_H(\mathbf{s},\mathcal{P}) + (1 / \lambda  + 1)  \left\lceil {\lambda c} \right\rceil  {C_A(\mathbf{s},\mathcal{P})}/{c},$  
which is smaller than the bound in Eq.~\eqref{eq:consistency_2}.

Case $2)$: There are some OFF channels. In this case, LAPDOA will generate three types of updates: small updates, big updates, and zero updates. Note that these zero updates only increase the holding costs. 
We use ${C_{A,s}}$ and ${C_{H,s}}$ to denote the ACK cost and the holding cost of all the small updates, respectively; use ${C_{A,b}}$ and ${C_{H,b}}$ to denote the ACK cost and the holding cost of all the big updates, respectively; and ${C_{H,z}}$ to denote the holding cost of all the zero updates. 
Obviously, we have $C({\bf{s}},{\cal P},\lambda ) = {C_{A,b}} + {C_{H,b}} + {C_{A,s}} + {C_{H,s}} + {C_{H,z}}$. 
Furthermore, similar to the analysis in Case $1$, for the big updates, we can obtain ${C_{A,b}} + {C_{H,b}} \le (1/\lambda  + 1) \cdot \left\lceil {\lambda c} \right\rceil  \cdot {C_A}({\bf{s}},{\cal P})/c$; and for the small updates, we have ${C_{A,s}} + {C_{H,s}} \le (\lambda  + 1){C_H}({\bf{s}},{\cal P})$. 
To analyze the holding costs of zero updates ${C_{H,z}}$, 
our idea is to bound it by the holding cost of small updates and big updates, which can be further bounded by the total cost of prediction.  
To this end, 
we assume that LAPDOA makes a sequence of ACKs $\pi = \{t_1, t_2, \dots, t_K\}$, where LAPDOA makes the $i$-th ACK at the ON slot $t_i$
(i.e., $s(t_i)d(t_i) = 1$). 
Our goal is to show that in any $k$-th  ($k \in [0,K]$) ACK interval $[{t_k} + 1,{t_{k+1}}]$ (where the first ACK interval is $[1,{t_{1}}]$ when $k = 0$ and the last ACK interval is $[{t_K} + 1,T]$ when $k= K$), the holding costs of zero updates can be bounded by the holding cost of small updates and big updates.

\begin{figure}[!t]
	\centering
	\subfigure[when $j \ne t'$]{
         \includegraphics[scale = 0.45]{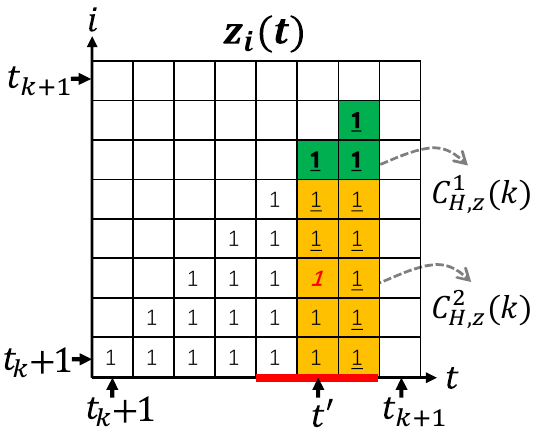}
         \label{fig:ml_proof_consistency_primal_1}}
         \subfigure[when $j = t'$]{\includegraphics[scale = 0.45]{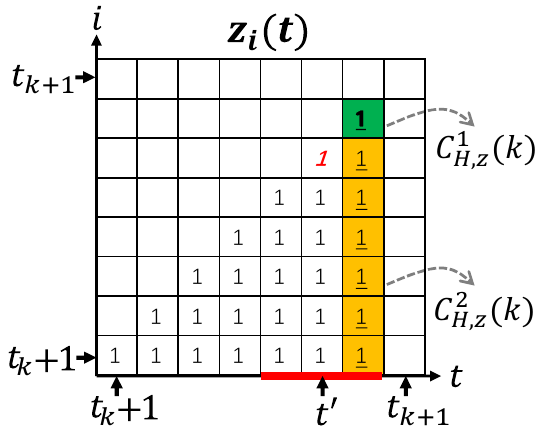}
         \label{fig:ml_proof_consistency_primal_2}}
         \caption{An illustration of $C_{H,z}(k)$, $C_{H,z}^1(k)$, and $C_{H,z}^2(k)$ in two different cases ($j \ne t'$ and $j = t'$). $C_{H,z}(k)$ is the sum of the underlined $1$'s, $C_{H,z}^1(k)$ is an equilateral triangle made of $1$ (the bold underlined $1$'s with green background), and
         $C_{H,z}^2(k)$ is a rectangle made of $1$ (the sum of some regular $1$'s and some underlined $1$'s with orange background).}         \label{fig:ml_proof_consistency_primal}
         \vspace{-10pt}
\end{figure}

We consider the $k$-th ($k \in [0, K-1]$) ACK interval $[{t_k} + 1,{t_{k+1}}]$ (denoted by $I_k$), where the LAPDOA makes two ACKs at the ON slots $t_k$ and $t_{k+1}$, respectively. Still, there are two cases when we make an ACK at $t_{k+1}$: $1)$ the ACK marker $M$ is equal to or larger than $1$ at $t_{k+1}$ and $t_{k+1}$ is an ON slot; $2)$ the ACK marker $M$ equals or is larger than $1$ 
at some OFF slot $t'$ $(t'< t_{k+1})$ and $t_{k+1}$ is the very first ON slot after slot $t'$. 
Note that though the following analysis is for the general ACK interval $[{t_k} + 1,{t_{k+1}}]$ ($k \in [0, K-1]$), they can be easily extended the last ACK interval $[{t_K} + 1,T]$.

\begin{enumerate}
\item The ACK marker $M$ is equal to or larger than $1$ at the ON slot $t_{k+1}$.  
In this case, there is no zero update, and the holding cost of zero updates is $0$.

\item  The ACK marker $M$ equals or is larger than $1$ at some OFF slot $t'$ $(t'< t_{k+1})$ and $t_{k+1}$ is the very first ON slot after slot $t'$ (i.e., the channels are OFF during $[t', t_{k+1} - 1]$).  
Assuming that the ACK marker $M$ is equal to or larger than $1$ after the updating of the $j$-th packet at slot $t'$. In this case, for the holding costs of zero updates in $I_k$ (denoted by ${C_{H,z}}(k)$), we can compute it under two different cases based on packet $j$: $1)$ packet $j$ is not packet $t'$ ($j \ne t'$), and $2)$ packet $j$ is packet $t'$ ($j = t'$).
In the first case ($j \ne t'$), we can denote ${C_{H,z}}(k)$ as ${C_{H,z}}(k) = \sum\nolimits_{i = j + 1}^{t'} {{z_i}(t')}  + \sum\nolimits_{\tau  = t' + 1}^{{t_{k + 1}} - 1} {\sum\nolimits_{i = {t_k} + 1}^\tau  {{z_i}(\tau )} } $ (see an illustration in Fig.~\ref{fig:ml_proof_consistency_primal_1}); and in the second case ($j = t'$), we can denote ${C_{H,z}}(k)$ as ${C_{H,z}}(k) = \sum\nolimits_{\tau  = t' + 1}^{{t_{k + 1}} - 1} {\sum\nolimits_{i = {t_k} + 1}^\tau  {{z_i}(\tau )} } $ (see an illustration in Fig.~\ref{fig:ml_proof_consistency_primal_2}). 
In the following, we only focus on the analysis of ${C_{H,z}}(k)$ in the first case since the analysis in the second case is very similar. 
Next, we bound ${C_{H,z}}(k)$ by two areas: $C_{H,z}^1(k) \triangleq \sum\nolimits_{\tau  = t'}^{{t_{k + 1}} - 1} {\sum\nolimits_{i = t'}^\tau  {{z_i}(\tau )} } $ and $C_{H,z}^2(k) \triangleq \sum\nolimits_{\tau  = t'}^{{t_{k + 1}} - 1} {\sum\nolimits_{i = {t_k} + 1}^{t' - 1} {{z_i}(\tau )} } $ (see an illustration in Fig.~\ref{fig:ml_proof_consistency_primal_1}). 
Clearly, we have ${C_{H,z}}(k) = \sum\nolimits_{i = j + 1}^{t'} {{z_i}(t')}  + \sum\nolimits_{\tau  = t' + 1}^{{t_{k + 1}} - 1} {\sum\nolimits_{i = {t_k} + 1}^\tau  {{z_i}(\tau )} }  \le \sum\nolimits_{\tau  = t'}^{{t_{k + 1}} - 1} {\sum\nolimits_{i = t'}^\tau  {{z_i}(\tau )} }  + \sum\nolimits_{\tau  = t'}^{{t_{k + 1}} - 1} {\sum\nolimits_{i = {t_k} + 1}^{t' - 1} {{z_i}(\tau )} }  = C_{H,z}^1(k) + C_{H,z}^2(k)$.
We use $C_{H,z}^1$ to denote the sum of $C_{H,z}^1(k)$ over all the ACK intervals $I_k$ ($k \in [0, K]$), i.e., $C_{H,z}^1 \triangleq \sum\nolimits_{k = 0}^K {C_{H,z}^1(k)} $. 
For any of the zero updates in $C_{H,z}^1$, there is one corresponding packet in the prediction $\mathcal{P}$ that has not been acked yet since the channels are OFF during some $[t',{t_{k + 1}} - 1]$, which requires the prediction $\mathcal{P}$ to pay a holding cost of $1$ for this packet. 
Similarly, for any of the small updates, by the definition of small updates, there is one corresponding packet in the prediction $\mathcal{P}$ that has not been acked yet, which requires the prediction $\mathcal{P}$ to pay a holding cost of $1$ for this packet. Therefore, the total number of the zero updates in $C_{H,z}^1$ and the small updates is at most $C_H(\mathbf{s},\mathcal{P})/1=C_H(\mathbf{s},\mathcal{P})$, 
and each of such update has a total cost at most $(1+\lambda)$, 
which indicates that the total cost of the zero updates in $C_{H,z}^1$  and the small updates is at most $(1+\lambda)C_H(\mathbf{s},\mathcal{P})$, i.e.,  
\begin{equation}
    {C_{A,s}} + {C_{H,s}} + C_{H,z}^1 \le (1+\lambda)C_H(\mathbf{s},\mathcal{P}).
\end{equation}
For the holding cost $C_{H,z}^2(k)$, same as the analysis in $(a)$ of Eq.~\eqref{eq:robustness_proof_holding_dual_ratio}, it is upper bounded by the sum of the holding cost of big updates and small updates in $I_k$ and the holding cost of the zero updates in $C_{H,z}^1(k)$, i.e., $C_{H,z}^2(k) \le C_{H,s}(k) + C_{H,b}(k) + C_{H,z}^1(k)$. 
More generally, let $C_{H,z}^2$ denote the sum of $C_{H,z}^2(k)$ over all the ACK intervals $I_k$ ($k \in [0, K]$), i.e., $C_{H,z}^2 \triangleq \sum\nolimits_{k = 0}^K {C_{H,z}^2(k)} $. 
Then we have 
\begin{align*}
    C_{H,z}^2  & \le C_{H,s} + C_{H,b} + C_{H,z}^1 \\
    & = (C_{H,s} + C_{H,z}^1) + C_{H,b} \\
    & \buildrel (a) \over \le C_H(\mathbf{s},\mathcal{P}) + C_{H,b} \\
    & \buildrel (b) \over \le C_H(\mathbf{s},\mathcal{P}) + \left\lceil {\lambda c} \right\rceil \times {C_A(\mathbf{s},\mathcal{P})}/{c},
\end{align*}
where $(a)$ is because as we showed before, the total number of the zero updates in $C_{H,z}^1$ and the small updates is at most $C_H(\mathbf{s},\mathcal{P})$, and each of small updates or zero updates increases the holding costs of LAPDOA by $1$, so the total holding cost of them is at most $C_H(\mathbf{s},\mathcal{P})$; $(b)$  is due to the total number of big updates is at most $\left\lceil {\lambda c} \right\rceil \times C_A(\mathbf{s},\mathcal{P})/c$, and each of big updates increases the holding costs by $1$, so we have  $C_{H,b}  \le \left\lceil {\lambda c} \right\rceil \times C_A(\mathbf{s},\mathcal{P})/c$. 
\end{enumerate}
In summary, the total cost of LAPDOA in Case $2$ is
\begin{align*}
    & C(\mathbf{s},\mathcal{P}, \lambda) \\
    & =  {C_{A,b}} + {C_{H,b}} + {C_{A,s}} + {C_{H,s}} + {C_{H,z}} \\
    & \le {C_{A,b}} + {C_{H,b}} + {C_{A,s}} + {C_{H,s}} + C_{H,z}^1 + C_{H,z}^2 \\
    & =  ({C_{A,b}} + {C_{H,b}} ) + ({C_{A,s}} + {C_{H,s}} + C_{H,z}^1) + C_{H,z}^2 \\
    & \le (1 / \lambda  + 1) \times \left\lceil {\lambda c} \right\rceil \times {C_A(\mathbf{s},\mathcal{P})}/{c} + (1+\lambda)C_H(\mathbf{s},\mathcal{P})  \notag \\
    & \quad + C_H(\mathbf{s},\mathcal{P}) + \left\lceil {\lambda c} \right\rceil \times {C_A(\mathbf{s},\mathcal{P})}/{c} \\
    & = (1 / \lambda  + 2) \times \left\lceil {\lambda c} \right\rceil \times {C_A(\mathbf{s},\mathcal{P})}/{c} + (2+\lambda)C_H(\mathbf{s},\mathcal{P}).
\end{align*}

Finally, combining the results in Case $1$ and Case $2$, we see that Eq.~\eqref{eq:consistency_2} holds.
\end{proof}

\end{document}